\documentclass[a4paper]{article}

\newtheorem{theorem}{Theorem}
\newtheorem{propositioni}[theorem]{Proposition}
\newenvironment{proposition}
   {\begin{propositioni}\rm}{\end{propositioni}}

\newtheorem{lemmai}[theorem]{Lemma}
\newenvironment{lemma}
   {\begin{lemmai}\rm}{\end{lemmai}}

\newtheorem{propertyi}[theorem]{Property}

\newtheorem{postulatei}[theorem]{Postulate}

\newtheorem{definitioni}{Definition}
\newenvironment{definition}
   {\begin{definitioni}\rm}{\end{definitioni}}

\newtheorem{corollaryi}[theorem]{Corollary}
\newenvironment{corollary}
   {\begin{corollaryi}\rm}{\end{corollaryi}}

\newtheorem{defii}[theorem]{Definition}

\newtheorem{claimi}[theorem]{Claim}

\newtheorem{conjecturei}[theorem]{Conjecture}

\newtheorem{exercisei}[theorem]{Exercise}

\newtheorem{questioni}[theorem]{Open Question}

\newtheorem{conventioni}[theorem]{Convention}

\newtheorem{facti}[theorem]{Fact}

\newtheorem{problemi}[theorem]{Problem}

\newtheorem{remarki}[theorem]{Remark}
\newenvironment{remark}
   {\begin{remarki}\rm}{\end{remarki}}

\newtheorem{examplei}[theorem]{Example}
\newenvironment{example}
   {\begin{examplei}\rm}{\end{examplei}}

\newtheorem{notationi}[theorem]{Notation}
\newenvironment{notation}
   {\begin{notationi}\rm}{\end{notationi}}

\newenvironment{proof}
   {\begin{trivlist}\item[]{\sc Proof.}}
   {\hfill {\sc qed}\end{trivlist}}

\newtheorem{claim2}{Claim}

\newcounter{tbsnr}
\newenvironment{tbs}
{\addtocounter{tbsnr}{1}\par\bigskip \noindent\fbox{\thetbsnr}
\hspace*{\fill}\begin{minipage}{0.8\textwidth}\tt}
{\end{minipage}\hspace*{\fill}\bigskip}

\newcommand{\Cn}{\mathop{\rm Cn}}

\usepackage{times}
\usepackage{amsmath, amssymb}
\usepackage{graphicx}
\usepackage{scalefnt}
\usepackage{colortbl}

\newcommand{\ben}{\begin{enumerate}}
\newcommand{\een}{\end{enumerate}}
\newcommand{\bit}{\begin{itemize}}
\newcommand{\eit}{\end{itemize}}

\newcommand{\imp}{\rightarrow}
\newcommand{\Imp}{\Rightarrow}

\newcommand{\sm}[1]{\textcolor{red}{#1}}

\newcommand{\con}{\overline} %%% for the contrary relation

\newcommand{\K}{\mathcal{K}}
\newcommand{\A}{\mathcal{A}}

\newcommand{\C}{\mathcal{C}}

\newcommand{\R}{\mathcal{R}}

\newcommand{\D}{\mathcal{D}}

\newcommand{\LA}{\mathcal{L}}

\newcommand{\begintab}[1]{ \begin{tabular}{#1} }
\newcommand{\btab}[1]{\bigskip \par \begintab{#1}}
\newcommand{\etab} {\end{tabular} \medskip \par \noindent }

\newcommand{\ASPIC}{\emph{ASPIC}$^+$}

\title{A General Account of Argumentation with Preferences}

\author{Sanjay Modgil$^1$  and Henry Prakken$^2$\\
\begin{small}1. Department of Infomatics, King's College London (sanjay.modgil@kcl.ac.uk)\end{small}\\
\begin{small}2. Department of Information and Computing Sciences, Utrecht University, and \end{small}\\\begin{small}Faculty of Law, University of Groningen (H.Prakken@uu.nl) \end{small}}
\date{}
\begin{document}

\maketitle

\begin{abstract}

This paper builds on the recent \ASPIC\ formalism, to develop a general framework for argumentation with preferences. We motivate a revised definition of conflict free sets of arguments, adapt \ASPIC\ to accommodate a broader range of instantiating logics, and show that under some assumptions, the resulting framework satisfies key properties and rationality postulates. We then show that the generalised framework accommodates Tarskian logic instantiations extended with preferences, and then study instantiations of the framework by classical logic approaches to argumentation. We conclude by arguing that \ASPIC's modelling of defeasible inference rules further testifies to the generality of the framework, and then examine and counter recent critiques of Dung's framework and its extensions to accommodate preferences.

\end{abstract}

\section{Introduction}\label{SectionIntroduction}

Argumentation is a key topic in the logical study of nonmonotonic reasoning and the dialogical study of inter-agent communication \cite{ArgInAI,RahSim}. Argumentation is a form of reasoning that makes explicit the reasons for the conclusions that are drawn and how conflicts between reasons are resolved. This provides a natural mechanism to handle inconsistent and uncertain information and to resolve conflicts of opinion between intelligent agents.  In logical models of nonmonotonic reasoning, the argumentation metaphor has proved to overcome some drawbacks of other formalisms. Many of these have a mathematical nature that is remote from how people actually reason, which makes it difficult to understand and trust the behaviour of an intelligent system. The argumentation approach bridges this gap by providing logical formalisms that are rigid enough to be formally studied and implemented, while at the same time being close enough to informal reasoning to be understood by designers and users.

Many theoretical and practical developments build on Dung's seminal theory of abstract argumentation \cite{dung95}. A Dung \emph{argumentation framework} (\emph{AF}) consists of a conflict-based binary \emph{attack} relation $\C$ over a set of arguments $\A$. The justified
arguments are then evaluated based on subsets of $\A$ (\emph{extensions}) defined under a range of semantics. The arguments in an extension are required to not attack each other (extensions are \emph{conflict free}), and attack any argument that in turn attacks an argument in the extension (extensions \emph{reinstate/defend} their contained arguments). Dung's theory has been developed in many directions, including argument game proof theories \cite{ModCam} to determine extension membership of a given argument. Also, several works augment \emph{AF}s with preferences and/or values \cite{AmCay,BC03,ModgilAIJ,hp10aspicJAC}, so that the conflict-free extensions, and so justified arguments, are evaluated only with respect to the successful attacks (\emph{defeats}), where an argument $X$ is said to defeat an argument $Y$ iff $X$ attacks $Y$ and $Y$ is not preferred to $X$.

The widespread impact of Dung's work can partly
be attributed to its level of abstraction. \emph{AF}s can be instantiated by a wide range of logical formalisms; one is free to choose a logical language $\LA$ and define what constitutes an argument and attack between arguments defined by a theory in $\LA$. The theory's
inferences can then be defined in terms of the conclusions of the theory's justified arguments. Indeed, the inference relations of existing logics, including logic programming and various non-monotonic logics,
%(e.g. default, auto-epistemic and defeasible logic)
%%%
% FINE \marginpar{Sanjay, check shortening}
%%%
have been given argumentation based characterisations \cite{BDKT97,dung95,GM00}. Dung's theory thus provides a dialectical semantics for these logics, and the above-mentioned argument games can be viewed as alternative dialectical proof theories for these logics. The fact that reasoning in existing non-monotonic logics can thus be characterised, testifies to the generality of the dialectical principles of attack and reinstatement; principles that are also both intuitive and familiar in human modes of reasoning, debate and dialogue. Argumentation theory thus provides a characterisation of both human and logic-based reasoning in the presence of uncertainty and conflict, through the abstract dialectical modelling of the \emph{process} whereby arguments can be moved to attack and reinstate/defend other arguments. The theory's value can therefore in large part be attributed to its explanatory potential for making  non-monotonic reasoning processes inspectable and readily understandable for human users, and its underpinning of dialogical and more general communicative interactions that may involve heterogenous (human and software) agents reasoning in the presence of uncertainty and conflict.

More recently, the \emph{ASPIC} framework \cite{c+a07} was developed in response to the fact that the abstract nature of Dung's theory gives no guidance as to what kinds of instantiations satisfy intuitively rational properties. \emph{ASPIC} was not designed from scratch but was meant to integrate, generalise and further develop existing work on structured argumentation, partly originating from before Dung's paper (e.g.\ \cite{Pol87,S+L92,pol94,BDKT97,PS97}). \emph{ASPIC} adopts an intermediate level of abstraction between Dung's fully abstract level and concrete instantiating logics, by making some minimal assumptions on the nature of the logical language and the inference rules, and then providing abstract accounts of the structure of arguments, the nature of attack, and the use of preferences. \cite{c+a07} then formulated consistency and closure postulates that cannot be formulated at the abstract level, and showed these postulates to hold for a special case of \emph{ASPIC}; one in which preferences were \emph{not} accounted for. In \cite{hp10aspicJAC}, \ASPIC\ then generalised \emph{ASPIC} to accommodate a broader range of instantiations (including assumption-based argumentation \cite{BDKT97} and systems using argument schemes), and showed that under some assumptions, the postulates were satisfied when applying preferences. \cite{g+p11} subsequently showed that the Carneades system \cite{gpw07} is an instance of \ASPIC\ with no defeat cycles.

In this paper we build on and modify \cite{hp10aspicJAC}'s \ASPIC\  framework, to develop a more general structured framework for argumentation with preferences. We make three main contributions. We first motivate a revised definition of conflict free sets of arguments for \ASPIC, adapt \ASPIC\  to accommodate a broader range of instantiating logics, and  show that the resulting framework satisfies the key properties and postulates in \cite{dung95} and \cite{c+a07}. Second, we formalise instantiation of the new framework by Tarskian (and in particular classical) logics extended with preferences, and demonstrate that such instantiations satisfy \cite{c+a07}'s rationality postulates. Third, we examine and counter recent critiques of Dung's framework and its extensions to accommodate preferences\footnote{The current paper extends \cite{ModPrakIJCAI} in which the revised definition of conflict free sets is first proposed, and \ASPIC\ is adapted to accommodate classical logic instantiations.}.

With regard to the first contribution, Section \ref{Section-ArgumentationAndLogic} presents the conceptual foundations for our framework in the context of the above value proposition of argumentation as providing a bridging role between formal logic and human modes of reasoning. Specifically, we : i) posit criteria for defining attack relations, given their dual role in declaratively denoting the mutual incompatibility of the information contained in the attacking arguments, and their dialectical use; ii) motivate the distinction between preference dependent and preference independent attacks, where only the former's use in a dialectical context (as defeats) should be contingent upon preferences; iii) argue that unlike current approaches \cite{AmCay,BC03,ModgilAIJ}, including \cite{hp10aspicJAC}'s \ASPIC, it is conceptually more intuitive to define conflict-free sets in terms of those that do not contain attacking arguments, so that defeats
are only deployed dialectically. Section \ref{Section-ASPIC+} then revisits and generalises \cite{hp10aspicJAC}'s \ASPIC\ framework in light of Section \ref{Section-ArgumentationAndLogic}'s conceptual foundations. The new notion of conflict-free is adopted, and \cite{hp10aspicJAC}'s \ASPIC\ framework is extended to accommodate instantiation by arguments with consistent premises, thus generalising the framework to accommodate a broader range of instantiations. Section \ref{SectionPropertiesPostulates} then presents key technical results. We show that Section \ref{Section-ASPIC+}'s revised and generalised \ASPIC\ satisfies properties of Dung's theory and \cite{c+a07}'s rationality postulates.

Section \ref{SectionInstantiations} then presents the second main contribution, so testifying to the generality of the framework proposed here. To start with, we generalise results of \cite{hp10aspicJAC}, in which preferences defined over arguments on the basis of pre-orderings over arguments' constituent rules and premises,  are  shown to satisfy properties that ensure satisfaction of rationality postulates. In this paper we show that these properties are also satisfied by other ways of defining preferences, and furthermore address some limitations of \cite{hp10aspicJAC}'s way of defining preferences. We then relate our work to Amgoud \& Besnard's \cite{a+b09,a+b10} recent `abstract logic' approach to argumentation, which considers instantiations of Dung's framework by Tarskian logics. We combine this approach with the \ASPIC\ framework,
%and show that the combination makes the strict rules of \ASPIC\ well-behaved with respect to the rationality postulates. Then we
and then extend \cite{a+b09,a+b10}'s abstract logic approach with preferences, and also combine this extension with \ASPIC. Given Section \ref{SectionPropertiesPostulates}'s results, these combinations imply that we are the first to show satisfaction of \cite{c+a07}'s rationality postulates for Tarskian logic instantiations with and without preferences. Following this, we reconstruct classical logic approaches to argumentation \cite{b+h01,b+h08,g+h11}, including those that additionally accommodate preferences \cite{AmCay}. To the best of our knowledge, we are the first to prove \cite{c+a07}'s postulates for classical logic approaches with preferences. Finally, we show a correspondence between a particular classical logic instantiation of Section \ref{Section-ASPIC+}'s \ASPIC\ framework and Brewka's preferred subtheories \cite{BrewkaPS}.

Section \ref{SectionRelatedWork} discusses related work, and so presents our third main contribution. Specifically, we compare the generality of \ASPIC\ with the abstract logic proposal for structured argumentation, and argue that the latter only applies to deductive (e.g., classical logic) approaches, and not to mixed deductive and defeasible argumentation which requires modelling of defeasible inference rules. We also argue that inclusion of defeasible inference rules in models of argumentation is required if argumentation is to bridge the gap between formalisms and human reasoning, as defeasible reasons are an essential ingredient of human reasoning. Section \ref{SectionRelatedWork} also counters a number of recent criticisms of Dung's abstract approach, as well as critiques of Dung's approach extended with preferences. We claim that a proper modelling of the use of preferences requires making the structure of arguments explicit.

\section{Logic, Argumentation and Preferences}\label{Section-ArgumentationAndLogic}

\subsection{Background}\label{SectionBackground}
A \emph{Dung argumentation framework} (\emph{AF}) \cite{dung95} is a tuple ($\A,\C$), where $\C \subseteq \A \times \A$ is a binary attack relation on the arguments $\A$. $S \subseteq \A$ is then said to be \emph{conflict free} iff $\forall X,Y \in S$, ($X,Y$) $\notin \C$. The status of arguments is then evaluated as follows:

\begin{definition}\label{Dung semantics} Let ($\A,\C$) be a \emph{AF}. For any
$X \in \A$, $X$ is acceptable with respect to some $S \subseteq \A$ iff $\forall Y$ s.t. ($Y,X$) $\in \C$ implies $\exists Z \in S$ s.t. ($Z,Y$) $\in \C$. Let $S \subseteq \A$ be \emph{conflict free}. Then:
\begin{itemize}
  \item $S$ is an \emph{admissible} extension iff $X \in S$ implies $X$ is acceptable w.r.t. $S$;\\[-17pt]
  \item $S$ is a \emph{complete} extension iff $X \in S$ whenever $X$ is acceptable w.r.t. $S$;\\[-17pt]
  \item $S$ is a preferred extension iff it is a set inclusion maximal complete extension;\\[-17pt]
  \item $S$ is the grounded extension iff it is the set inclusion minimal complete extension;\\[-17pt]
  \item $S$ is a stable extension iff it is preferred and $\forall Y \notin S$, $\exists X \in S$ s.t. ($X,Y$) $\in \C$.
\end{itemize}

\noindent For $T \in \{$complete, preferred, grounded, stable$\}$, $X$ is \emph{sceptically} or \emph{credulously} justified under the $T$ semantics if $X$ belongs to all, respectively at least one, $T$ extension.
\end{definition}

A number of works \cite{AmCay, BC03, ModgilAIJ} augment \emph{AF}s to formalise the role of the relative strengths of arguments at the \emph{abstract} level. The basic idea in all these works is that an attack by $X$ on $Y$ succeeds as a \emph{defeat} only if $Y$ is not stronger than $X$. For example, preference-based \emph{AF}s (\emph{PAF}s) \cite{AmCay} are tuples $(\A,\C,\preceq)$, where given the preordering $\preceq$ $ \subseteq \A \times \A$, $Y$ is stronger than $X$ iff $Y$ is strictly preferred to $X$ ($X \prec Y$ iff $X \preceq Y$ and $Y \npreceq X$). In \cite{ModgilAIJ}, preferences between arguments are not based on a given preordering, but rather are themselves defeasible and possibly conflicting, and so are themselves the conclusions of arguments. In \cite{hp10aspicJAC}'s \emph{ASPIC}$^+$ framework, arguments are defined by strict and defeasible rules and premises expressed in some abstract language. Attacks between arguments are defined, and a preference relation over arguments is used to derive a defeat relation. Unlike \emph{PAF}s and \cite{BC03}'s value based \emph{AF}s, \emph{ASPIC}$^+$'s use of preferences to define defeat takes the structure of arguments into account.

In all the above approaches, the justified arguments are then evaluated on the basis of the derived defeat relation, rather than the original attack relation. In other words, a conflict free set is one that contains no two defeating arguments, and the defeat relation replaces the attack relation $\C$ in Definition \ref{Dung semantics}.

Prior to discussing the role of, and relationship between attacks, preferences and defeats, recall that
Section \ref{SectionIntroduction} discussed how abstract argumentation and argument game proof theories: a) provide dialectical semantics, respectively proof theories, for non-monotonic reasoning, where; b) the abstract modelling of the process whereby arguments are submitted to attack and defend, comports with intuitive human modes of reasoning and debate. Thus, the added value of argumentation is in large part due to its potential for facilitating dynamic, interactive and heterogenous (both automated and human) reasoning in the presence of uncertain and conflicting knowledge.

It is in this context that we motivate criteria for defining attack relations, the role of preferences, and a new approach to defining the extensions of a framework in terms of both defeat \emph{and} attack relations. In what follows we assume that arguments are built from strict (i.e., deductive) and defeasible inference rules (a distinction that is made more precise in Section \ref{Section-ASPIC+} and further discussed in Section \ref{SectionRelatedWork}), and refer to an argument's \emph{conclusion} following from its constituent \emph{premises} and \emph{rule} applications (referred to collectively as the argument's \emph{support}).

\subsection{The Two Roles of Attacks}\label{Section-ArgumentationAndLogic-DefiningAttacks}

Attacks play two roles. Firstly, that $X$ attacks $Y$, is an abstract, declarative representation of the mutual incompatibility of the information contained in the attacking arguments. Secondly, the attack abstractly characterises the dialectical use of $X$ as a counter-argument to $Y$. The former role suggests a necessary condition for specifying an attack between $X$ and $Y$, namely, that they contain mutually incompatible information.
However the second role suggests that this condition is not sufficient; attacks should also be defined in such a way as to reflect their use in debate and discussion. Intuitively, if $Y$ is proposed as an argument, then in seeking a counter-argument to $Y$, one seeks to construct an argument $X$ whose conclusion is in conflict with the conclusion or some supporting element of $Y$. This motivates a definition of attack according to which only an argument's \emph{final} conclusion is relevant for whether it attacks another argument.
%This precludes defining an attack from $X$ to $Y$, where despite the arguments containing incompatible information, $X$'s conclusion in not incompatible with some supporting element of $Y$.
For example, consider argument $Y$ concluding \emph{Tweety flies}, supported by the premise \emph{Tweety is a bird} and the defeasible rule that \emph{birds fly}. Consider also argument $X$ concluding \emph{Tweety does not fly}, supported by the premise and defeasible rule \emph{Tweety is a penguin} and \emph{penguins don't fly}. Then it is reasonable to say that $X$ and $Y$ attack each other, but if $X$ is extended with the defeasible rule that \emph{non-flying animals do not have wings}, to obtain $X'$ claiming \emph{Tweety does not have wings}, then $X'$ should not attack $Y$, since its final conclusion does not conflict with any element of $Y$. Intuitively, $X'$ would not be moved as a counter-argument to $Y$; rather it is the sub-argument $X$ of $X'$ that would be moved. An additional reason for not allowing $X'$ to attack $Y$ is that otherwise any continuation of $X$ (and not just $X'$) with further inferences would also attack $Y$, which may dramatically increase the number of attacks defined by a theory (and thus the computational expense incurred in evaluating the justified arguments). For example, if arguments can be constructed with the full power of classical logic, then this would yield an infinite number of attackers of $Y$.

A final requirement for attacks is that they should only be targeted at fallible elements of an argument, i.e., only on uncertain premises or defeasible inferences. In particular, conclusions of deductive inferences in an argument cannot be attacked. This should be obvious since the very meaning of deductive
% (that is, standard logical)
 inference is that the truth of the premises of a deductive inference \emph{guarantees} the truth of its conclusion. Any disagreement with the conclusion of a deductive inference should therefore be expressed as an attack on either uncertain premises or defeasible subarguments of the attacked argument. This informal analysis is supported by recent formal results \cite{c+a07,g+h11} showing that allowing attacks on deductive inferences leads to violation of rationality postulates.

\subsection{Distinguishing Preference Dependent and Independent Attacks}\label{Section-ArgumentationAndLogic-RoleOfPreferences}

We now motivate the distinction between preference dependent and preference independent attacks. Firstly, note that we assumed above that arguments have three elements: a conclusion, a set of premises,
and inference steps from the premises to the conclusion. Arguments can then
in general be attacked in three ways: on their premises, on their
conclusion and on their inference steps. We also argue that in practice, preferences are often
used in argumentation, so that a formal framework that aims to bridge the gap with human modes of argumentation, should accommodate preferences
as first class citizens, instead of implicitly
encoding them by other means (such as with explicit exception or
applicability predicates). We now discuss to what extent these three types of attacks require
preferences to succeed as defeats. To start with, we claim that attacks on conclusions should be resolved with preferences, since such attacks
arise because of conflicting reasons for and against a
conclusion. In such cases, explicit
preferences are used to resolve such conflicts, e,g, based on rule priorities in legal
systems, orderings on desires or values in practical reasoning, or
reliability orderings in epistemic reasoning. For example, consider the above symmetrically attacking arguments $X$ and $Y$ respectively concluding \emph{Tweety does not fly} and \emph{Tweety flies}. Based on the specificity principle's prioritisation of properties of sub-classes over super-classes, one preferentially concludes \emph{Tweety does not fly}. The use of the specificity principle can be modelled at the meta-level (i.e., meta to the object-level logic in which arguments $X$ and $Y$ are constructed), as a preference for $X$ over $Y$, so that $X$ asymmetrically defeats $Y$.

However, assuming sufficient expressive power, one could also encode this metalevel arbitration of the conflict in the object level logic, as undercutting attacks on inference steps \cite{Pol87}. The inferential step licensed by the rule $bf$ = \emph{birds fly}, is blocked by a rule $pNf$ that states that if the bird is a penguin, then the inferential step encoded in the rule $bf$, is not valid. This suggests the use of undercut attacks on an inference step for yielding the same results as those obtained through the use of preferences, %in cases where the rationale for the preference is available, and can be suitably expressed in the object level. They not only do so
in a way that makes the rationale for preference application more explicit.
%(although note that the undercut does not \emph{explicitly} model the use of the specificity principle -- that penguins are a subclass of birds and  properties of subclasses $\ldots$ etc -- although clearly less is left implicit as compared with the use of a preference ordering),
Undercuts also yield effects that cannot be exclusively effected through preferences. Consider Pollock's classic example \cite{Pol87} in which \emph{there is a red light shining} undercuts the rule \emph{if an object looks red then it is red}, so blocking the inference from \emph{there is an object that looks red}, to the conclusion \emph{the object is red}. Here, the undercut effectively expresses a preference for not drawing the inference over drawing the inference; something that cannot be expressed as a preference ordering over arguments.% \footnote{How would one express an ordering over an argument for not concluding X over one for concluding X ?}.

We conclude that when specifying an attack by $Z$ on $Y$, based on $Z$'s conclusion undercutting a rule in $Y$, the attacking argument is first and foremost expressing reasons for preferring not to infer $Y$'s conclusion over inferring $Y$'s conclusion. Such attacks should therefore be `preference independent', since qualifying the success of such an attack (as a defeat) as being contingent on $Y$ not being preferred to $Z$, would be to contradict the preference that is effectively expressed by the attack itself. In other words, a priority relation that regards the undercut rule as of higher priority than the undercutting rule cannot be regarded as a preference for drawing the inference over not drawing the inference, since the opposite preference (for not drawing, over drawing, the inference) is already expressed in the undercutter.
Thus, we argue for a distinction between \emph{preference dependent} and \emph{preference independent} attacks, where undercuts fall into the latter category. Note that this does not preclude that a third argument $Z'$ attacks $Z$'s conclusion that $Y$'s conclusion should not be inferred, where $Z'$'s attack is preference dependent.

Finally, we claim that whether attacks on premises are preference-dependent, depends
on the nature of the premise that is attacked. Normally, preferences are
needed except if the premise states some kind assumption in the absence
of evidence to the contrary, as, for example, negation as failure
assumptions in logic programming. If $Y$ makes use of a negation as failure assumption of the form $\thicksim \alpha$, denoting that `$\alpha$ is not provable', then an argument $Z$ concluding $\alpha$, preference independent attacks $Y$, since the construction of $Z$ is contingent on the non-provability of $\alpha$, i.e., the absence of an acceptable argument $Y$ concluding $\alpha$.

%%%
%\marginpar{Last para deleted}
%%%
%We also advocate the preference independence of attacks on \emph{assumptions} (e.g., as described in \cite{BDKT97}). In general, `ordinary' premises used in the construction of arguments may be facts considered true given some evidential basis (e.g., based on perceptions). `Assumption' premises may be assumed without evidential basis, and so their very use in the construction of arguments is contingent on the absence of evidence to the contrary. Hence if $Y$ makes use of an assumption $a$, and an argument $Z$ concludes that $a$ does not hold, then $Z$ \emph{preference independent} attacks $Y$, since \emph{the very use of the assumption in} $Y$ is contingent on the absence of any such $Z$. For example, if the antecedent of \emph{birds fly} is augmented with an assumption that the bird to which the rule is applied, is assumed normal with respect to flying, then the validity
%of $Y$ concluding \emph{Tweety flies}, is contingent on the absence of an argument $Z$ claiming that \emph{Tweety is not normal with respect to flying} given that \emph{Tweety is a penguin} and the rule \emph{penguins are abnormal with respect to flying}. Hence $Z$ preference independent attacks $Y$.

\subsection{The Distinct Uses of Attacks and Defeats}\label{Section-ArgumentationAndLogic-AttacksAndDefeats}

To recap, attacks encode the mutual incompatibility of the information contained in the attacking and attacked arguments, in a way that accounts for their dialectical use. In turn, the dialectical use of attacks as defeats may or may not be contingent on the preferences defined over the arguments.

As described in Section \ref{SectionBackground}, existing works that account for preferences and/or values \cite{AmCay,BC03,ModgilAIJ}, including \cite{hp10aspicJAC}'s \emph{ASPIC}$^+$ framework, define conflict-free and acceptable sets of arguments with respect to the defeats. However, we argue that defining conflict free sets in terms of defeats is conceptually wrong. Since attacks indicate the mutual incompatibility of the information contained in the attacking and attacked arguments, then intuitively one should continue to define conflict-free sets in terms of those that do not contain attacking arguments. Defeats only encode the preference dependent use of attacks in the dialectical evaluation of the acceptability of arguments. They have no bearing on whether one argument can be said to be logically incompatible with another, but rather whether the attack can be validly employed in a dialectical setting.

In the following section, we therefore re-define \cite{hp10aspicJAC}'s \emph{ASPIC}$^+$ notion of a conflict free set, as one in which no two arguments \emph{attack} rather than defeat. We then examine the implications of this in Section~\ref{Section-ASPIC+R-RationalityPostulates}.

%We will show that satisfaction of \cite{c+a07}'s consistency postulates is more readily shown than in \cite{hp10aspicJAC}, but that intuitive assumptions on the nature of the preference relation are required in order to show some basic properties of Dung frameworks.

\section{The ASPIC+ Framework}\label{Section-ASPIC+}

In this section we review \cite{hp10aspicJAC}'s \emph{ASPIC}$^+$ framework in light of the criteria and requirements enumerated in Sections \ref{Section-ArgumentationAndLogic-DefiningAttacks} and \ref{Section-ArgumentationAndLogic-RoleOfPreferences}.
%Specifically, we present the \emph{ASPIC}$^+$ notions of arguments and preference dependent and independent attacks.
We also modify the framework in two ways: 1) we change the definition of conflict free, as proposed above; 2) we further generalise \emph{ASPIC}$^+$ so as to capture deductive approaches to argumentation \cite{a+b09,a+b10,AmCay,b+h08}. In addition, we simplify some of  \cite{hp10aspicJAC}'s notations and definitions.

\subsection{ASPIC+ Arguments}\label{Section-ASPIC+-Arguments}

The \emph{ASPIC}$^+$ framework defines arguments, as in \cite{Vre97}, as inference trees formed by applying strict or defeasible inference rules to premises that are well-formed formulae (wff) in some logical language. The distinction between two kinds of inference rules is taken from  \cite{Pol87,LS89,Pol95,Vre97}. Informally, if an inference rule's antecedents are accepted, then if the rule is strict, its consequent must be accepted \emph{no matter what}, while if the rule is defeasible, its consequent must be accepted \emph{if there are no good reasons not to accept it}.  Arguments can be attacked on their (non-axiom) premises and on their applications of defeasible inference rules.  Some attacks succeed as \emph{defeats}, which is partly determined by preferences. The acceptability status of arguments is then defined by applying any of \cite{dung95}'s semantics for abstract argumentation frameworks to the resulting set of arguments with its defeat relation.

We emphasise that \emph{ASPIC}$^+$ is not a system but a framework  for specifying systems. It defines the notion of an abstract \emph{argumentation system} (a notion adapted from \cite{Vre97}) as a structure consisting of a logical language $\LA$ with a binary relation $\mbox{}^{-}$ , a naming convention $n$ for defeasible rules and a set $\R$ consisting of two subsets $\R_s$ and $\R_d$ of strict  and defeasible inference rules. (As is usual, inference rules are defined over the language $\LA$, and are not elements in the language.) \emph{ASPIC}$^+$ as a framework does not make any assumptions on how these elements are defined in a given argumentation system (the idea to abstract from the precise nature of $\LA/\R$ is taken from \cite{LS89,Vre97,BDKT97} while the idea to abstract from $\mbox{}^{-}$ and $n$ is taken from  \cite{BDKT97} and \cite{Pol95}, respectively).

\ASPIC's inference rules can be used in two ways: they could encode domain-specific information but they could also express general laws of reasoning. When used in the latter way, the defeasible rules could, for example, express argument schemes \cite{Wal96}, while the strict rules could be determined by the choice of the logical language $\LA$: its formal semantics will then tell which inference rules over $\LA$ are valid and can therefore be added to $\R_s$. If the strict rules are thus chosen then they could consist, for example, of all classically valid inferences or more generally conform to any Tarskian consequence notion (cf.\  \cite{a+b09}).  Notice that inclusion of defeasible rules in \ASPIC\ requires some explanation, given that much current work formalises construction of arguments as deductive \cite{a+b09,a+b10}, and in particular classical \cite{b+h08,g+h11} inference. We justify the need for inclusion of defeasible inference rules in Section \ref{SectionRelatedWork}.

As just explained, the basic notion of \emph{ASPIC}$^+$ is that of an argumentation system.
Arguments are then constructed with respect to a knowledge base. Definitions of these are taken
from \cite{hp10aspicJAC} (with some modifications that will be subsequently described).'

\begin{definition}\label{argumentation-system}[\textbf{\emph{ASPIC}$^+$ argumentation system}]
An \emph{argumentation system} is a tuple $AS = (\LA,\mbox{}^{-},\R,n)$ where:\\[-17pt]

\bit

\item $\LA$ is a logical language.\\[-17pt]
\item $\mbox{}^{-}$ is a function from $\LA$ to $2^{\LA}$, such that:
\\[3pt]
$\bullet$ $\varphi$ is a \emph{contrary} of $\psi$ if $\varphi \in \con{\psi}$, $\psi \not\in
\con{\varphi}$;
\\[3pt]
$\bullet$ $\varphi$ is a \emph{contradictory} of $\psi$ (denoted by `$\varphi = -\psi$'), if
$\varphi \in \con{\psi}$, $\psi \in
\con{\varphi}$;
\\[3pt]
$\bullet$ each $\varphi \in \LA$ has at least one contradictory.
\\[-17pt]
\item $\R = \R_s \cup \R_d$ is a set of strict ($\R_s$) and defeasible
($\R_d$) inference rules of the form $\varphi_1$,
\ldots, $\varphi_n$ $\to$ $\varphi$ and  $\varphi_1$, \ldots,
$\varphi_n$ $\Rightarrow$ $\varphi$ respectively (where $\varphi_i,\varphi$ are meta-variables ranging over wff in $\LA$), and $\R_s \cap \R_d = \emptyset$.\\[-17pt]
\item $n : \R_d \longrightarrow \LA$ is a naming convention for defeasible rules.
%\item $\leq$ is a preordering on $\R_d$.
\eit
%Henceforth, a set $S$ $\subseteq$ $\LA$ is said to be consistent iff
%$\nexists$ $\psi$, $\varphi$ $\in$ $S$ such that $\psi \in
%\con{\varphi}$, otherwise it is \emph{inconsistent}.
%We say that $\mbox{}^{-}$ \emph{corresponds to negation} iff $\LA$ contains a connective $\neg$ such that
%%%
%\marginpar{Delete ``corresponds to negation''?}
%%%
%$\psi \in \con{\varphi}$ just in case $\psi = \neg \varphi$ or $\varphi = \neg \psi$. {\bf *** HP: We could delete the last sentence to save space. We only use it directly after the definition of consistency. ***}
\end{definition}

Intuitively, contraries can be used to model well-known constructs like negation as failure in logic programming or consistency checks in default logic. Note that we illustrate requirements for the asymmetric notion of contrary (as opposed to the more familiar symmetric notion of contradictory associated standardly with negation) in Section \ref{Section-ASPIC+-AttacksAndDefeats}. Note also that in previous publications on  \emph{ASPIC}$^+$ (including \cite{hp10aspicJAC}) the idea of a naming convention $n$ was instead informally introduced  when defining undercutting attack (see Definition~\ref{DefAttacks} below).  Informally,  $n(r)$ is a wff in $\LA$ which says that the defeasible rule $r \in \R$ is applicable.
\begin{definition}\label{consistency}
For any $S \subseteq \LA$, let the \emph{closure of $S$ under strict rules}, denoted $Cl_{R_s}(S)$, be the smallest set containing $S$ and the consequent of any strict rule in $\R_s$ whose antecedents are in $Cl_{R_s}(S)$. Then a  set $S$ $\subseteq$ $\LA$ is
\bit
\item \emph{directly consistent} iff $\nexists$ $\psi$, $\varphi$ $\in$ $S$ such that $\psi \in \con{\varphi}$
\item \emph{indirectly consistent} iff $Cl_{R_s}(S)$ is directly consistent.
\eit
\end{definition}
This definition is generalised from \cite{c+a07}, in which these two notions of consistency were defined for the special case where  $\mbox{}^{-}$ corresponds to negation.
%{\bf *** HP: We could delete the next two sentences to save space. ***} Moreover, in \cite{c+a07}, and also in \cite{hp10aspicJAC,ModPrakIJCAI}, only direct consistency was explicitly defined (and called ``consistency''), while indirect consistency was implicitly defined in the definition of the rationality postulate of indirect consistency. Note that the definition of indirect consistency is parametrised by the choice of $\R_s$ and $\mbox{}^{-}$.
%
\begin{definition}\label{Definition-knowledge-base}[\textbf{\emph{ASPIC}$^+$ knowledge base}] A \emph{knowledge base} in an argumentation system
$(\LA,\mbox{}^{-},\R,n)$ is a set $\K \subseteq
\LA$ consisting of two disjoint subsets  $\K_n$ (the \emph{axioms}) and $\K_p$ (the \emph{ordinary premises}).
\end{definition}

% $\begin{tiny}\overset{a}\end{tiny} {\rightharpoonup}$

%
%Henceforth, we use the term `premise' to denote any element in $\K$.
%%%
%\marginpar{$n$ (and later an undercutter) added to the example.}
%%%

Intuitively, the axioms are certain knowledge and thus cannot be attacked, whereas the ordinary premises are uncertain and thus can be attacked. The distinction between ordinary premises and axiom premises is needed to capture systems like, for instance, Pollock's system \cite{Pol87}, which does not allow attacks on premises, and which therefore need to be modelled as axiom premises. In \cite{hp10aspicJAC}, the knowledge base was also assumed to have \emph{issue} and \emph{assumption} premises, which were used, respectively, to prove that Carneades \cite{gpw07} and assumption-based argumentation \cite{BDKT97} are special cases of \ASPIC. In the present paper we omit issue premises for simplicity while, as further discussed below in Section~\ref{Section Comparison with General Frameworks for Argumentation}, \cite{hp10aspicJAC}'s result on assumption-based argumentation also holds if all premises are ordinary instead of assumption premises. Furthermore, in previous \ASPIC publications (including \cite{hp10aspicJAC}) we included preorderings on $\R_d$ and $\K_p$ in the definitions of argumentation systems and knowledge bases respectively. We remove references to these preorderings in the above general definitions, and only introduce them when they are required for defining preference orderings over arguments.

\begin{example}\label{ExampleIllustration1} Let $(\LA,\mbox{}^{-},\R,n)$ be an argumentation system where:\\[-14pt]
\begin{itemize}
\item $\LA$ is a language of propositional literals, composed from a set of propositional atoms $\{a,b,c, \dots\}$ and the symbols $\neg$ and $\sim$ respectively denoting strong and weak negation (i.e., negation as failure). $\alpha$ is a strong literal if $\alpha$ is a propositional atom or of the form $\neg \beta$ where $\beta$ is a propositional atom (strong negation cannot be nested). $\alpha$ is a wff of $\LA$, if $\alpha$ is a strong literal or of the form $\sim \beta$ where $\beta$ is a strong literal (weak negation cannot be nested).\\[-17pt]
 \item $\alpha \in \con{\beta}$ iff (1) $\alpha$ is of the form $\neg
\beta$ or $\beta$ is of the form $\neg \alpha$; or (2) $\beta$ is of the form
$\thicksim \alpha$ (i.e., for any wff $\alpha$, $\alpha$ and $\neg \alpha$ are contradictories and $\alpha$ is a contrary of $\sim \alpha$).\\[-17pt]
 \item  $\R_s = \{t,q \rightarrow \neg p\}$,  $\R_d = \{\sim s \Rightarrow t; r \Rightarrow q; a \Rightarrow p\}$\\[-17pt]
\item $n(\sim s \Rightarrow t) = d_1$, $n(r \Rightarrow q) = d_2$, $n(a \Rightarrow p) = d_3$\\[-14pt]
% \item  $\leq = \{ r \Rightarrow q < a \Rightarrow p \}$ (for simplicity we will leave $l \leq l'$ and $l \leq' l'$ relations implicit when specifying these orderings)\\[-14pt]
\end{itemize}
Furthermore, $\K$ is the knowledge base such that $\K_n = \emptyset$ and $\K_p = \{a,r, \neg r,\sim s\}$.
\end{example}

Arguments are defined below (as in \cite{hp10aspicJAC}), together with some associated notions. Informally, for any argument $A$, $\mathtt{Prem}$ returns all the
formulas of $\K$ (\emph{premises}) used to build $A$,
%$\mathtt{Prop}$ returns all the propositions used in that argument,
%{\bf *** is Prop used at all? ***}
$\mathtt{Conc}$ returns $A$'s conclusion, $\mathtt{Sub}$ returns all of
$A$'s sub-arguments, $\mathtt{DefRules}$ and $\mathtt{StRules}$ respectively return all defeasible and all strict rules in $A$, and $\mathtt{TopRule}(A)$ returns the last rule applied in $A$.

%
%\begin{notation}\label{NotationPremises} For any argument $A$:
%
% \begin{itemize}
%   \item $\mathtt{LastDefRules}(A) = \emptyset$ if $\mathtt{DefRules}(A) = \emptyset$.
%   \item If $A$ = $A_1,\ldots, A_n \Imp \phi$ then $\mathtt{LastDefRules}(A)$ = $\mathtt{Conc}(A_1),\ldots, \mathtt{Conc}(A_n) \Imp \phi$,
%   otherwise $\mathtt{LastDefRules}(A) = \mathtt{LastDefRules}(A_1) \cup \ldots \cup \mathtt{LastDefRules}(A_n)$.
%
% \end{itemize}
%
%
%\end{notation}

\begin{definition}\label{arg}[\textbf{\emph{ASPIC}$^+$ arguments}]
An \emph{argument} $A$ on the basis of a knowledge base $\K$ in
an argumentation system $(\LA,\mbox{}^{-},\R,n)$ is:
\begin{enumerate}
    \item\label{arg1} $\varphi$ if $\varphi$ $\in$ $\K$ with:
                $\mathtt{Prem}(A)       = \{\varphi\}$;
%                $\mathtt{Prop}(A)       = \{\varphi\}$, \\
                $\mathtt{Conc}(A)       = \varphi$;
                $\mathtt{Sub}(A)        = \{\varphi\}$;
                $\mathtt{Rules}(A)$  = $\emptyset$;
              $\mathtt{TopRule}(A)$  = undefined.\\[-15pt]
\item\label{arg2} $A_1, \ldots A_n \rightarrow \psi$ if
$A_1, \ldots, A_n$ are arguments such that
            there exists a strict rule
            $\mathtt{Conc}(A_1), \ldots,\mathtt{Conc}(A_n)
\rightarrow \psi$
            in $\R_s$. \\
 $A_1, \ldots A_n \Rightarrow \psi$ if
$A_1, \ldots, A_n$ are arguments such that
            there exists a defeasible rule
            $\mathtt{Conc}(A_1), \ldots,\mathtt{Conc}(A_n) \Rightarrow \psi$
            in $\R_d$. \\
            $\mathtt{Prem}(A)\footnote{Note that all premises in ASPIC+ arguments are used in deriving its conclusion, so enforcing a notion of relevance
            analgous to the subset minimality condition requirement on premises in classical logic approaches to argumentation (see Section \ref{sectionALstrict}).}       = \mathtt{Prem}(A_1) \cup \ldots
\cup \mathtt{Prem}(A_n)$, \\
%                $\mathtt{Prop}(A)       = \mathtt{Prop}(A_1) \cup
%\ldots \cup \mathtt{Prop}(A_n) \cup \{\psi\}$,    \\
                $\mathtt{Conc}(A)       = \psi$, \\
                $\mathtt{Sub}(A)        = \mathtt{Sub}(A_1) \cup \ldots
\cup \mathtt{Sub}(A_n) \cup \{A\}$. Note that $A_1 \ldots A_n$ are referred to as the \emph{proper} sub-arguments of $A$\\
                 $\mathtt{Rules}(A)   = \mathtt{Rules}(A_1) \cup \ldots
\cup \mathtt{Rules}(A_n) \cup
                  \{\mathtt{Conc}(A_1), \ldots, \mathtt{Conc}(A_n)
\rightarrow$/$\Rightarrow \psi\}$ \\
                $\mathtt{DefRules}(A)$  = $\{r | r \in \mathtt{Rules}(A), r \in \R_d\}$\\
                $\mathtt{StRules}(A)$  = $\{r | r \in \mathtt{Rules}(A), r \in \R_s\}$\\
              $\mathtt{TopRule}(A)$  = $\mathtt{Conc}(A_1), \ldots
\mathtt{Conc}(A_n) \rightarrow$/$\Rightarrow \psi$

           \end{enumerate}
Furthermore, for any argument $A$:\\
$\bullet$ $\mathtt{Prem}_n(A)$ = $\mathtt{Prem}(A) \cap \K_n$ and $\mathtt{Prem}_{p}(A)$ = $\mathtt{Prem}(A) \cap \K_p$
\\[3pt]
$\bullet$ If $\mathtt{DefRules}(A) = \emptyset$, then $\mathtt{LastDefRules}(A) = \emptyset$, else;\\
 if $A$ = $A_1,\ldots, A_n \Imp \phi$ then $\mathtt{LastDefRules}(A)$ = $\{\mathtt{Conc}(A_1),\ldots, \mathtt{Conc}(A_n) \Imp \phi\}$, otherwise $\mathtt{LastDefRules}(A) = \mathtt{LastDefRules}(A_1) \cup \ldots \cup \mathtt{LastDefRules}(A_n)$.
\\[3pt]
$\bullet$ $A$
is: \emph{strict} if $\mathtt{DefRules}(A) = \emptyset$;
\emph{defeasible} if $\mathtt{DefRules}(A) \not= \emptyset$; \emph{firm}
if $\mathtt{Prem}(A)   \subseteq \K_n$; \emph{plausible} if
$\mathtt{Prem}(A)   \not\subseteq \K_n$; \emph{fallible} if $A$ is plausible or defeasible; \emph{finite} if $\mathtt{Rules}(A)$ is finite.\footnote{As explained in \cite{hp10aspicJAC}, Definition~\ref{arg} allows for arguments that are `backwards' infinite in that they do not `bottom' out in premises from the knowledge base.}
\end{definition}

Henceforth, we may employ the following notation for arguments.

\begin{notation}\label{NotationArguments}[\textbf{Notation for Arguments}]\\[-17pt]

\begin{enumerate}

    \item $S \vdash \varphi$ may be written to denote that there exists a strict argument $A$ such that $\mathtt{Conc}(A) = \varphi$, with all premises taken from $S$ (i.e., $\mathtt{Prem}(A) \subseteq S$).\\[-17pt]
         %and $S \nc \varphi$ may be written to denote that there
%exists a defeasible argument $A$ such that $\mathtt{Conc}(A) = \varphi$, with all premises taken from $S$.
  \item Arguments may be written as lists of premises and rules separated by semi-colons, or in the case that an argument has a top rule, we may write such an argument as the top rule with the antecedents replaced by the names of the sub-arguments that conclude the antecedents. For example, we may write $A$ = $[s$; $s\Rightarrow r;$ $q$; $r,q \rightarrow \neg p]$ or $A$ = [$A_1,A_2 \rightarrow \neg p$], where $A_1$ = $[s;$ $s\Rightarrow r]$, $A_2$ = $[q]$.\\[-17pt]
  \item Letting $\Gamma$ be a set of arguments, we may as an abuse of notation write $\mathtt{F}(\Gamma)$ to denote $\bigcup_{A \in \Gamma} \mathtt{F}(A)$, where $\mathtt{F} \in \{\mathtt{Prem}$, $\mathtt{Conc}$, $\mathtt{Sub}$, $\mathtt{Rules}$, $\mathtt{TopRule}$, $\mathtt{DefRules},$ $\mathtt{StRules}\}$

\end{enumerate}

\end{notation}

\begin{example}\label{ExampleArgs} The arguments (shown in Figure \ref{FigASPICArgs}) defined on the basis of the knowledge base and argumentation system in Example \ref{ExampleIllustration1} are: $A'$ = $[a]$, $A$ = $[A' \Rightarrow p]$,  $B_1$ = $[\sim s]$, $B_1'$ = $[B_1 \Rightarrow t]$,  $B_2$ = $[r]$, $B_2'$ = $[B_2 \Rightarrow q]$, $B$ = $[B_1' , B_2' \rightarrow \neg p]$, $C$ = $[\neg r]$.\\[2pt]
Furthermore, $\mathtt{Prem}(B) = \{\sim s,r\}$; $\mathtt{Conc}(B) = \neg p $; $\mathtt{Sub}(B) = \{B_1,B_2,B_1' , B_2' \}$; $\mathtt{TopRule}(B) = t,q \imp \neg p$; $\mathtt{DefRules}(B) = \{\sim s \Imp t, r \Imp q \}$; $\mathtt{StRules}(B) = \{ t,q \imp \neg p\}$.

\begin{figure}[h]
\centering
\includegraphics[width=4in]{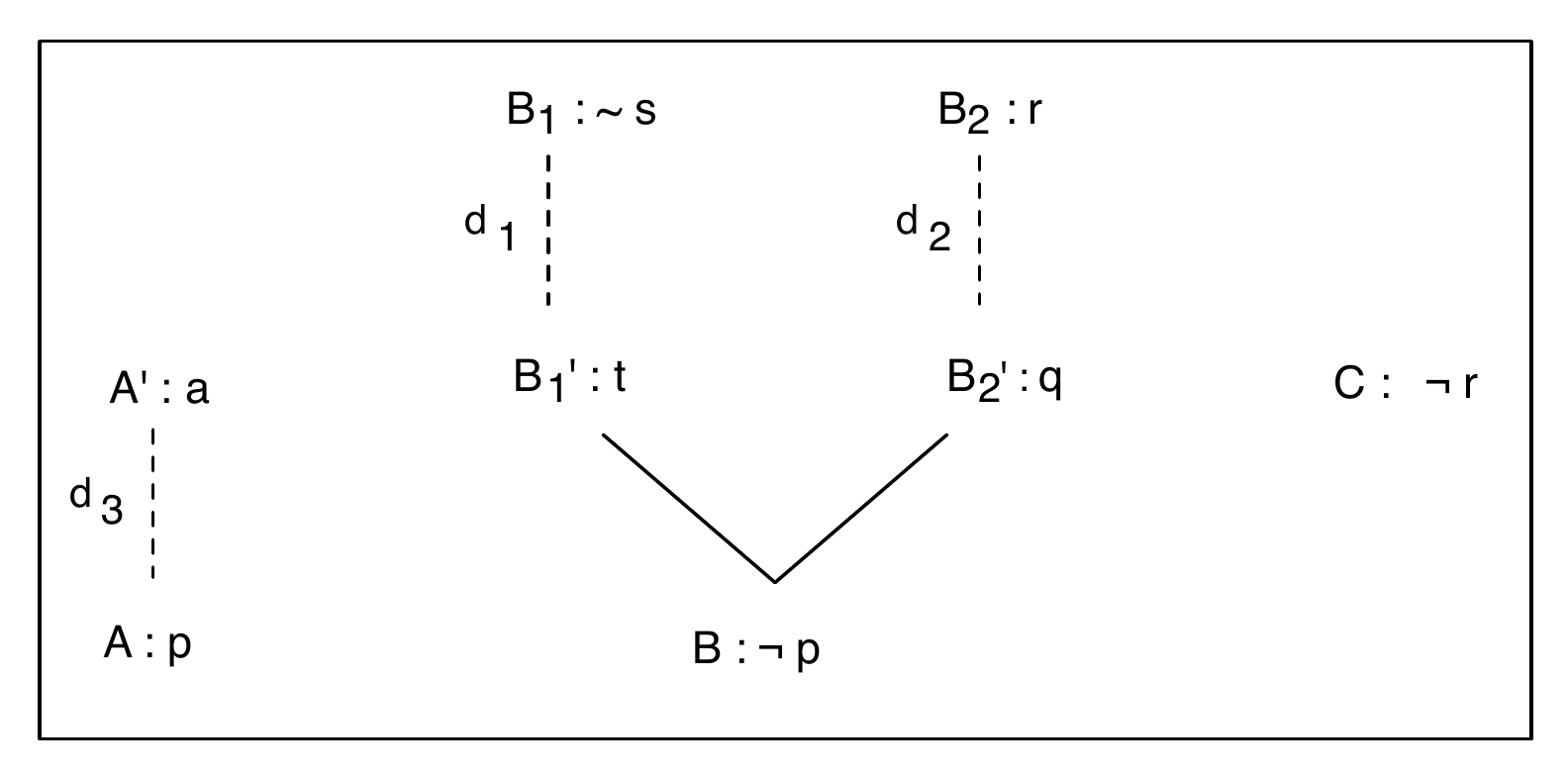}
%\caption{Example \ref{ExampleIllustration2}'s \emph{ASPIC}$^+$ arguments and attacks are shown on the left ($D$ and $E$ are extra arguments not defined by the knowledge base and argumentation system in Example \ref{ExampleIllustration1}. The defeat graph is shown on the right.}\label{FigASPIC+Illustration}
\caption{ \emph{ASPIC}$^+$ arguments and their conclusions, with dashed and solid lines respectively representing application of defeasible and strict inference rules.\label{FigASPICArgs}}
\end{figure}

\end{example}

We now adapt \cite{hp10aspicJAC}'s above definition of an argument so as to consider a special class of arguments whose premises are `c-consistent' (for ``contradictory-consistent'). We thus generalise
%\emph{ASPIC}$^+$ so as to subsume \cite{a+b09, a+b10}'s recent abstract framework for structured
%argumentation based on Tarski's notion of an abstract logic. We thus also generalise
\emph{ASPIC}$^+$ so as to accommodate deductive approaches to argumentation \cite{a+b09,a+b10,AmCay,b+h08} that require that the arguments defined by the instantiating logic have consistent premises.
%In Section \ref{SectionDefininingClassicalLogicInstantiations} we formally define classical logic instantiations of \emph{ASPIC}$^+$.
%relate \emph{ASPIC}$^+$ to \cite{a+b09, a+b10}'s Tarskian logic based argumentation systems, and also formally define classical logic instantiations.

\begin{definition}\label{DefinitionC-Consistent}[\textbf{c-consistent}] A set $S$ $\subseteq$ $\LA$ is \emph{c-consistent} if for no $\varphi$ it holds that $S \vdash \varphi, -\varphi$. Otherwise $S$ is \emph{c-inconsistent}. We say that $S \subseteq \LA$ is minimally c-inconsistent iff $S$ is c-inconsistent and $\forall S' \subset S$, $S'$ is c-consistent.

\end{definition}

Note that we use the term `c-consistent' to distinguish the notion of consistency in Definition~\ref{argumentation-system}. Also note that if $S \vdash \varphi,\phi$, where $\phi \in \con{\varphi}$, then $S$ can still be c-consistent. As we will see later, such situations do not arise when capturing deductive approaches in \ASPIC, as in these approaches there are no contraries, only contradictories.

\begin{definition}\label{DefinitionConsistentArguments}[\textbf{c-consistent argument}] An \emph{argument} $A$ on the basis of a knowledge base $\K$ in
an argumentation system $(\LA,\mbox{}^{-},\R,n)$, is c-consistent iff $\mathtt{Prem}(A)$ is c-consistent.
\end{definition}

\subsection{Attacks and Defeats}\label{Section-ASPIC+-AttacksAndDefeats}

We now review \cite{hp10aspicJAC}'s definition of attacks and defeats amongst arguments. An argument $A$ attacks an argument $A'$ if the conclusion of $A$ (i.e., $\mathtt{Conc}(A)$) is a contrary or contradictory of: an ordinary premise in $A'$; the consequent of a defeasible rule in $A'$, or; a defeasible inference step in $A'$. These three kinds of attack are respectively called
undermining, rebutting and undercutting attacks.

\begin{definition}\label{DefAttacks}[\textbf{\emph{ASPIC}$^+$ attacks}] $A$ \emph{attacks} $B$ iff $A$ \emph{undercuts}, \emph{rebuts} or \emph{undermines} $B$, where:\\[2pt]
$\bullet$ $A$ \emph{undercuts} argument $B$ (on $B'$) iff
$\mathtt{Conc}(A) \in  \con{n(r)}$ for some $B'$ $\in$ $\mathtt{Sub}(B)$
such that $B'$'s top rule $r$ is defeasible.\\[1pt]
$\bullet$ $A$ \emph{rebuts} argument $B$ (on $B'$) iff
$\mathtt{Conc}(A)\in \con{\varphi}$ for some  $B'$ $\in$
$\mathtt{Sub}(B)$ of the form $B''_1, \ldots, B''_n \Rightarrow
\varphi$. In such a case $A$ \emph{contrary-rebuts} $B$ iff
$\mathtt{Conc}(A)$ is a contrary of $\varphi$.\\[2pt]
$\bullet$  Argument $A$ \emph{undermines} $B$ (on $B'$) iff
$\mathtt{Conc}(A) \in \con{\varphi}$ for some $B'$ = $\varphi$, $\varphi$ $\in$
$\mathtt{Prem}_{p}(B)$. In such a case $A$
\emph{contrary-undermines} $B$ iff  $\mathtt{Conc}(A)$ is a contrary of
$\varphi$.

\end{definition}

%Notice that an attack on an assumption is categorised as a special case of a contrary undermining attack.

\begin{example}\label{ExampleIllustration2}[\textbf{Example \ref{ExampleArgs} continued}] \noindent For the arguments in Example \ref{ExampleArgs}, $B$ rebuts $A$ on $A$,  $C$ undermines $B$ and $B_2'$  on $B_2$,  and $C$ and $B_2$ undermine each other. The attack graph is shown in Figure \ref{FigASPIC+Illustration}a).
Notice that if in addition one had the argument $D$ = $[\sim d ; \sim d \Rightarrow s]$, then $D$ would contrary-undermine $B$ and $B_1'$ on $B_1$. Moreover, if in addition one had the argument $E$ = $[r ; r \rightarrow \neg d_3]$, then $E$ would undercut $A$ on $A$.
%The above arguments and attacks are depicted in Figure \ref{FigASPIC+Illustration}a) in which we first show the arguments with their internal structure, and then the abstract attack graph.
\end{example}

Note that Definition \ref{DefAttacks} complies with Section \ref{Section-ArgumentationAndLogic-DefiningAttacks} and \ref{Section-ArgumentationAndLogic-RoleOfPreferences}'s rationale for defining attacks. An attack originating from an argument $A$ requires that its conclusion $\mathtt{Conc}(A)$ (and not the conclusion of any sub-argument of $A$) be in conflict with some \emph{fallible} element -- i.e., some ordinary premise, or defeasible rule or conclusion of a defeasible rule -- in the attacked argument. Thus, while $B_2$ rebut-attacks $C$ in Example \ref{ExampleIllustration2}, the argument $B$, that contains $B_2$ as a sub-argument, does not attack $C$. Also,
although $A$ and $B$ have contradictory conclusions, only $B$ rebuts $A$; $A$ does not rebut $B$ as $B$'s conclusion is the consequent of a strict rule. \cite{c+a07} refer to this as a \emph{restricted rebut}, and show  for a special case of \emph{ASPIC}$^+$ that if the restriction is lifted so as to allow $A$ to rebut $B$, then this would lead to violation of of \cite{c+a07}'s rationality postulates.

%%%
%\marginpar{Technically we should now first define AT's and (c)-SAFs, since the defs $\preceq$ and defeat assume an AT and (c)-SAF.}
%%%
 \sm{Attacks can then be distinguished as to whether they are preference-dependent or preference-independent, where the formers' success as defeats is determined by the strict counterpart $\prec$ of a preference ordering $\preceq$ on the constructed arguments (where if $X \prec Y$ then $X \preceq Y$ and $Y \npreceq X$).%\footnote{Notice that we do not This assumption suffices to make $X \prec Y$ imply $Y \not\prec X$. We do not make the 'if' part of the assumption, for reasons explained below in Section 5.1.}.
 We make no assumptions on the properties of $\preceq$. In Section \ref{SubSection-ASPIC+WeakestLast} we will utilise two preorderings $\leq$  on defeasible rules and $\leq'$ on ordinary premises to give example definitions of $\prec$, but the definition of defeat does not rely on these preorderings.}
%As usual:\\[-16pt]
%%%
%%%
% \begin{itemize}
%   \item the strict counterpart $\prec$ of $\preceq$ is defined as $X \prec Y$ iff $X \preceq Y$ and $Y \npreceq X$, and;\\[-16pt]
%   \item $X \approx Y$ denotes that $X \preceq Y, Y \preceq X$.
% \end{itemize}

% The notion of a \emph{defeat} can then be defined:

%
%{\bf *** HP: I don't think the new def of defeat is fully correct, since $A$ can pref-dependent attack $B$ on one subargument and pref-independent attack $B$ on another subargument.  Moreover, the old definition is simpler and shorter. If you still want to retain the new definition, then I think it should read as follows (and here I am realy missing the dashes in e.g. ``undercut attacks''.  ***}

\begin{definition}\label{DefDefeats}[\textbf{\emph{ASPIC}$^+$ defeats}]  Let $A$ attack $B$ on $B'$. If $A$ undercut, contrary-rebut, or contrary-undermine attacks $B$ on $B'$ then $A$ is said to \emph{preference-independent} attack $B$ on $B'$, otherwise $A$ is said to \emph{preference-dependent} attack $B$ on $B'$.\\[2pt]
  \noindent Then, $A$ \textbf{defeats} $B$ iff for some $B'$ either $A$ \emph{preference-independent} attacks $B$ on $B'$, or $A$ \emph{preference-dependent} attacks $B$ on $B'$ and $A \nprec B'$.\\[2pt]
$A$ strictly \emph{defeats} $B$ iff $A$ \emph{defeats} $B$ and $B$ does not \emph{defeat} $A$
\end{definition}
%{\bf *** By the way, if you restore the old definition, then it should strictly speaking also include the words ``on some $B'$''. ***}
%\begin{definition}\label{DefDefeats}[\textbf{\emph{ASPIC}$^+$ defeats}]\marginpar{Please check modified definition Henry} Let $A$ attack $B$ on $B'$. If $A$ undercut, contrary-rebut, or contrary-undermine attacks $B$ then $A$ is said to \emph{preference-independent} attack $B$, otherwise $A$ is said to \emph{preference-dependent} attack $B$.\\[2pt]
%  \noindent Then, $A$ \textbf{defeats} $B$ iff either $A$ \emph{preference-independent} attacks $B$, or $A$ \emph{preference-dependent} attacks $B$ and $A \nprec B'$.\\[2pt]
%$A$ strictly \emph{defeats} $B$ iff $A$ \emph{defeats} $B$ and $B$ does not \emph{defeat} $A$
%\end{definition}

\begin{notation}\label{NotationAttacksDefeats} Henceforth, $\rightharpoonup$ may denote the attack relation, and $\hookrightarrow$ the defeat relation.

\end{notation}

%\marginpar{Please check modified discussion Henry} {\bf *** HP: The new discussion is OK except that I prefer to remove the reference to applicability predicates. ***}
The definition of defeats complies with Section \ref{Section-ArgumentationAndLogic-RoleOfPreferences}'s rationale for distinguishing between preference dependent and preference independent attacks. Undercuts always succeed as defeats, and so are preference independent.
% The preference independence of attacks on assumptions is ensured by declaring attacks on assumptions to be a special case of contrary-undermining attacks.
 As discussed in Section \ref{Section-ArgumentationAndLogic-RoleOfPreferences}, undercutting attacks encode an asymmetry: the use of the attacked rule named $r$ in an argument $B$ is contingent on the absence of an acceptable attacking argument $A$ with an undercutting, conclusion $\neg r$.
% {\bf *** HP: we could delete the next sentence to save space.  ***} However, the reverse does not hold, in that the use of $\neg r$ in an argument is not contingent on the absence of an argument concluding $r$.
 The notion of a contrary relation generalises the above cases of asymmetric preference independent attacks, providing for
greater flexibility in declaring formulas $\varphi$ and $\psi$ incompatible, where attacks from $\varphi$ to $\psi$ are not undercuts  but are still preference independent.\footnote{
Notice that in such cases it would be counter-intuitive to allow $\psi$ to be an axiom premise or the conclusion of a strict rule. In the rest of this paper we will therefore assume
that such cases do not arise. This assumption is formalised in Definition \ref{DefStrictDefAssumptions}.} As discussed in Section  \ref{Section-ArgumentationAndLogic-RoleOfPreferences} an example of such a cases is when $\psi$ is a negation as failure assumption. This is illustrated in Example \ref{ExampleIllustration1} in which $\alpha$ is a contrary of $\thicksim \alpha$, so that an undermining attack from an argument $A$ concluding $\alpha$ on an ordinary premise $\thicksim \alpha$ in an argument $B$, is preference independent.

%This complies with the semantic interpretation of weak negation (negation as failure), whereby the construction of argument $B$ is contingent on the non-provability of $\alpha$, i.e., the absence of an acceptable argument concluding $\alpha$. Hence $A$ preference independent attacks $B$.

%An example illustrating the preference independence of contrary attacks is given in Example
%
%We also advocate the preference independence of attacks on \emph{assumptions} (e.g., as described in \cite{BDKT97}). In general, `ordinary' premises used in the construction of arguments may be facts considered true given some evidential basis (e.g., based on perceptions). `Assumption' premises may be assumed without evidential basis, and so their very use in the construction of arguments is contingent on the absence of evidence to the contrary. Hence if $Y$ makes use of an assumption $a$, and an argument $Z$ concludes that $a$ does not hold, then $Z$ \emph{preference independent} attacks $Y$, since \emph{the very use of the assumption in} $Y$ is contingent on the absence of any such $Z$. For example, if the antecedent of \emph{birds fly} is augmented with an assumption that the bird to which the rule is applied, is assumed normal with respect to flying, then the validity
%of $Y$ concluding \emph{Tweety flies}, is contingent on the absence of an argument $Z$ claiming that \emph{Tweety is not normal with respect to flying} given that \emph{Tweety is a penguin} and the rule \emph{penguins are abnormal with respect to flying}. Hence $Z$ preference independent attacks $Y$.
%
% required is when

\begin{figure}[h]
\centering
\includegraphics[width=5in]{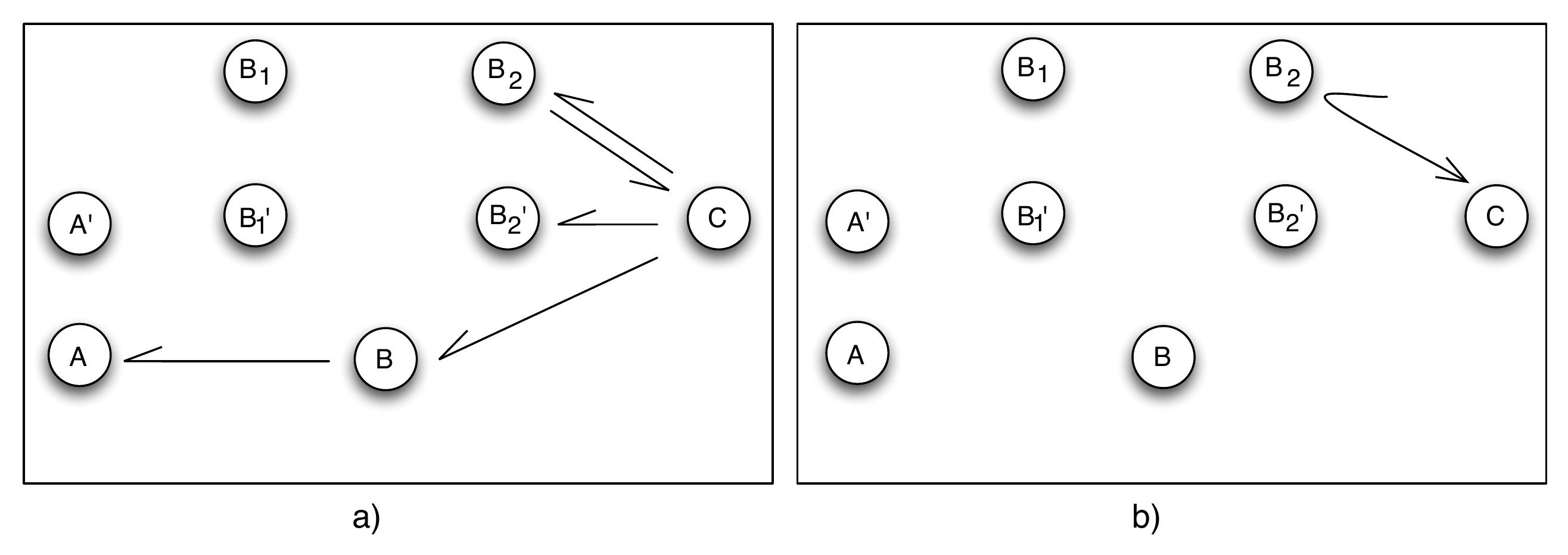}
%\caption{Example \ref{ExampleIllustration2}'s \emph{ASPIC}$^+$ arguments and attacks are shown on the left ($D$ and $E$ are extra arguments not defined by the knowledge base and argumentation system in Example \ref{ExampleIllustration1}. The defeat graph is shown on the right.}\label{FigASPIC+Illustration}
\caption{Example \ref{ExampleIllustration2}'s \emph{ASPIC}$^+$ attack graph shown in a). Note that the dashed and solid lines representing application of defeasible and strict inference rules have been removed. The defeat graph (see Example~\ref{ExampleIllustration3}) is shown in b). \label{FigASPIC+Illustration}}
%The attack graph with the additional arguments $A^+_1$ $A^+_2$, respectively obtained by transposed rules $t, p \rightarrow \neg q$ and $p, q \rightarrow \neg t$, are shown in c).}
\end{figure}

\begin{example}\label{ExampleIllustration3}[\textbf{Example \ref{ExampleIllustration2} continued}] \sm{We assume that the argument ordering $\prec$} is defined in terms of preorderings $\leq$ on defeasible rules and $\leq'$ on ordinary premises (in ways fully specified in Section~\ref{SubSection-ASPIC+WeakestLast} below). Assume that $r \Rightarrow q < a \Rightarrow p$ and $\neg r <' r; \neg a \approx' r$;~$\sim s <' \neg r$. (As usual, $l \approx l'$ iff $l \leq l'$ and $l' \leq l$, while $l < l'$ iff $l \leq l'$ and $l' \nleq l$; likewise for $\approx'$ and $<'$.)

Now let $B_2' \prec A$, $B \prec A$  (because of $r \Rightarrow q < a \Rightarrow p$), $C \prec B_2, C \prec B_2', C \prec B$ (because of $\neg r <' r$). Then $B$ does not defeat $A$ ($B \not\hookrightarrow A$), $C \not\hookrightarrow B$, $C \not\hookrightarrow B_2'$ and $B_2 \hookrightarrow C$ (the arguments and defeats are depicted in Figure \ref{FigASPIC+Illustration}b)).\\[2pt] Note that if one had the additional arguments $D$ and/or $E$ described in Example \ref{ExampleIllustration2}, then $D$ would defeat $B_1$ and so $B_1'$ and $B$, while $E$ would defeat $A$. Note that $D \rightharpoonup B_1$ is preference independent since $s$ is a contrary of $\sim s$; %This effectively implements the negation as failure semantics of $\sim$, since
the validity of $B_1, B_1'$ and $B$ is contingent on $s$ not being provable (i.e., there being no acceptable argument for $s$).
%Of course, in general, attacks by contraries can also be on the consequents of defeasible rules.
\end{example}

\subsection{Structuring Argumentation Frameworks}\label{Section-ASPIC+-c-SAF}

We now define two notions of a structured argumentation framework instantiated by an argumentation theory. The first is defined as in \cite{hp10aspicJAC}. The second accounts for this paper's definition of c-consistent arguments.

\begin{definition}\label{DefinitionArgumentationTheory}[\textbf{Argumentation theory}]
An \emph{argumentation theory} is a tuple $AT$ =
$(AS,\K)$ where $AS$ is an argumentation system and $\K$ is a knowledge base in $AS$.
\end{definition}

\begin{definition}\label{DefinitionStructuredAF}[\textbf{(c-)Structured Argumentation Frameworks}]
Let $AT$ be an argumentation theory $(AS,\K)$.
\\[2pt]
-- A \emph{structured argumentation framework (\emph{SAF})} defined by $AT$, is a triple $\langle\A$, $\C$, $\preceq$ $\rangle$ where $\A$ is the set of all finite arguments constructed from $\K$ in
$AS$ (henceforth called the set of arguments on the basis of $AT$), $\preceq$ is an ordering on $\A$, and $(X,Y) \in \C$ iff $X$ attacks $Y$.
\\[2pt]
-- A \emph{c-structured argumentation framework (\emph{c-SAF})} defined by $AT$, is a triple $\langle\A$, $\C$, $\preceq$ $\rangle$ where $\A$ is the set of all c-consistent finite arguments constructed from $\K$ in
$AS$, $\preceq$ is an ordering on $\A$, and $(X,Y) \in \C$ iff $X$ attacks $Y$.\\[4pt]
\noindent Henceforth, we may write `\emph{(c-)SAF}' instead of writing `\emph{SAF} or \emph{c-SAF}'. Note that a c-SAF is a SAF in which all arguments are required to have a c-consistent set of premises.
\end{definition}

In \cite{hp10aspicJAC}, it is assumed that any argumentation theory satisfies a number of properties. We repeat these here, and add an additional `c-classicality' property for c-\emph{SAF}s, in which we refer to the notion of `closure under strict rules' and the notation `$S \vdash \phi$' given in Definition \ref{consistency} and Notation \ref{NotationArguments} respectively.

%
% that for any $P \subseteq \LA$, the closure of $P$ under strict rules, denoted $Cl_{\R_s}(P)$, is the smallest set containing $P$ and the consequent of any strict rule in $\R_s$ whose antecedents are in $Cl_{\R_s}(P)$.

\begin{definition}\label{DefStrictDefAssumptions}[\textbf{Well defined \emph{(c-)SAF}s}] Let $AT$ = $(AS,\K)$ be an argumentation theory, where  $AS$ = $(\LA,\mbox{}^{-},\R,n)$. We say that $AT$ is:
\\[3pt]
\noindent $\bullet$ \emph{closed under contraposition}  iff for all $S \subseteq \LA$, $s \in S$ and $\phi$, if $S \vdash \phi$, then $S \backslash \{s\} \cup \{-\phi\} \vdash -s$. \\[2pt]
\noindent $\bullet$ \emph{closed under transposition} \footnote{The notion of closure under transposition is taken from \cite{c+a07}.} iff if $\phi_1,\ldots,\phi_n \rightarrow \psi$ $\in \R_s$, then for $i = 1 \ldots n$, $\phi_1, \phi_{i-1}, - \psi, \phi_{i+1},$ $\ldots,\phi_n \rightarrow - \phi_i$ $\in \R_s$;
\\[2pt]
\noindent $\bullet$ \emph{axiom consistent} iff $Cl_{\R_s}(\K_n)$ is consistent.\\[2pt]
\noindent $\bullet$ \emph{c-classical} iff for any minimal c-inconsistent $S \subseteq \LA$ and for any $\varphi \in S$, it holds that  $S \setminus \{\varphi\} \vdash -\varphi$ (i.e., amongst all arguments defined there exists a strict argument with conclusion $-\varphi$ with all premises taken from $S \setminus \{\varphi\}$).
\\[2pt]
\noindent $\bullet$ \emph{well formed} if whenever $\varphi$ is a contrary of $\psi$ then $\psi \notin \K_{n}$ and $\psi$ is not the consequent of a strict rule.\footnote{This formulation repairs an error in the one of \cite{hp10aspicJAC}, which allowed for counterexamples to some results.}
\\[2pt]
If a \emph{c-SAF} is defined by an $AT$ that is c-classical, axiom consistent, well formed and closed under contraposition or closed under transposition, then the \emph{c-SAF} is said to be \emph{well defined}.
\\[2pt]
\noindent
If a \emph{SAF} is defined by an $AT$ that is axiom consistent, well formed and closed under contraposition or closed under transposition, then the \emph{SAF} is said to be \emph{well defined}.

\end{definition}

Henceforth, we will assume that any \emph{(c-)SAF} is well defined. The intuitions underlying the first four properties are self-evident. The rationale for the well-formed assumption is discussed in Section \ref{Section-ASPIC+-AttacksAndDefeats}.
%Finally, it should be obvious to see that if the strict rules conform to a Tarskian consequence operator (cf. \cite{a+b09}), for example, if they consist of all valid propositional or first-order inferences over $\LA$, then the properties of \emph{closure under contraposition} and \emph{transposition} are always satisfied.

\emph{(c-)SAF}s can now be linked to Dung frameworks. Firstly, note that as with existing approaches \cite{AmCay, BC03, ModgilAIJ}, \cite{hp10aspicJAC}'s notion of a conflict free set of arguments is defined with respect to the derived defeat relation.

\begin{definition}\label{DefinitionDDConflictFree}[\textbf{Defeat conflict free for \emph{(c-)SAF}s}] Let $\Delta$ = $\langle \A$, $\C$, $\preceq$ $\rangle$ be a \emph{(c-)SAF}, and $\D \subseteq \A \times \A$, where $(X,Y) \in \D$ iff $X$ defeats $Y$ according to Definition~\ref{DefDefeats}. Then $S \subseteq \A$ is \emph{defeat conflict free} iff $\forall X,Y \in S$, $(X,Y) \notin \D$.
\end{definition}

However, we have in Section \ref{Section-ArgumentationAndLogic-AttacksAndDefeats} argued that conflict free sets should be defined with respect to the attack relation, and defeats reserved for the dialectical use of attacks:

\begin{definition}\label{DefinitionADConflictFree}[\textbf{Attack conflict free for \emph{(c-)SAF}s}] Let $\Delta$ = $\langle \A$, $\C$, $\preceq$ $\rangle$ be a \emph{(c-)SAF}. Then $S \subseteq \A$ is \emph{attack conflict free} iff $\forall X,Y \in S$, $(X,Y) \notin \C$.
\end{definition}

In either case, the justified arguments are then evaluated on the basis of the extensions of a Dung framework instantiated by the arguments and derived defeat relation:

%\begin{definition}\label{DefinitionASPIC+Acceptability}[\textbf{Acceptability for \emph{(c-)SAF}s}] Let $\Delta$ = $\langle \A$, $\C$, $\preceq$ $\rangle$ be a \emph{SAF} or \emph{c-SAF}, and $\D \subseteq \A \times \A$, where $(X,Y) \in \D$ iff $X$ \emph{defeats} $Y$. Then for any argument $X
%\in \A$, $X$ is acceptable with respect to some $S \subseteq \A$ iff $\forall Y$ s.t. ($Y,X$) $\in \D$ implies $\exists Z \in S$ s.t. ($Z,Y$) $\in \D$.
%\end{definition}

\begin{definition}\label{DefinitionSAFExtensions}[\textbf{Extensions and justified arguments/conclusions of \emph{(c-)SAF}s}] Let $\Delta$ = $\langle \A$, $\C$, $\preceq$ $\rangle$ be a \emph{(c-)SAF}, and $\D \subseteq \A \times \A$, where $(X,Y) \in \D$ iff $X$ defeats $Y$. Let $S \subseteq \A$ be defeat or attack conflict free. The \emph{extensions and justified arguments of} $\Delta$ are the extensions of the Dung framework $(\A,\D)$, as defined in Definition \ref{Dung semantics}.\\[2pt]
For $T \in \{$admissible, complete, preferred, grounded, stable$\}$, we say that:\\[-17pt]
 \begin{itemize}
   \item $\varphi$ is a $T$ credulously justified conclusion of $\Delta$ iff there exists an argument $A$.
such that $\mathtt{Conc}(A)$ = $\varphi$, and $A$ is credulously justified under the $T$ semantics.\\[-17pt]
   \item $\varphi$ is a $T$ sceptically justified conclusion of $\Delta$ iff for every $T$ extension $E$, there exists an argument $A\in E$
such that $\mathtt{Conc}(A)$ = $\varphi$.\\[-17pt]
 \end{itemize}
$S$ is a \emph{def}-$T$ extension if
$S$ is defined as defeat conflict free, and an \emph{att}-$T$ extension if $S$ is defined as attack conflict free.

\end{definition}

We now recall a definition from \cite{hp10aspicJAC}, and then in this paper we define the notion of an argument $A$ being a strict continuation of a set of arguments $\{A_1,\dots,A_n\}$.

\begin{definition} The set $M(B)$ of the \emph{maximal fallible sub-arguments} of $B$ is defined such that for any $B' \in \mathtt{Sub}(B)$, $B' \in M(B)$ iff:\\[-17pt]

\begin{enumerate}
  \item $B'$'s top rule is defeasible or $B'$ is an ordinary premise, and;\\[-13pt]
  \item there is no $B'' \in \mathtt{Sub}(B)$ s.t. $B'' \neq B$ and $B' \in \mathtt{Sub}(B'')$, and $B''$ satisfies 1).

\end{enumerate}

\end{definition}
The maximal fallible subarguments of an argument $B$ are those with the `last' defeasible inferences in $B$ or else (if $B$ is strict) they are $B$'s ordinary premises. That is, they are the maximal subarguments of $B$ on which $B$ can be attacked. In Example~\ref{ExampleArgs} we have that $M(A) = \{A\}$, $M(B) =  \{B_1',B_2'\}$ $M(C) = \{C\}$.

\begin{definition}\label{Def-StrictExtension}[\textbf{Strict Continuations of Arguments}] For any set of arguments $\{A_1,\dots,A_n\}$, the argument $A$ is a \emph{strict continuation} of $\{A_1,\dots,A_n\}$ iff:\\[3pt]- $\mathtt{Prem}_{p}(A)$ = $\bigcup_{i=1}^n \mathtt{Prem}_{p}(A_i)$ \\(i.e., the ordinary premises in $A$ are exactly those in $\{A_1,\dots,A_n\}$);
 \\[3pt]-
$\mathtt{DefRules}(A)$ = $\bigcup_{i=1}^n \mathtt{DefRules}(A_i)$ \\(i.e., the defeasible rules in $A$ are exactly those in $\{A_1,\dots,A_n\}$);
 \\[3pt]- $\mathtt{StRules}(A) \supseteq \bigcup_{i=1}^n\mathtt{StRules}(A_i)$ and $\mathtt{Prem_n}(A) \supseteq \bigcup_{i=1}^n\mathtt{Prem_n}(A_i)$\\
 (i.e., the strict rules and axiom premises of $A$ are a superset of the strict rules and axiom premises in $\{A_1,\dots,A_n\}$).

\end{definition}

\begin{example}\label{ExampleIllustration4}[\textbf{Example \ref{ExampleArgs} continued}]
In Example \ref{ExampleArgs}, we have that $B$ is a strict continuation of $B_1'$ and $B_2'$. Now notice that the argumentation theory in Example \ref{ExampleIllustration1} is not well defined, since it is neither closed under contraposition or transposition. Closure under transposition augments $\R_s$ with rules $t, p \rightarrow \neg q$ and $p, q \rightarrow \neg t$, so obtaining the additional arguments $A^+_1$ = $[B_1', A \Rightarrow \neg q]$ that rebut-attacks $B_2'$ (and so $B$), and $A^+_2$ = $[A, B_2' \Rightarrow \neg t]$ that rebut-attacks $B_1'$ (and so $B$). Figure \ref{FigASPIC+Illustration2}a) shows these additional arguments and attacks.

\begin{figure}[h]
\centering
\includegraphics[width=5.2in]{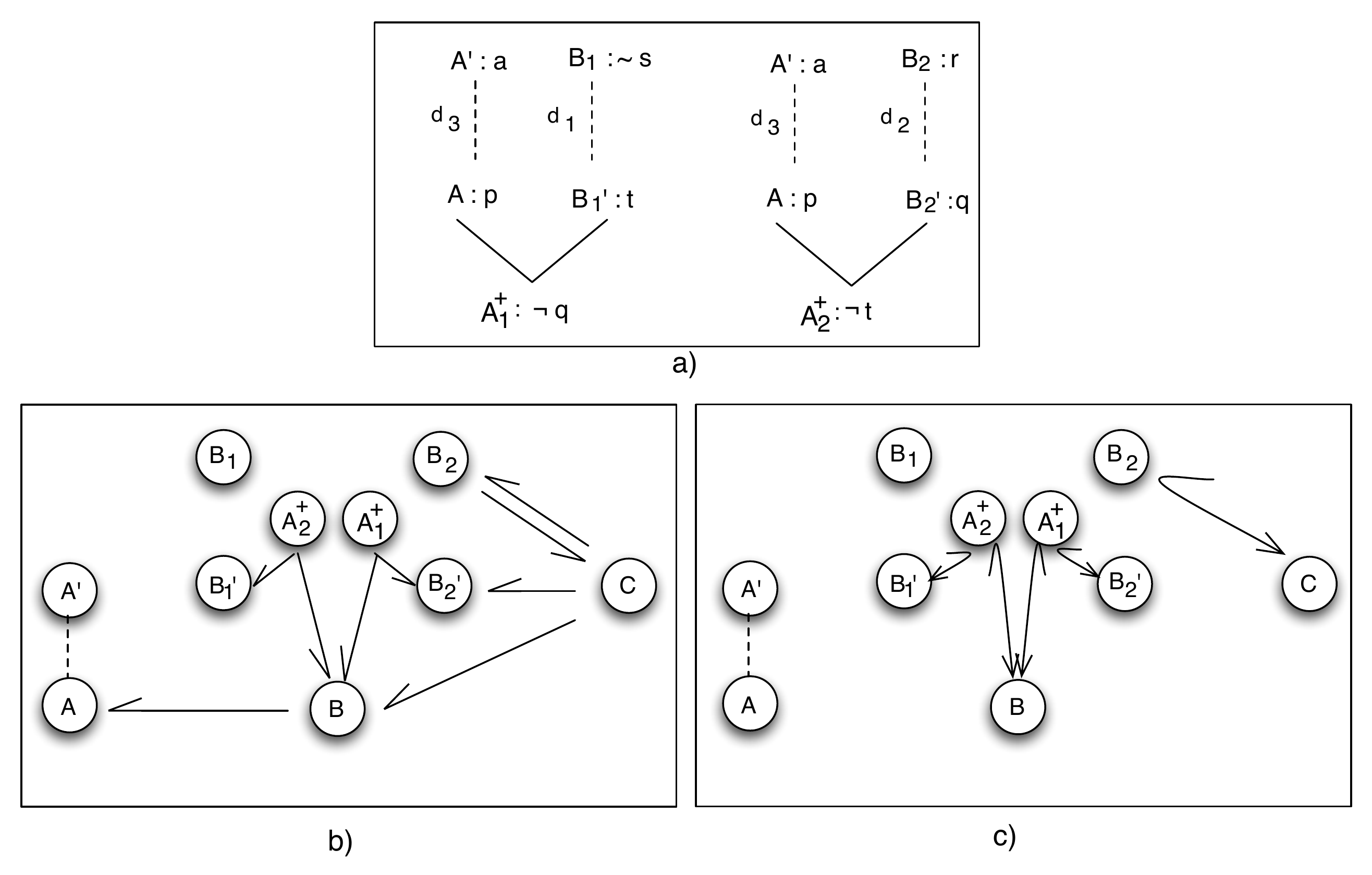}
%\caption{Example \ref{ExampleIllustration2}'s \emph{ASPIC}$^+$ arguments and attacks are shown on the left ($D$ and $E$ are extra arguments not defined by the knowledge base and argumentation system in Example \ref{ExampleIllustration1}. The defeat graph is shown on the right.}\label{FigASPIC+Illustration}
\caption{Example \ref{ExampleIllustration4}'s arguments $A^+_1$ and $A^+_2$ built from transpositions of the strict rule $t,q \imp \neg p$ are shown in a).  \emph{ASPIC}$^+$ attack graph shown in b). A possible defeat graph is shown in c). \label{FigASPIC+Illustration2}}
%The attack graph with the additional arguments $A^+_1$ $A^+_2$, respectively obtained by transposed rules $t, p \rightarrow \neg q$ and $p, q \rightarrow \neg t$, are shown in c).}
\end{figure}
\end{example}

\section{Properties and Postulates}\label{SectionPropertiesPostulates}

In this section we examine the implications of the attack definition of conflict free sets. We  show that under some assumptions on the preference ordering over arguments, both \emph{SAF}s and \emph{c-SAF}s satisfy the key properties of Dung frameworks. We also show that \cite{c+a07}'s rationality postulates straightforwardly hold. On the other hand, we will show that under \cite{hp10aspicJAC}'s `defeat definition' of conflict free, key properties of Dung frameworks straightforwardly hold, whereas satisfaction of \cite{c+a07}'s rationality postulates requires assumptions on preference orderings. Note that for the defeat definition, \cite{hp10aspicJAC} has already shown satisfaction of the rationality postulates for \emph{SAF}s. This paper extends \cite{hp10aspicJAC}'s results to \emph{c-SAF}s. Finally, we will show equivalence of admissible and complete extensions under the attack and defeat definitions of conflict free.

\subsection{Properties of \emph{SAF}s and \emph{c-SAF}s under the Attack Definition of Conflict Free}\label{Section-ASPIC+-Propoerties-Attack}

Defining conflict free sets in terms of the attack relation, while using the defeat relation for determining the acceptability of arguments, potentially undermines some key results shown for Dung frameworks. To illustrate, consider Example \ref{ExampleIllustration3}'s \emph{SAF}, with the arguments and attacks shown in Figure \ref{FigASPIC+Illustration}a). As shown in  Figure \ref{FigASPIC+Illustration}b), no argument defeats $B$, so $\{B\}$ is \emph{att}-admissible (as defined in Definition \ref{DefinitionSAFExtensions}). Since $B \prec A$, then $B \not\hookrightarrow A$, and so $A$ is acceptable w.r.t. $\{B\}$. But $\{A,B\}$ is not attack conflict free and so not \emph{att}-admissible. This violates Dung's \emph{fundamental lemma} \cite{dung95}, which states that if $S$ is admissible and $A$ is acceptable w.r.t. $S$ then $S \cup \{A\}$ is admissible. However, if the \emph{SAF} is well defined (Definition \ref{DefStrictDefAssumptions}), then under the assumption that an argument ordering is \emph{reasonable}, we can show that the fundamental lemma holds.

An argument ordering $\preceq$ is reasonable if it satisfies properties that one might expect to hold of orderings over arguments composed from fallible and infallible elements. Firstly, whenever an argument $A$ is not fallible (i.e., strict and firm), then it is strictly preferred over all arguments with fallible elements, and not less preferred than any other argument. Also, continuing an argument with only axiom premises and strict inferences does not change its relative preference.
%The intuitions here are self-evident.
The second property is essentially a strengthening of the requirement that the strict counter-part $\prec$ of $\preceq$ is asymmetric, by stating that for any set $\A'$ = $\{C_1,\ldots,C_n\}$ of arguments, it cannot be that for all $i$, $C' \prec C_i$ where $C'$ is a strict continuation of $\A'\backslash C_i$.

\begin{definition}\label{Def-Reasonable}[\textbf{Reasonable Argument Orderings}]
An argument ordering $\preceq$ is \emph{reasonable} iff:
\begin{enumerate}
\item i) $\forall A,B$, if $A$ is strict and firm and $B$ is plausible or defeasible, then $B \prec A$;\\
      ii) $\forall A,B$, if $B$ is strict and firm then $B \nprec A$;\\
      iii) $\forall A,A',B$ such that $A'$ is a strict continuation of $\{A\}$, if $A \nprec B$ then $A' \nprec B$, and if $B \nprec A$ then $B \nprec A'$ (i.e.,  applying strict rules to a single argument's conclusion and possibly adding new axiom premises does not weaken, respectively strengthen, arguments)%\footnote{This is a variant of the notion of an admissible ordering described in \cite{hp10aspicJAC}}.

\item Let $\{C_1,\ldots,C_n\}$ be a finite subset of $\A$, and for $i$ = $1 \ldots n$, let $C^{+\backslash i}$ be some strict continuation of $\{C_1,\ldots,C_{i-1},C_{i+1},\ldots,C_n\}$. Then it is not the case that: $\forall i$, $C^{+\backslash i} \prec C_i$.
\end{enumerate}
\end{definition}

In Section \ref{SubSection-ASPIC+WeakestLast} we give example definitions of argument orderings in terms of preorderings on the ordinary premises and defeasible rules, corresponding to the commonly used weakest and last link principles. We will then show that these argument orderings are reasonable. Henceforth, we will assume that the ordering $\preceq$ of any \emph{(c-)SAF} is reasonable.
%Indeed, we will show that for these argument orderings, Definition \ref{Def-Reasonable}-2 is equivalent to requiring that the strict counterpart $\prec$ of $\preceq$ is a strict partial ordering (transitive, irreflexive, and so asymmetric). Henceforth, we will assume that the ordering $\preceq$ of any \emph{(c-)SAF} is reasonable.
%%%

We now examine the implications of an argument ordering being reasonable. Under the assumption that Example \ref{ExampleIllustration3}'s \emph{SAF} is well-defined, we additionally have the arguments and attacks described in Example \ref{ExampleIllustration4}, and shown in Figure \ref{FigASPIC+Illustration}c).  Recall that we have the maximal fallible sub-arguments $\{A, B_1', B_2'\}$ of $A$ and $B$, where:\\[-17pt]

\begin{itemize}
\item $B$ is a strict continuation of $\{B_1',B_2'\}$;\\[-17pt]
\item $A^+_1$ a strict continuation of $\{B_1', A\}$;\\[-17pt]
\item $A^+_2$ a strict continuation of $\{B_2', A\}$.\\[-17pt]
\end{itemize}

\noindent Assuming $\preceq$ is reasonable, then by Definition~\ref{Def-Reasonable}-2:\\[-12pt]

\hspace{15mm}\emph{it cannot be that} $B \prec A$, $A^+_1 \prec B_2'$ \emph{and} $A^+_2 \prec B_1'$.\\[2pt] Since by assumption $B \prec A$, then it must be that either $A^+_1 \nprec B_2'$ or $A^+_2 \nprec B_1'$, and so $A^+_1$ defeats $B_2'$ or $A^+_2$ defeats $B_1'$.
Indeed, if we refer to the preordering over defeasible rules given in Example \ref{ExampleIllustration3}, then since no rule in $B_2'$ is strictly stronger than a rule in $A^+_1$, and no rule in $B_1'$ is strictly stronger than a rule in $A^+_2$, then $A^+_1 \nprec B_2'$, $A^+_2 \nprec B_1'$ \footnote{This is verified by the argument orderings defined on the basis of the last link principle in Section \ref{SubSection-ASPIC+WeakestLast}}, and we obtain the defeat graph shown in Figure \ref{FigASPIC+Illustration2}.

%
%  Intuitively, since $B \prec A$ is based on $A$ being strictly preferred to some maximal fallible sub-argument in $B$ (i.e., $B_1'$ or $B_2'$), then if both $A^+_1 \prec B_2'$ ($B_2'$ strictly preferred to $B_1'$ or $A$) and $A^+_2 \prec B_1'$ ($B_1'$ strictly preferred to $B_2'$ or $A$), then this would imply a cycle in the preference ordering over $\{A, B_1', B_2'\}$.
%

In fact, the following general result can be shown:

\begin{proposition}\label{Prop$B'$-extension} Let $A$ and $B$ be arguments where $B$ is plausible or defeasible and $A$ and $B$ have contradictory conclusions, and assume $\mathtt{Prem}(A)\cup \mathtt{Prem}(B)$ is c-consistent if
$A$ and $B$ are defined as in Definition \ref{DefinitionConsistentArguments}, that is, if they are assumed to have c-consistent premises.
 Then:
\begin{enumerate}
\item For all $B' \in M(B)$, there exists a strict continuation $A^+_{B'}$ of $(M(B)\backslash \{B'\}) \cup M(A)$ such that
%$\mathtt{Prem}(A^+) \subseteq \mathtt{Prem}(A) \cup \mathtt{Prem}(B)$, and
$A^+_{B'}$ rebuts or undermines $B$ on $B'$.

\item If $B \prec A$, and $\preceq$ is reasonable, then for some $B' \in M(B)$, $A^+_{B'}$ defeats $B$.
\end{enumerate}

\end{proposition}
This says that if the argument ordering is reasonable, then whenever an argument $B$ with a strict top rule rebuts (but not contrary rebuts) an argument $A$ with a defeasible top rule but is inferior to $A$, we can strictly continue $A$ into an argument defeating $B$.

Let us now generalise the earlier suggested counter-example to the fundamental lemma. Assume $B \in S$, $S$ is admissible, and either: 1) $B$ attacks $A$ on $A'$, $B \prec A'$, and so $B$ does not defeat $A$ (i.e., the example described at the beginning of this section), or; 2) $A$ attacks $B$ on $B'$, $A \prec B'$, and so $A$ does not defeat $B'$. The proof of Proposition \ref{PropositionConflictFree} below then makes use of Proposition \ref{Prop$B'$-extension} to show that in neither case can $A$ be acceptable w.r.t. $S$. This means that the result that \emph{if $A$ is acceptable w.r.t. an admissible $S$ then $S \cup \{A\}$ is conflict free}, and hence Dung's fundamental lemma, is not under threat. Prior to Proposition \ref{PropositionConflictFree}, we state a key result for \emph{c-SAF}s in order to show that Dung's fundamental lemma and the rationality postulates can be shown when arguments are restricted to those with consistent premises:

\begin{proposition}\label{Lemma4} Let ($\A$,
$\C$, $\preceq$) be a c-\emph{SAF}. If $A_1,\ldots,A_n$ are acceptable w.r.t. some conflict-free $E \subseteq \A$, then $\bigcup_{i=1}^n\mathtt{Prem}(A_i)$ is c-consistent.
\end{proposition}

\begin{proposition}\label{PropositionConflictFree} Let $A$ be acceptable w.r.t an admissible
extension $S$ of a \emph{(c-)SAF} ($\A$,
$\C$, $\preceq$). Then
$S'$ = $S$ $\cup$ $\{A\}$  is conflict free.

\end{proposition}

Proposition \ref{PropositionConflictFree} implies that Dung's fundamental lemma holds:

\begin{proposition}\label{PropositionFundamentalLemma} Let $A, A'$ be acceptable w.r.t an admissible
extension $S$ of a \emph{(c-)SAF} ($\A$,
$\C$, $\preceq$). Then:
\begin{enumerate}
\item $S'$ = $S$ $\cup$ $\{A\}$ is admissible \\[-17pt]\item $A'$ is
acceptable w.r.t. $S'$.
\end{enumerate}

\end{proposition}

We have shown that given reasonable argument orderings, a well defined \emph{(c-)SAF} satisfies Dung's fundamental lemma. This implies that the admissible extensions of a \emph{(c-)SAF} form a complete partial order w.r.t. set inclusion, and that every admissible extension is contained in a preferred extension. Also, given the definitions of defeat and acceptability, it is easy to see (in the same way as for Dung frameworks) that if $A$ is acceptable w.r.t. $S$, then $A$ is acceptable w.r.t. any superset of $S$ (this result is stated as Lemma \ref{Basic properties}-\ref{Basic1} in the Appendix). Thus, a \emph{(c-)SAF}'s characteristic function is monotonic, implying that the grounded extension can be identified by the function's least fixed point. It is also easy to see that every stable extension is a preferred extension.

\subsection{Rationality Postulates for \emph{SAF}s and \emph{c-SAF}s under the Attack Definition of Conflict Free}\label{Section-ASPIC+R-RationalityPostulates}

As discussed in Section \ref{SectionIntroduction}, the intermediate level of abstraction (between concrete instantiating logics and Dung's fully abstract theory) adopted by \emph{ASPIC} \cite{c+a07} and \emph{ASPIC}$^+$ \cite{hp10aspicJAC} frameworks, allows for the formulation and evaluation of postulates \cite{c+a07} whose satisfaction ensure that any concrete instantiations of the frameworks fulfil some rational criteria. We now show that under the attack definition of conflict free, well-defined \emph{SAF}s and c-\emph{SAF}s satisfy \cite{c+a07}'s rationality postulates for the complete (and so by implication the grounded, preferred and stable) semantics defined in Definition \ref{Dung semantics}.

Theorem \ref{TheoremSub-argumentClosure} below states that for any argument $A$ in a complete extension $E$, all sub-arguments of $A$ are in $E$. Theorem \ref{TheoremStrictRulesClosure} then states that the conclusions of arguments in a complete extension are closed under strict inference (recall that the closure $Cl_{R_s}(S)$ of $S$ under strict rules is defined in Definition~\ref{consistency}).

\begin{theorem}\label{TheoremSub-argumentClosure}[\textbf{Sub-argument Closure}] Let $\Delta$ = $(\A,\C,\preceq)$ be a \emph{(c-)SAF} and $E$ an \emph{att}-complete extension of $\Delta$. Then for all $A \in E$: if $A' \in \mathtt{Sub}(A)$ then $A' \in E$.
\end{theorem}

\begin{theorem}\label{TheoremStrictRulesClosure}[\textbf{Closure under Strict Rules}] Let $\Delta$ = $(\A,\C,\preceq)$ be a \emph{(c-)SAF} and $E$ an \emph{att}-complete extension of $\Delta$. Then $\{ \mathtt{Conc}(A) | A \in E\}$ = $Cl_{R_s}(\{ \mathtt{Conc}(A) | A \in E\})$.
\end{theorem}

Theorem \ref{TheoremDirectConsistency} below, states that the conclusions of arguments in an admissible extension (and so by implication complete extension) are mutually consistent. Theorem \ref{TheoremIndirectConsistency} then states the mutual consistency of the strict closure of conclusions of arguments in a complete extension.

\begin{theorem}\label{TheoremDirectConsistency}[\textbf{Direct Consistency}] Let $\Delta$ = $(\A,\C,\preceq)$ be a \emph{(c-)SAF} and $E$ an \emph{att}-admissible extension of $\Delta$. Then $\{ \mathtt{Conc}(A) | A \in E\}$ is consistent.
\end{theorem}

\begin{theorem}\label{TheoremIndirectConsistency}[\textbf{Indirect Consistency}] Let $\Delta$ = $(\A,\C,\preceq)$ be a \emph{(c-)SAF} and $E$ an \emph{att}-complete extension of $\Delta$. Then $Cl_{\R_s}(\{\mathtt{Conc}(A) | A \in E\})$ is consistent.
\end{theorem}

Note that the task of showing that \cite{c+a07}'s consistency postulates are satisfied is simplified by the fact that a conflict free set excludes attacking arguments.
% The proof that `direct consistency' holds essentially amounts to isolating the case where two arguments $A, B \in E$ have contrary conclusions ($\mathtt{Conc}(A) \in \con{\mathtt{Conc}(B)}$), but neither attacks the other. But then one can show that these arguments are of the type described in Proposition \ref{Prop$B'$-extension}, and so by the proposition's first result, there is a strict continuation $A^+$ of fallible sub-arguments in $A$ and $B$ such that $A^+$ attacks $B$. Since $A$ and $B$ are by assumption acceptable w.r.t. $E$, then $A^+$ must be acceptable w.r.t. $E$. But then this contradicts Proposition \ref{PropositionConflictFree} which states that $E \cup \{A^+\}$ must be conflict free.

\subsection{Comparison of Attack and Defeat Definition of Conflict Free}\label{Section-Attack-Defeat-Comparison}

Under the defeat definition of conflict free, the properties discussed in Section \ref{Section-ASPIC+-Propoerties-Attack} are of course shown to hold for \emph{(c-)SAF}s in the same way as for Dung frameworks. We now state the equivalence of extensions of \emph{(c-)SAF}s under the attack and defeat definitions of conflict free.

\begin{proposition}\label{Prop-Equivalence-a-d-admissible} Let $\Delta$ be a \emph{(c-)SAF}. For $T \in$ $\{$ \emph{admissible}, \emph{complete}, \emph{grounded}, \emph{preferred}, \emph{stable}$\}$, $E$ is an \emph{att}-$T$ extension of $\Delta$ iff $E$ is a \emph{def}-$T$ extension of $\Delta$.\\[-15pt]

\end{proposition}

Given the previous section's results, Proposition \ref{Prop-Equivalence-a-d-admissible} implies that \cite{c+a07}'s rationality postulates not only hold for \emph{SAF}s under the defeat definition (as already shown in \cite{hp10aspicJAC}), but also for \emph{c-SAF}s under the defeat definition.

\begin{corollary}\label{EquivalenceCorollary} Let $\Delta$ be a \emph{(c-)SAF}. Then Theorems \ref{TheoremSub-argumentClosure}-\ref{TheoremIndirectConsistency}
hold for the \emph{def}-admissible and \emph{def}-complete extensions of $\Delta$.
\end{corollary}

Notice that directly proving satisfaction of the consistency postulates for the defeat definition of conflict free is more involved. One must rely on Proposition \ref{Prop$B'$-extension} to show that an admissible extension contains arguments that do not defeat each other. The trade off is that with the attack definition, proof of the fundamental lemma is more involved since one needs to first show that any argument acceptable w.r.t. an admissible extension is conflict free when included in that extension. It is the proof of \emph{this} result that crucially depends on Proposition \ref{Prop$B'$-extension}. Notice that in both cases, one needs to consider the internal structure of arguments and assume a reasonable preference ordering.

Proposition \ref{Prop-Equivalence-a-d-admissible}'s equivalence begs the question as to why one should advocate the attack definition of conflict free. Firstly, a result that shows that the two different notions of conflict-freeness are (under certain assumptions) equivalent in the extensions they produce is theoretically valuable in itself. Apart from this, we have argued in Section \ref{Section-ArgumentationAndLogic} that the attack definition is conceptually more well justified. In Example \ref{ExampleIllustration3}, neither $B$ or $A$ defeat each other, and neither $B$ or $C$ defeat each other. Under the defeat definition, $\{B, A\}$ and $\{B, C\}$ are `conflict free'. But in what meaningful sense can these sets be said to be conflict free, when they contain elements that are mutually inconsistent? Consider then \cite{AmgoudRepairPAFs}'s example that purports to illustrate violation of the consistency postulates by approaches augmenting Dung frameworks with preferences. An expert argues ($A$) that a given violin is a Stradivarius and therefore expensive. A three-year old child's argument $B$ then states that it is not a Stradivarius. According to \cite{AmgoudRepairPAFs}, $B$ attacks $A$ but $A$ does not attack $B$, and $A$ is preferred over $B$ since the expert is more reliable than the child. \cite{AmgoudRepairPAFs} observe that the unique \emph{PAF}-extension $\{A,B\}$ violates the consistency postulate. In Section \ref{Section Comparison with other works on Preference-based Argumentation} we demonstrate that the problem does not arise if all arguments that can be constructed are taken into account (the expert can use a sub-argument $A'$ of $A$ that defeats $B$ so that $\{A,B\}$ is not admissible), illustrating that the problem has more to do with imperfect reasoners. However, we also note that the attack definition of conflict free is more tolerant of imperfect reasoning. Without taking into account all constructible arguments, $\{A,B\}$ is of course \emph{not} conflict free and so not a \emph{PAF}-extension, and so consistency is not violated.

To conclude, given our advocacy of the attack definition of conflict free, we henceforth assume any extension of a \emph{(c-)SAF} to be attack conflict free, and thus will henceforth
(for a given semantics $T$) refer to a `$T$ extension' rather than an `\emph{att}-$T$ extension'. However, for the results shown in the next section, we will indicate, when appropriate, that the results also hold under the defeat definition of conflict free.

\section{Instantiating Structured Argumentation Frameworks}\label{SectionInstantiations}

We have modified \emph{ASPIC}$^{+}$ in two ways: we have additionally defined \emph{c-SAF}s whose arguments must be built on mutually consistent premises, and motivated an alternative attack definition of conflict free sets. We have shown that
properties and postulates hold for well defined \emph{SAF}s and \emph{c-SAF}s with argument preference orderings that are \emph{reasonable}. In this section we study various ways to instantiate the \emph{ASPIC}$^{+}$ framework. Section \ref{SubSection-ASPIC+WeakestLast} consider ways of `instantiating' preference orderings over arguments in terms of preorderings over defeasible rules and ordinary premises. We show that the defined argument orderings are reasonable. In Section \ref{sectionALstrict} we extend with preferences Amgoud \& Besnard's approach to structured argumentation \cite{a+b09,a+b10} based on Tarski's notion of an abstract logic. We then combine
%%%
%\marginpar{Sanjay, please check change}
%%%
the extended abstract logic approach with \emph{ASPIC}$^{+}$. In Section \ref{SectionClassicalLogicInstantiations} we define classical logic instantiations of \emph{c-SAF}s, and show an equivalence between one such instantiation and Brewka's Preferred Subtheories \cite{BrewkaPS}.

\subsection{Weakest and Last Link Preference Relations}\label{SubSection-ASPIC+WeakestLast}

%In this section we present some example definitions of the pre-ordering $\preceq$ on arguments.
In \cite{hp10aspicJAC}, a strict argument ordering $\prec$ is defined on the basis of two preorderings $\leq$ on $\R_d$ and
%%%\marginpar{Sanjay, please check changes}
%%%
 $\leq'$ on $\K_p$ under the well known \emph{weakest}-link \cite{BrewkaPS,CayRoy} and \emph{last}-link \cite{KT96,PS97} principles\footnote{Note that the cited papers make use of the principles without explicitly naming them as such.}. Intuitively,
$B \prec A$ is defined by separate set comparisons of the defeasible rules in $B$ and $A$, and the ordinary premises in $B$ and $A$. Then $B \prec A$ by the weakest link principle if:

\begin{enumerate}
  \item from amongst \emph{all the} defeasible rules in $B$ there exists a rule which is weaker than (strictly less than according to $\leq$) \emph{all the} defeasible rules in $A$, \emph{and}
  \item from amongst all the ordinary premises in $B$ there is an ordinary premise which is weaker (strictly less than according to $\leq'$) all the ordinary premises in $A$
\end{enumerate}

Then $B \prec A$ by the last link principle if the above set comparison (henceforth referred to as the $\mathtt{Elitist}$ comparison) on defeasible rules is now applied only to the last defeasible rules in $B$ and $A$ (recall the definition of $\mathtt{LastDefRules}$ in Definition \ref{arg}); i.e., `all the last' replaces `all the' in 1). If there are no defeasible rules in $B$ and $A$, then only the ordinary premises are compared, and so $B \prec A$ by the last link principle if 2) holds.

In this paper we provide an alternative interpretation of the weakest and last link principles based on an alternative set comparison (sometimes referred to as the $\mathtt{Democratic}$ comparison \cite{CayRoy}). We provide a general definition of a set comparison $\triangleleft_{\mathtt{s}}$ over sets of  defeasible rules or premises, and which is then parameterised according to the $\mathtt{Elitist}$ and $\mathtt{Democratic}$ comparisons (i.e., $\mathtt{s} = \mathtt{Eli}$ and $\mathtt{Dem}$ respectively). Note that the following definition references a preordering $\leq$ over defeasible rules or premises, where as usual:  $X < Y$ iff  $X \leq Y$ and $Y \nleq X$;  $X \approx Y$ iff  $X \leq Y$ and $Y \leq X$.

\begin{definition}\label{Def-hd12}\textbf{[Orderings $\triangleleft_{\mathtt{s}}$]} Let $\Gamma$ and $\Gamma'$ be finite sets\footnote{Notice that it suffices to restrict $\lhd$ to finite sets since \emph{ASPIC}$^+$ arguments are defined as finite (in Definition \ref{arg}) and so their ordinary premises/defeasible rules must be finite.}. Then $\triangleleft_{\mathtt{s}}$ is defined as follows:

\begin{enumerate}
  \item If $\Gamma = \emptyset$ then $\Gamma \ntriangleleft_{\mathtt{s}} \Gamma'$ ;
  \item If $\Gamma' = \emptyset$ and $\Gamma \not= \emptyset$ then $\Gamma \triangleleft_{\mathtt{s}} \Gamma'$ ;\\[3pt]
  else, assuming a preordering $\leq$ over the elements in $\Gamma \cup \Gamma'$:\\[-13pt]
  \item\label{Eli}
    if $\mathtt{s} = \mathtt{Eli}$:\\[2pt]
  $\Gamma \triangleleft_{\mathtt{Eli}} \Gamma'$ if $\exists X \in \Gamma$ s.t. $\forall Y \in \Gamma'$, $X < Y$.\\[3pt]
   else:\\[-13pt]
  \item\label{Dem1} if $\mathtt{s} = \mathtt{Dem}$:\\[2pt]
  $\Gamma \triangleleft_{\mathtt{Dem}} \Gamma'$ if $\forall X \in \Gamma$, $\exists Y \in \Gamma'$, $X < Y$.

\end{enumerate}

Then $\Gamma \trianglelefteq_{\mathtt{s}} \Gamma'$ iff $\Gamma  \triangleleft_{\mathtt{s}} \Gamma'$ or
 $\Gamma' = \Gamma$.\footnote{\sm{Where `=' denotes identity.}}

\end{definition}

Conditions 1 and  2 of Definition \ref{Def-hd12} intuitively impose that for any sets of defeasible rules/ordinary premises $S$ and $S'$,  if $S$ is the empty set, it cannot be that $S \triangleleft_{\mathtt{s}} S'$, and if $S'$ is the empty set, it must be that $S \triangleleft_{\mathtt{s}} S'$ for any non-empty $S$.
%
% Hence the above definition explicitly imposes that any set comparison $\trianglelefteq_{\mathtt{s}}$ satisfies these properties  since one cannot assume them to be satisfied for every possible $\mathtt{s}$. For example, the $\mathtt{Democratic}$ comparison does not in general satisfy these properties\footnote{Since if $S$ = $\emptyset$, $S' \neq \emptyset$, then trivially: $\forall X \in S, \exists Y \in S' s.t. X \leq Y$ and it is not the case that $\forall Y \in S', \exists X \in S s.t. Y \leq X$}.
%

\begin{definition}\label{Def Last Link}[\textbf{Last-link principle}] Let $\mathtt{s} \in $ $\{\mathtt{Eli}, \mathtt{Dem}\}$. Then $B \prec A$ under the last-link principle iff\\[-12pt]
\begin{enumerate}
\item $\mathtt{LastDefRules}(B) \triangleleft_\mathtt{s} \mathtt{LastDefRules}(A)$; or\\[-17pt]
\item $\mathtt{LastDefRules}(B) = \emptyset$, $\mathtt{LastDefRules}(A) = \emptyset$, and $\mathtt{Prem_{p}}(B)$ $\triangleleft_\mathtt{s}$ $\mathtt{Prem_{p}}(A)$
\end{enumerate}
Then $B \preceq A$ iff $B \prec A$ or, if $\mathtt{LastDefRules}(A) \neq \emptyset$ then
$\mathtt{LastDefRules}(A) = \mathtt{LastDefRules}(B)$, else $\mathtt{Prem_{p}}(A) = \mathtt{Prem_{p}}(B)$.
\end{definition}

\begin{definition}\label{Def Weakest Link}[\textbf{Weakest-link principle}] Let $\mathtt{s} \in $ $\{\mathtt{Eli}, \mathtt{Dem}\}$. Then $B \prec A$ under the weakest-link principle iff:

\begin{enumerate}

    \item If both $B$ and $A$ are strict, then $\mathtt{Prem_{p}}(B) \triangleleft_\mathtt{s} \mathtt{Prem_{p}}(A)$, else;\\[-17pt]
    \item If both $B$ and $A$ are firm, then $\mathtt{DefRules}(B) \triangleleft_\mathtt{s} \mathtt{DefRules}(A)$, else;\\[-17pt]
    \item $\mathtt{Prem_{p}}(B) \triangleleft_\mathtt{s} \mathtt{Prem_{p}}(A)$ and $\mathtt{DefRules}(B) \triangleleft_\mathtt{s} \mathtt{DefRules}(A)$
\end{enumerate}
Then $B \preceq A$ iff $B \prec A$ or, $\mathtt{DefRules}(A) = \mathtt{DefRules}(B)$ and  $\mathtt{Prem_{p}}(A) = \mathtt{Prem_{p}}(B)$.
\end{definition}

%Notice that in this paper we correct an anomaly in \cite{hp10aspicJAC}'s definition of the weakest link principle, in which if both $B$ and $A$ are strict (contain no defeasible rules), then $B \prec A$ if there are ordinary premises in $B$ that are $\lhd$ the ordinary premises in $A$. However, if both $B$ and $A$ are firm (contain no ordinary premises), then it is not the case that $B \prec A$ if there are defeasible rules in $B$ that are $\lhd$ the defeasible rules in $A$. Thus there is an asymmetry in the way that premises and defeasible rules are compared. In Definition \ref{Def Weakest Link} above, the weakest link is defined so that the defeasible rules and ordinary premises are treated in the same way.

\sm{Notice that in this paper we correct an anomaly in \cite{hp10aspicJAC}'s definition of the weakest link principle, in which if both $B$ and $A$ are strict (contain no defeasible rules) but not firm (and so do contain ordinary premises), then $B \prec A$ if $B$'s set of ordinary premises is strictly inferior to $A$'s set of ordinary premises. However, if both $B$ and $A$ are firm (contain no ordinary premises) but not strict (and so do contain defeasible rules), then it is not the case that $B \prec A$ if $B$'s set of defeasible rules is strictly inferior to $A$'s set of defeasible rules. Thus there is an asymmetry in the way that premises and defeasible rules are compared. In Definition \ref{Def Weakest Link} above, the weakest link is defined so that the defeasible rules and ordinary premises are treated in the same way.}

%%%\marginpar{Example corrected; see also Ex 5}
\begin{example}\label{ExampleIllustration5}[\textbf{Example \ref{ExampleIllustration2} continued}] Given:
\begin{itemize}
  \item $r \Rightarrow q < a \Rightarrow p$;
  \item $\neg r <' r$ ; $\neg a \approx' r$ ; $\sim s <' \neg r$
\end{itemize}
on defeasible rules and ordinary premises in Example \ref{ExampleIllustration1}, and employing the abbreviations $\mathtt{DR}$ for $\mathtt{DefRules}$ and $\mathtt{LDR}$ for $\mathtt{LastDefRules}$, we have:
\\[2pt]
\noindent $\mathtt{DR}(A)$ = $\mathtt{LDR}(A)$ = $\{a \Rightarrow p\}$, $\mathtt{Prem_{p}}(A) = \{a\}$;\\[2pt]
$\mathtt{DR}(A')$ = $\mathtt{LDR}(A')$ = $\emptyset$, $\mathtt{Prem_{p}}(A') = \{a\}$;\\[2pt]
$\mathtt{DR}(B)$ = $\mathtt{LDR}(B)$ = $\{\sim s \Rightarrow t, r \Rightarrow q\}$, $\mathtt{Prem_{p}}(B) = \{\sim s, r\}$;\\[2pt]
$\mathtt{DR}(B_2)$ = $\mathtt{LDR}(B_2)$ = $\emptyset$, $\mathtt{Prem_{p}}(B_2) = \{r\}$;\\[2pt]
$\mathtt{DR}(C)$ = $\mathtt{LDR}(C)$ = $\emptyset$, $\mathtt{Prem_{p}}(C) = \{\neg r\}$.\\[3pt]
Then:\\[-17pt]

\begin{itemize}
  \item $\mathtt{LDR}(B) \triangleleft_{\mathtt{Eli}} \mathtt{LDR}(A)$ and so  $B \prec A$ under the last-link principle.\\ $\mathtt{DR}(B) \triangleleft_{\mathtt{Eli}} \mathtt{DR}(A)$, but $\mathtt{Prem_{p}}(B) \ntriangleleft_{\mathtt{Eli}} \mathtt{Prem_{p}}(A)$, so $B \nprec A$   under the weakest-link principle.
  \item $\mathtt{LDR}(B) \ntriangleleft_{\mathtt{Dem}} \mathtt{LDR}(A)$ and $\mathtt{Prem_{p}}(B) \ntriangleleft_{\mathtt{Dem}} \mathtt{Prem_{p}}(A)$, hence $B \nprec A$ under the last or weakest-link principle.
  \item $\mathtt{Prem_{p}}(C) \triangleleft_{\mathtt{Eli}} \mathtt{Prem_{p}}(B_2)$  and so $C \prec B_2$ under last or weakest-link principles.
  \item $\mathtt{Prem_{p}}(C) \triangleleft_{\mathtt{Dem}} \mathtt{Prem_{p}}(B_2)$ and so $C \prec B_2$ under last or weakest-link principle.
%      \item $\mathtt{Prem_{p}}(C) \trianglelefteq_{\mathtt{Dem}} \mathtt{Prem_{p}}(A')$, $\mathtt{Prem_{p}}(A') \trianglelefteq_{\mathtt{Dem}} \mathtt{Prem_{p}}(C)$, and so $C \preceq A'$, $A' \preceq C$, hence $C \approx A'$ under the last or weakest-link principle.
\end{itemize}

\end{example}

A natural question to ask is whether comparisons other than $\mathtt{Democratic}$ or $\mathtt{Elitist}$ can be employed when defining $\triangleleft_{\mathtt{s}}$ in Definition \ref{Def-hd12}.
% so allowing reference to a larger range of comparisons $\mathtt{s}$ in the definitions of the last and weakest link principles.
We identify a class of `reasonable inducing' comparisons:
%, by specifying properties that the defined set ordering $\trianglelefteq_{\mathtt{s}}$ satisfies.

\begin{definition}\label{Def-s-reasonable inducing}[\textbf{Inducing reasonable orderings}]  \sm{$\triangleleft_{\mathtt{s}}$} is said to \emph{reasonable inducing} if it is a strict partial ordering (irreflexive and transitive), and:
\begin{quote}
%   \item\label{lemma-irreflexive of lhd}  $\triangleleft_{\mathtt{s}}$ is irreflexive;
%  \item\label{lemma-transitivity of lhd} $\triangleleft_{\mathtt{s}}$ is transitive;
%  \item\label{Lem_lhd}
\hspace{-5mm}  for any $\mathtt{kr} \in \{\mathtt{LastDefRules},\mathtt{DefRules},\mathtt{Prem_{p}}\}$, for all arguments $B_1,\ldots,B_n, A$\\[-14pt]

 \hspace{-5mm}  such that $\bigcup_{i=1}^n\mathtt{kr}(B_i)$ \sm{$\lhd_{\mathtt{s}}$} $\mathtt{kr}(A)$, it holds that for some $i = 1\ldots n$, \mbox{$\mathtt{kr}(B_i) \sm{\triangleleft_{\mathtt{s}}} \mathtt{kr}(A)$}
%  :\\[-17pt]
%
%%  \item letting $B_1,\ldots,B_n, A$ be arguments such that $\bigcup_{i=1}^n\mathtt{kr}(B_i) \lhd \mathtt{kr}(A)$, where $\mathtt{kr} \in \{\mathtt{LastDefRules},\mathtt{DefRules},\mathtt{Prem_{p}}\}$, then:\\[-17pt]
%
%
% \begin{enumerate}
%\item\label{Lem_lhd_discrete-1} for some $i = 1\ldots n$, $\mathtt{kr}(B_i) \trianglelefteq \mathtt{kr}(A)$; and\\[-17pt]
%\item\label{Lem_lhd_discrete-2} for some $i = 1\ldots n$, $\mathtt{kr}(A) \ntrianglelefteq \mathtt{kr}(B_i)$.
%
%\end{enumerate}

\end{quote}

\end{definition}

The last and weakest link orderings are reasonable under the assumption that they are defined on the basis of reasonable inducing set comparisons. % $\trianglelefteq_{\mathtt{s}}$ that is reasonable inducing.

\begin{proposition}\label{PropLastReasonable} Let $\preceq$ be defined according to the last-link principle, based on a reasonable inducing $\triangleleft_{\mathtt{s}}$. Then $\preceq$ is reasonable.
\end{proposition}

\begin{proposition}\label{PropWeakestReasonable} Let $\preceq$ be defined according to the weakest-link principle, based on a reasonable inducing $\triangleleft_{\mathtt{s}}$. Then $\preceq$ is reasonable.

\end{proposition}

The following propositions imply that the last and weakest link orderings $\preceq$ defined in Definitions \ref{Def Last Link} and \ref{Def Weakest Link}  are reasonable.

\begin{proposition}\label{PropReasonableInducingEli} $\triangleleft_{\mathtt{Eli}}$ is reasonable inducing.
\end{proposition}

\begin{proposition}\label{PropReasonableInducingDem} $\triangleleft_{\mathtt{Dem}}$ is reasonable inducing.
\end{proposition}
Finally, if \sm{$\triangleleft_{\mathtt{s}}$} is transitive, then the strict weakest and last link orderings $\prec$ are strict partial orders:
\begin{proposition}\label{PropLastStrictPartial} Let $\prec$ be defined according to the last-link principle, based on a set comparison $\triangleleft_{\mathtt{s}}$ that is a strict partial order. Then $\prec$ is a strict partial order.
\end{proposition}

\begin{proposition}\label{PropWeakestStrictPartial} Let $\prec$ be defined according to the weakest-link principle, based on a set comparison $\triangleleft_{\mathtt{s}}$ that is a strict partial order. Then $\prec$ is a strict partial order.
\end{proposition}

%{\bf *** HP:  As I told you earlier, the democratic ordering is not admissible in the sense of my JAC-10 paper. Now you have proven that this ordering is also reasonable, so if I am right, you have thus proven that there are non-admissible orderings that are still reasonable. Then it might be worth noting that we have also generalised the results of JAC-10 in this respect (provided that it follows that any admissible ordering satisfies condition (1) of Def 16 and that condiiton 2 of Def 16 is implied by my JAC-10 definition of a reasonable argument ordering.) ***}

\subsection{Reconstructing and extending the abstract logic approach as an instance of \emph{ASPIC}$^+$}\label{sectionALstrict}

%\textbf{*** SM: Might be better to title this section `Reconstructing Amgoud and Besnard's abstract logic approach as instance of \emph{ASPIC}$^+$' ? Also, I presume it is not possible to show `c-classicality' as well as well-formedness and contraposition, and so show that all postulates are satisfied ?
%***}

In \cite{a+b09,a+b10}, Amgoud \& Besnard present an abstract approach to defining the structure of arguments and attacks, based on Tarski's notion of an \emph{abstract logic}.

\begin{definition}\label{DefAbstractLogic}[\textbf{Abstract Logic}] An abstract logic is a pair $(\LA,Cn)$, where $\LA$ is a language and the consequence operator $Cn$ is a function from $2^{\LA}$ to $2^{\LA}$ satisfying the following conditions for all $X \subseteq \LA$:\\[-15pt]

\ben
\item $X \subseteq Cn(X)$\\[-15pt]
\item $Cn(Cn(X)) = Cn(X)$\\[-15pt]
\item $Cn(X) = \bigcup_{Y \subseteq_f X} Cn(Y)$\\[-15pt]
\item $Cn(\{p\}) = \LA$ for some $p \in \LA$\\[-15pt]
\item $\Cn(\emptyset) \neq \LA$
\een
Here $Y \subseteq_f X$ means that $Y$ is a finite subset of $X$. A set $X \subseteq \LA$ is defined as \emph{consistent} if $Cn(X) \neq \LA$, and as inconsistent otherwise.
\end{definition}

\noindent Amgoud \& Besnard  \cite{a+b09} note that the following properties hold:\\[-15pt]

\ben\addtocounter{enumi}{5}
\item If $X \subseteq X'$ then $Cn(X) \subseteq Cn(X')$ (\emph{monotonicity})\\[-15pt]
\item If $Cn(X) = Cn(X')$ then $Cn(X \cup Y) = Cn(X' \cup Y)$
\een

\noindent\cite{a+b09} also restricts its focus to so-called \emph{adjunctive} abstract logics:\\[-15pt]
\ben\addtocounter{enumi}{7}
\item $\forall x,y \in \LA$ such that $Cn(\{x,y\}) \not= Cn(\{x\}), Cn(\{x,y\}) \not= Cn(\{y\})$, $\exists z$ such that $z \neq x$, $z \neq y$ and $Cn(\{z\}) = Cn(\{x,y\})$ \footnote{For example, classical logic is adjunctive because of the conjunction connective.}.
\\[-15pt]
\een

They then define arguments and various kinds of attack relations, and  investigate consistency properties of various types of attack relations when instantiating Dung's framework with arguments and attacks. We discuss this part of their work in Section~\ref{SectionRelatedWork}. We repeat here \cite{a+b10}'s notion of an undermining attack\footnote{\cite{a+b10} call undermining ``undercutting'' but to be consistent with \ASPIC's terminology we rename it to `undermining'.}. We also extend their approach to accommodate a pre-ordering over the formulae in an abstract logic theory.

\begin{definition}\label{ALarg+ALattack}[\textbf{Arguments and attacks in abstract logics}]
Let $(\LA,Cn)$ be an abstract logic and $(\Sigma,\leq)$ a theory in $(\LA,Cn)$, where $\Sigma \subseteq \LA$ and $\leq$ a preordering over $\Sigma$:\\[-13pt]
\bit
\item an \emph{AL-argument} is a pair $(X,p)$ such that: 1) $X \subseteq \Sigma$; 2) $X$ is consistent; 3) $p \in Cn(X)$; 4) no proper subset of $X$ satisfies (1-3).\\[-17pt]
%\item  $(X,p)$ \emph{AL-rebuts}  $(Y,q)$ if $\{p,q\}$ is inconsistent.\\[-17pt]
\item $(X,p)$ \emph{AL-undermines} $(Y,q)$ if there exists a $q' \in Y$ such that $\{p,q'\}$ is inconsistent.\\[-17pt]

\eit
\end{definition}

%$(\Sigma,\leq)$ in $(\LA,Cn)$ are the extensions and justified arguments of
%
%the present section we show that if \ASPIC's rules are generated by an abstract logic, the resulting $AT$ is closed under contraposition and well formed (two conditions needed for making the $AT$ well defined and so satisfy the rationality postulates).

 We formally define the notion of an \ASPIC\ argumentation theory based on an abstract logic with preferences. This involves defining the set of strict rules in terms of the abstract-logic's consequence notion but also relating the $\mbox{}^{-}$ relation to the Tarskian notion of consistency. Note that the latter does not allow for defining asymmetric contrary relations, and so we have to assume that \ASPIC's $\mbox{}^{-}$ relation is symmetric. Next, two conditions are needed to relate the $\mbox{}^{-}$ relation to the Tarskian notion of consistency. Firstly, if two formulas are contradictories of each other then they are jointly inconsistent. Secondly, if two formulas are jointly inconsistent, then each of them has a consequence that is a contradictory of the other.
%%%
%\marginpar{Henry, I'm a bit confused by `to indicate any or some ...' ?}
%%%
Also, a knowledge base will consist of the elements of an abstract logic theory as ordinary premises, while the argument ordering will be defined in terms of a preordering on the abstract logic theory. %{\bf *** HP: the following sentence can be deleted if we drop requirement (2) in Def 24(2), and I think we should drop that requirement, since we do not use it anywhere.  ***} In what follows, note the slight abuse of notation, whereby if $S = \{s_1,\ldots,s_n\}$, we write $S \imp p \in \R_s$ to indicate any or some $s_1,\ldots,s_n$ (letting the context disambiguate). %likewise for the notation of arguments with $\imp$.
Finally, to avoid confusion we henceforth refer to the abstract logic notion of consistency as `AL-consistency'.

\begin{definition}\label{asbasedonal}[\textbf{$AT$ and \emph{c-SAF} based on abstract logic with preferences}] Let $(\LA',Cn)$ be an abstract logic and $(\Sigma,\leq')$ a theory in $(\LA',Cn)$. An \emph{abstract logic (AL) argumentation theory} based on $(\LA',Cn)$ and $(\Sigma,\leq')$, is a pair $(AS,\K)$ such that $AS$ is an argumentation system $(\LA,\mbox{}^{-},\R,n)$ based on $(\LA',Cn)$, where:
%%%\marginpar{Sanjay, please check changes}
%%%
\ben
\item $\LA = \LA'$;
\item\label{AL-RS} $\R_d = \emptyset$,  and for all finite $S \subseteq \LA$ and $p \in \LA$, $S \imp p \in \R_s$ iff $p \in Cn(S)$;
\item\label{concon} $\mbox{}^{-}$ is defined such that:
\ben
\item\label{concona} if $\varphi \in \con{\psi}$ then $\psi \in \con{\varphi}$;
\item\label{conconb} if  $\varphi \in \con{\psi}$ then $\{\varphi,\psi\}$ is AL-inconsistent;
\item\label{conconc} if $\{\varphi,\psi\}$ is AL-inconsistent then there exists a $\varphi' \in Cn(\{\varphi\})$ such that $\varphi' \in \con{\psi}$.
\item\label{concond} $\con{\varphi}$ is nonempty for all $\varphi$.
\een
\een
\noindent $\K$ is a knowledge base  such that $\K_n = \emptyset$ and $\K_p = \Sigma$.\\[2pt]
\noindent $(\A,\C,\preceq)$ is the \emph{c-SAF} based on $(AS,\K)$, as defined in Definition \ref{DefinitionStructuredAF} and where $\preceq$ is defined in terms of $\leq'$ as in Section~\ref{SubSection-ASPIC+WeakestLast}. We also say that $(\A,\C,\preceq)$ is the \emph{c-SAF} based on $(\LA',Cn)$ and $(\Sigma,\leq')$.
\end{definition}

We can then show that a \emph{c-SAF} based on an abstract logic with preferences is well defined:

\begin{proposition}\label{ab=aspic}
A \emph{c-SAF} based on an AL argumentation theory is closed under contraposition, axiom consistent, c-classical, and well-formed.
\end{proposition}

Since $\preceq$ is reasonable, Proposition \ref{ab=aspic} implies that all the results and rationality postulates in Sections \ref{Section-ASPIC+-Propoerties-Attack} and \ref{Section-ASPIC+R-RationalityPostulates} hold for \emph{c-SAF}s based on an abstract logic with preferences. However, note that these \emph{c-SAF}s are instantiated by \ASPIC\ arguments and undermining attacks (since rebuts and undercuts only apply to defeasible rules). The question naturally arises as to whether they are equivalent to the AL arguments and undermining attacks in Definition \ref{ALarg+ALattack}. We first show that the attacks are indeed equivalent. To do so, we define the notion of an \emph{AL-c-SAF}:

\begin{definition}\label{DefAL-ASPIC-Attacks} Let an \emph{\ASPIC-AL-undermining} attack be defined in the same way as an
\ASPIC undermining attack, with `$\mathtt{Conc}(A) \vdash -\varphi$' replacing `$\mathtt{Conc}(A) \in \con{\varphi}$' in Definition \ref{DefAttacks}.\\[2pt] Then an \emph{AL-c-SAF} defined by $(AS,\K)$ is defined as in Definition \ref{asbasedonal}, with `$(X,Y) \in \C$ iff $X$ \ASPIC-AL-undermines $Y$' replacing `$(X,Y) \in \C$ iff $X$ attacks $Y$' in Definition \ref{DefinitionStructuredAF}.
\end{definition}

 Given Lemma \ref{LemmaFact} in Section \ref{SectionProofsAL}, which shows that $\mathtt{Conc}(A) \vdash -\varphi$ iff $\{\mathtt{Conc}(A),\varphi\}$ is AL-inconsistent, then the following result shows that \ASPIC's undermining attacks faithfully reconstruct abstract logic undermining attacks:
%%%\marginpar{Sanjay, check additions ``and $\leq'$''}
%%%
\begin{proposition}\label{ALprop}\footnote{
Note that in the Appendix, the proof of Proposition \ref{ALprop} shows that the result holds under both attack \emph{and} defeat definitions of conflict free. } Let $(AS,\K)$ be based on $(\LA',Cn)$ and $(\Sigma,\leq')$. Let $\Delta_1$ be the \emph{c-SAF} defined by $(AS,\K)$ and $\leq'$, and $\Delta_2$ the \emph{AL-c-SAF} defined by $(AS,\K)$ and $\leq'$. Then, for $T \in \{$complete, grounded, preferred, stable$\}$, $E$ is a $T$ extension of $\Delta_1$ iff $E$ is a $T$ extension of $\Delta_2$.
\end{proposition}

Now, observe that we do \emph{not} have an equivalence between \ASPIC\ arguments and AL arguments, because the latter imposes a subset minimality condition on the premises. This condition is not imposed on \ASPIC\ arguments in Definition~\ref{arg}, and neither is it implied by the definition of $\R_s$ in Definition \ref{asbasedonal}.
%\footnote{Subset minimality ensures an argument contains only the premises relevant to proving the argument's conclusion; a requirement met by \ASPIC\ arguments (see footnote in Definition \ref{arg})}.
Consider the following counter-example. Given $q \in Cn(\{p\})$, $s \in Cn(\{p,r\})$ and $q \in Cn(\{s\})$, we obtain $\R_s = \{p \imp q; p,r \imp s; s \imp q \}$. Then we have the strict arguments $\{p\} \vdash q$ and $\{p,r\} \vdash q$ where the latter is not subset minimal.

In general, minimality of premise sets is undesirable. Suppose a defeasible rule $p \Imp q$ and a strict rule $p,r \imp q$: then we clearly do not want to rule out an argument for $q$ with premises $p$ and $r$, since it could well be stronger than the defeasible argument. However, since the \ASPIC\ arguments defined by an AL argumentation theory are strict, we define here the notion of premise minimal \ASPIC\ arguments and show an equivalence with AL arguments.

\begin{definition}\label{DefinitionPremiseMinimal}[\textbf{Premise minimal \ASPIC\ arguments}]
Let for any argument $A$, $A_-$ be any argument such that $\mathtt{Prem}(A_-)$ $\subseteq \mathtt{Prem}(A)$ and $\mathtt{Conc}(A_-) =  \mathtt{Conc}(A)$. Given a set of \ASPIC\ arguments $\A$, let $\A^-$ = $\{A \in \A \mid$ there is no $A_- \in \A$ such that $\mathtt{Prem}(A_-) \subset \mathtt{Prem}(A)\}$ be the premise minimal arguments in $\A$.
\end{definition}

\begin{proposition}\label{abarg=aspicarg} Let $(AS,\K)$ be based on $(\LA',Cn)$ and $(\Sigma,\leq')$. Then $A$ is a c-consistent premise minimal argument on the basis of $(AS,\K)$ iff $(\mathtt{Prem}(A),\mathtt{Conc}(A))$ is an abstract logic argument on the basis of $(\Sigma,\leq')$.
\end{proposition}

We can then show that for \emph{c-SAF}s \emph{and} \emph{SAF}s, when restricting consideration to arguments with minimal premise sets, the conclusions of arguments in complete extensions remains unchanged, under the assumption that an argument cannot be strengthened by just adding premises. The latter is formulated by requiring that if  $B$ is not strictly preferred to $A$ then $B$ is not strictly preferred to $A_-$ (since if $B$ were strictly preferred to $A_-$ this would imply that $A_-$ has been strengthened by adding premises to obtain $A$).

\begin{proposition}\label{classicalmin}
Let $\Delta$ be the \emph{(c)-SAF} $(\A,\C,\preceq)$ defined on the basis of an $AT$ for which $\preceq$ is defined such that for any $A \in \A$,  if $A \nprec B$ then $A_- \nprec B$.
\\[2pt]
 Let $\Delta^-$ be the premise minimal \emph{(c)-SAF} $(\A^-,\C^-,\preceq^-)$ where:\\[-17pt]

\begin{itemize}
\item $\A^-$ is the set of premise minimal arguments in $\A$.\\[-17pt]
\item $\C^-$ = $\{(X,Y) | (X,Y) \in \C, X,Y \in \A^-\}$.\\[-17pt]
\item $\preceq^-$ = $\{(X,Y) | (X,Y) \in \preceq, X,Y \in \A^-\}$.
\end{itemize}

 Then for $T \in \{$complete, grounded, preferred, stable$\}$, $E$ is a $T$ extension of $\Delta$ iff $E'$ is a $T$ extension of $\Delta^-$, where $E'\subseteq E$ and
$ \bigcup_{X\in E}\mathtt{Conc}(X) = \bigcup_{Y\in E'}\mathtt{Conc}(Y)$.

\end{proposition}

Note that although the above proposition assumes the attack definition of conflict free, it immediately follows from Proposition \ref{Prop-Equivalence-a-d-admissible}  that  Proposition \ref{classicalmin} also holds  if the defeat definition of conflict free is assumed.

\begin{corollary}\label{classicalminCorollary} Given $\Delta$ and $\Delta^-$ as defined in Proposition \ref{classicalmin}:
\begin{enumerate}
  \item $\varphi$ is a $T$ credulously (sceptically) justified conclusion of $\Delta$ iff $\varphi$ is a $T$ credulously (sceptically) justified conclusion of $\Delta^-$.\\[-17pt]
  \item $\Delta^-$ satisfies the postulates \emph{closure under strict rules}, \emph{direct consistency}, \emph{indirect consistency} and \emph{sub-argument closure}.

\end{enumerate}

\end{corollary}

The assumption that arguments are not strengthened by adding premises is not satisfied by all ways of defining $\preceq$. Consider a \emph{c-SAF} $(\A,\C,\preceq)$ defined by an AL argumentation theory, where $\preceq$ is defined on the basis of the democratic comparison $\triangleleft_{\mathtt{Dem}}$, and suppose arguments $A_-$, $A$ and $B$ such that $\mathtt{Prem}(A_-)$ = $\{p\}$, $\mathtt{Prem}(A)$ = $\{p,q\}$, $\mathtt{Prem}(B)$ = $\{r\}$, and assume the preordering on the premises is \sm{$p <' r$}. Then $\{p,q\} \sm{\ntriangleleft_{\mathtt{Dem}}} \{r\}$, but $\{p\} \sm{\triangleleft_{\mathtt{Dem}}} \{r\}$, and so it is easy to verify that $A \nprec B$, but $A_- \prec B$.
However, the assumption is satisfied by the elitist \sm{$\triangleleft_{\mathtt{Eli}}$}:

\begin{proposition}\label{PropStrengtheningAssumption} Let $(\A,\C,\preceq)$ be defined by an AL argumentation theory, where $\preceq$ is defined under the weakest or last link principles, based on the set comparison \sm{$\triangleleft_{\mathtt{Eli}}$}. Then $\forall A,B \in \A$, $\forall A_- \in \A$, if $A \nprec B$ then $A_- \nprec B$.

\end{proposition}

\noindent In conclusion:\\[-7pt]\

\hspace{5mm}Let  $\Delta$ = $(\A,\C,\preceq)$ be the \emph{c-SAF} based on $(\LA',Cn)$ and $(\Sigma,\leq')$, \\[-10pt]

\hspace{5mm}as defined in Definition \ref{asbasedonal}.
\\[3pt]
We have shown that all results and rationality postulates hold for $\Delta$. We have also shown that the AL undermining attacks and $\C$ are equivalent in the complete extensions that they generate, and the AL arguments and premise minimal \ASPIC\ arguments in $\A$ are equivalent.

Furthermore, Proposition \ref{classicalmin} and Corollary \ref{classicalminCorollary} then imply that:
%%%\marginpar{Sanjay, please check changes}
%%%
\begin{remark}\label{ALRemark}
1) We have combined \cite{a+b09,a+b10}'s abstract logic approach to argumentation (which assumes no preference relation over $\Sigma$) with the \ASPIC\ framework, in that the justified conclusions of a Dung framework instantiated by AL arguments and AL undermining attacks, are exactly those of $\Delta$. We have also shown that Section \ref{Section-ASPIC+R-RationalityPostulates}'s rationality postulates hold for Amgoud \& Besnard's approach.
\\[4pt]
2) Given that we have extended the abstract logic approach to accommodate preferences, consider a preference-based argumentation framework (see Section \ref{SectionBackground}) $\Gamma$ = $(\A',\C',\preceq')$ defined by  $(\Sigma,\leq')$ and $(\LA',Cn)$, where $\A'$ and $\C'$ are the AL arguments and AL undermining attacks. Then, under the assumption that $\preceq'$ does not strengthen arguments when adding premises, the justified conclusions of $\Gamma$ (specifically the Dung framework instantiated by arguments and defeats defined by $\Gamma$) are exactly those of $\Delta$. This assumption is satisfied when
defining $\preceq'$ under the last or weakest link principles, based on Definition \ref{Def-hd12}'s elitist set comparison that utilises the preordering $\leq'$ over formulae in $\Sigma$. We have also shown that Section \ref{Section-ASPIC+R-RationalityPostulates}'s rationality postulates hold for Amgoud \& Besnard's approach extended with preferences.
\end{remark}
%%%\marginpar{Sanjay, please check changes}
%%%
We conclude by observing that Amgoud \& Besnard investigate the consistency of extensions of a Dung framework instantiated by the arguments and attacks defined by an abstract logic. Specifically, they consider whether for any extension the union of the premises of the extension's arguments is AL consistent. We now show that for a \emph{c-SAF} based on an abstract logic this is equivalent to indirect consistency, and then refer to this result in Section \ref{SectionRelatedWork} in which we compare \ASPIC\ and the abstract logic approach.

\begin{proposition}\label{ALprop0}  Let $\Delta$ be the \emph{c-SAF} based on $(\LA',Cn)$ and $(\Sigma,\leq')$. Then for any complete extension $E$ of $\Delta$: $S = \{\phi | \phi \in \mathtt{Prem}(A), A \in E\}$ is AL-inconsistent iff $S' = Cl_{\R_s}(\{\mathtt{Conc}(A) | A \in E\})$ is inconsistent.
\end{proposition}

\subsection{Classical logic instances of the \emph{ASPIC}$^{+}$ framework}\label{SectionClassicalLogicInstantiations}

The previous section's results allow us to
reconstruct classical logic approaches to argumentation as a special case of \emph{ASPIC}$^{+}$, and in so doing extend these approaches with preferences. We also prove a relation with Brewka's preferred subtheories.
%%%
%\marginpar{Add small examples to 5.3.1 and 5.3.2?}
%%%
\subsubsection{Defining classical logic instantiations of \emph{ASPIC}$^{+}$}\label{SectionDefininingClassicalLogicInstantiations}

Much recent work on structured argumentation formalises arguments as minimal classical consequences from consistent and finite premise sets in standard propositional or first-order logic \cite{AmCay,b+h01,b+h08,g+h11}. Since classical logic can be
specified as a Tarskian abstract logic $(\LA',Cn)$, where $\LA'$ is a a standard propositional or first-order language and $Cn$ the classical consequence relation, a \emph{classical argumentation theory} and its \emph{c-SAF} based on $(\LA',Cn)$ and an ordered theory $(\Sigma,\leq')$, is defined as in Definition \ref{asbasedonal}. The ordering $\leq'$ on $\Sigma$ and thus the ordinary premises, allows us to reconstruct classical logic approaches that additionally consider preferences (e.g., \cite{AmCay}). It is then easy to verify that if $\mbox{}^{-}$ is defined as classical negation, then all four conditions in Definition \ref{asbasedonal}-(3) are satisfied.

%Writing `$\vdash_c$' to denote classical consequence, then:
%
%\begin{definition}\label{DefClassicalArguments}
%Given a propositional or first order theory $\Sigma$, an argument is of the form $(\Gamma',\alpha)$ such that $\Gamma'$ is finite,
%$\Gamma' \subseteq \Sigma$, $\Gamma' \vdash_c \alpha$ and $\Gamma'' \nvdash_c \alpha$ for all $\Gamma'' \subset \Gamma'$
%\end{definition}

Amongst the above-cited works on classical argumentation, \cite{AmCay} and \cite{g+h11} adopt a Dung-style semantics, where only \cite{AmCay} considers preferences. Let us first consider \cite{g+h11}. They define seven alternative notions of attack and investigate their properties, including the rationality postulates of \cite{c+a07} studied in this paper. They show that the only two attack relations that are `well behaved', in the sense that they satisfy consistency postulate for all the semantics, are the so called `direct undercuts' and `direct defeaters':\\[-17pt]

\begin{itemize}
\item $Y$ \emph{directly undercuts} $X$ if $\mathtt{Conc}(Y) \equiv \neg p$ for some $p \in \mathtt{Prem}(X)$\\[-17pt]
\item $Y$ \emph{directly defeats} $X$ if $\mathtt{Conc}(Y) \vdash_c \neg p$ for some $p \in \mathtt{Prem}(X )$
\end{itemize}

Although our undermining attacks are not among \cite{g+h11}'s seven notions of attack, it can be shown that our undermining attacks are equivalent to their direct undercuts and defeats in that the complete extensions generated are the same. For direct defeats, this result is shown by Proposition \ref{ALprop}. For direct undercuts, it suffices to adapt the proof of Proposition~\ref{ALprop}, showing that: 1) if $Y$ undermines $X$, then letting $\mathtt{Conc}(Y) = q$, by the symmetry of classical negation $q = \neg p$ for some $p \in \mathtt{Prem}(X)$ and so $\mathtt{Conc}(Y)) \equiv \neg p$; 2) if $Y$ directly undercuts $X$ then $Y$ directly defeats $X$, and so as already shown, $Y$ undermines $X$. These equivalences and \cite{g+h11}'s negative results for their remaining five notions of attack justify why \ASPIC\ does not model these five notions.

It follows from the above, and the results and discussion in Section \ref{sectionALstrict}, that we have reconstructed and extended with preferences, \cite{g+h11}'s variants with direct undercut and direct defeat, and shown that \cite{c+a07}'s postulates are satisfied for classical logic approaches with preferences (recall that \cite{g+h11}'s other variants violate the consistency postulate even without preferences).

\begin{figure}[h]
\centering
\includegraphics[width=4.4in]{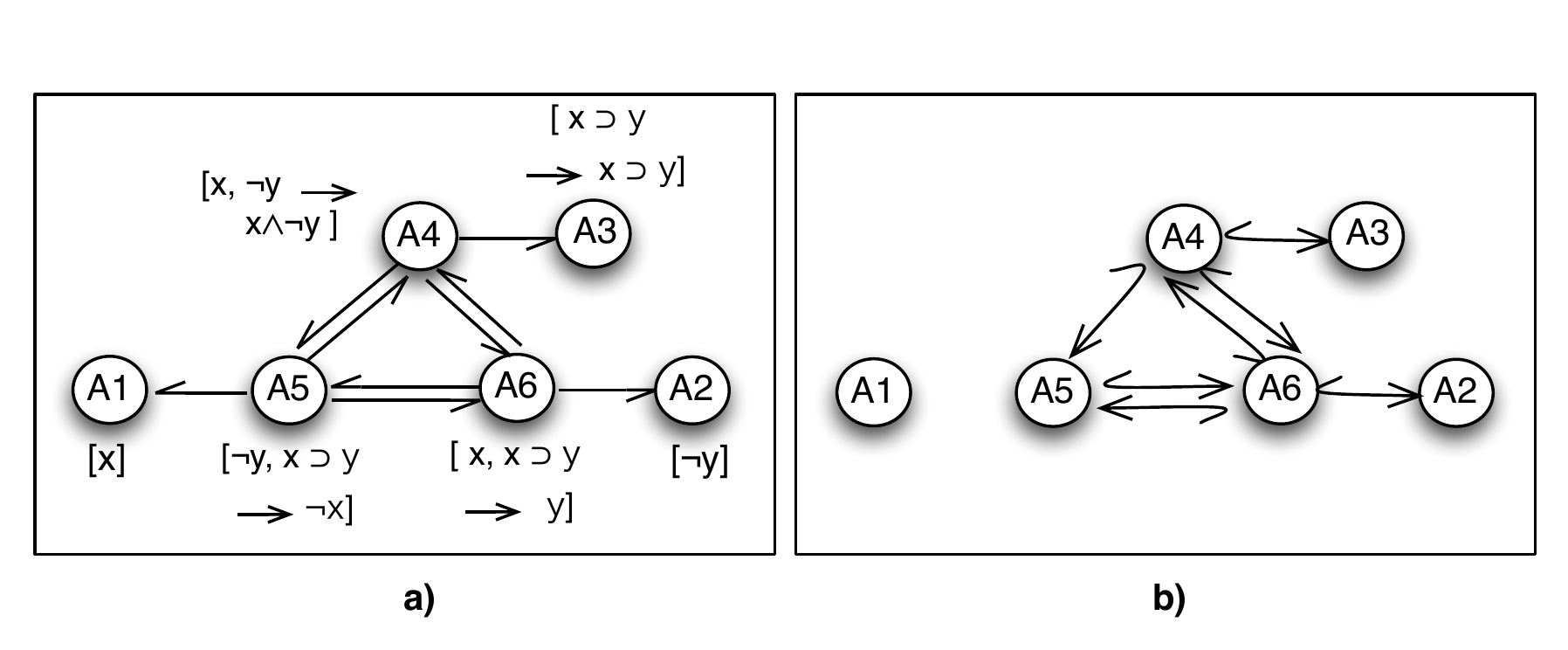}
\caption{Classical Logic argumentation: attack graph (a) and defeat attack (b)}\label{FigClassicalLogicEx}
\end{figure}

\begin{example}\label{ExampleClassicalLogic} Let the ordinary premises be the set $\Sigma$ = $\{  x, \neg y, x \supset y \}$  ($\supset$ denotes material implication) and assume $x >' \neg y, x >' x \supset y$. The attack graphs is shown in Figure \ref{FigClassicalLogicEx}-a). Under either the weakest or last link principles, and assuming either the elitist or democratic comparisons, $A_5 \prec A_1$ and so $A_5$ does not defeat $A_1$. Note also that $A_5$ attacks $A_4$ on $A_1$, and so $A_5$ does not defeat $A_4$. The defeat graph is shown in Figure \ref{FigClassicalLogicEx}-b).\\[2pt]We obtain $E_1'$ = $\{A_1, A_4, A_2\}$ and $E_2$ = $\{A_1, A_6, A_3\}$, where by satisfaction of the closure under strict rules postulate, $E_1$ = $E_1'$ extended with arguments concluding classical consequences of $\{x,\neg y, x \neg y\}$ is a preferred/stable extension, and $E_2$ = $E_2'$ extended with arguments concluding classical consequences of $\{x,y, x \supset y\}$ is a preferred/stable extension.

\end{example}

The above example shows how preferences arbitrate in favour of the sceptically justified conclusion $x$ over $\neg x$ .
Indeed, we argue that extending classical logic approaches with preferences is of particular importance, given that (as shown in \cite{Cay95,g+h11}) the preferred/stable extensions generated from a Dung framework instantiated by arguments and direct undercuts or defeats, simply correspond to the maximal consistent subsets of the theory $\Sigma$ from which the arguments are defined\footnote{In the sense that the union of formulae in the supports of arguments in each preferred/stable extension is a maximal consistent subset of $\Sigma$.}. Intuitively, one would expect this correspondence given that classical logic does not provide any logical machinery for arbitrating conflicts (in contrast with the use of undercuts and negation as failure in non-monotonic logics (as discussed in Sections \ref{Section-ArgumentationAndLogic-RoleOfPreferences} and \ref{Section-ASPIC+-AttacksAndDefeats}).
 One must therefore resort to some meta-logical mechanism, such as preferences, if argumentation is to be usefully deployed in resolving inconsistencies in a classical-logic setting.

We conclude by noting that \cite{AmCay} make use of preferences to determine the success of two of \cite{g+h11}'s variants of attack, and show that this leads to violation of the consistency postulates. We will discuss this in detail in Section~\ref{SectionRelatedWork}.

\subsubsection{Brewka's Preferred Subtheories as an instance of the \emph{ASPIC}$^{+}$ framework}\label{SectionPreferredSubtheories}

Brewka's \emph{preferred subtheories} \cite{BrewkaPS} models the use of an ordering over a classical propositional or first order theory $\Gamma$, in order to resolve inconsistencies. It has therefore been used to both formalise default reasoning and belief revision \cite{BrewkaNMReasoning}.

\begin{definition}\label{DefPreferredSubtheories} A default theory $\Gamma$ is a tuple $(\Gamma_1,\ldots, \Gamma_n)$, where each $\Gamma_i$ is a set of
formulae in a classical first order language $\LA'$. A preferred subtheory is a set $\Sigma$ = $\Sigma_1 \cup \ldots \cup \Sigma_n$ such that for $i = 1 \ldots n$, $\Sigma_1 \cup \ldots \cup \Sigma_i$ is a maximal (under set inclusion) consistent subset of $\Gamma_1,\ldots, \Gamma_i$

\end{definition}

Intuitively, a preferred subtheory is obtained by taking a maximal under set inclusion consistent subset of $\Gamma_1$, extending this with a maximal consistent subset of $\Gamma_2$, extending this with a maximal consistent subset of $\Gamma_3$, and so on. We can reconstruct preferred subtheories as an instance of the \emph{ASPIC}$^{+}$ framework.

\begin{definition} Let $\Gamma$ be a default theory $(\Gamma_1,\ldots, \Gamma_n)$, and $\forall \alpha,\beta \in \Gamma$, $(\alpha,\beta) \in \leq'$ iff $\alpha \in \Gamma_i, \beta \in \Gamma_j$, $i \geq j$. Let $\Delta$ be the \emph{c-SAF} $(\A,\C,\preceq)$ based on $(\LA', Cn)$ and $(\Gamma,\leq')$ as described in Section \ref{SectionDefininingClassicalLogicInstantiations} (with $\Gamma$ replacing $\Sigma$), and where $\preceq$ is defined under the weakest or last link principle, and on the basis of the \sm{$\triangleleft_{\mathtt{Eli}}$} set comparison. We say that $\Delta$ is the \emph{c-SAF} corresponding to $\Gamma$.
\end{definition}

\begin{theorem}\label{TheoremPS} Let $(\A,\C,\preceq)$ be a c-\emph{SAF} corresponding to a default theory $\Gamma$, and for any $\Sigma \subseteq \Gamma$,
let $\mathtt{Args}(\Sigma) \subseteq \A$ be the set of all arguments with premises taken from $\Sigma$. Then:\\
\noindent 1) If $\Sigma$ is a preferred subtheory of $\Gamma$, then $\mathtt{Args}(\Sigma)$ is a stable extension of $(\A,\C,\preceq)$.\\[3pt]
\noindent 2) If $E$ is a stable extension of $(\A,\C,\preceq)$, then $\bigcup_{A \in E}\mathtt{Prem}(A)$ is a preferred subtheory of $\Gamma$.

\end{theorem}

Note that although the above theorem assumes the attack definition of conflict free, it immediately follows from Proposition \ref{Prop-Equivalence-a-d-admissible}  that  Theorem \ref{TheoremPS} holds  if the defeat definition of conflict free is assumed. Finally, also note that the above theorem paves the way for applying argument-game proof theories and labelling algorithms for the stable semantics \cite{ModCam}, to preferred subtheories, as well as studying the preferred subtheories approach under the full range of semantics defined for Dung frameworks.

\section{A Discussion of Some Related  Work}\label{SectionRelatedWork}

\subsection{Comparison with General Frameworks for Argumentation}\label{Section Comparison with General Frameworks for Argumentation}

In this section we compare \ASPIC\ to related work. To start with, the inclusion of defeasible rules in  \emph{ASPIC}$^+$ requires some explanation, given that much current work formalises the construction of arguments as deductive \cite{a+b09,a+b10}, and in particular classical \cite{b+h08,g+h11} inference. These approaches regard argumentation-based inference as a form of inconsistency handling in deductive logic; the supposed advantage being that the logic of deductive inference is well-understood \cite[p.\ 16]{b+h08}. This raises the question of whether defeasible inference rules are needed at all. Our answer is that the research history in our field shows that at best only part of argumentation can be formalised as inconsistency handling in deductive logic.  To start with, the distinction between strict and defeasible inference rules has a long history in AI research on argumentation \cite{LS89,loui87, Pol87, pol94, Pol95,PS97,S+L92,Vre97}, so a truly general framework for structured argumentation must include this distinction. Pollock in particular  provides philosophical arguments that appeal to epistemological accounts of human reasoning, so that the modelling of defeasible rules is a particularly salient requirement in light of the bridging role (discussed in Sections \ref{SectionIntroduction} and \ref{Section-ArgumentationAndLogic}) that argumentation plays between human and formal logic-based models of reasoning.

 Moreover, conceptually, defeasible reasoning is not about handling inconsistent information but about making deductively unsound but still rational `jumps' to conclusions on the basis of consistent but deductively inconclusive information. Consider the following well-known example, with the given information  that quakers are normally pacifists, that republicans are normally not pacifists and that Richard Nixon was both a quaker and a republican.  A defeasible reasoner is then interested in what can be concluded about whether Nixon was a pacifist \emph{while consistently accepting all the given information}.  The reason that they are jointly consistent is that  `If $q$ then normally $p$'and `$q$' does not deductively imply $p$ since things could be abnormal: Nixon could be an abnormal quaker (or republican). A defeasible reasoner therefore does not want to reject any of the above statements, but rather wants to assume whenever possible that things are normal, in order to jump to conclusions about Nixon in the absence of evidence to the contrary.  In other words, defeasible reasoning is not about inconsistency handling but about making uncertain inferences from consistent (though deductively inconclusive) premises. Therefore, attempts to formalise defeasible reasoning as inconsistency handling are at least unnatural. Moreover, the literature on nonmonotonic logic suggests that such attempts\footnote{Including those that make use of applicability predicates to simulate the effects of priorities/preferences.} are prone to validating counterintuitive inferences (see e.g. \cite{BrewkaNMReasoning,gin94} or \cite{hp12mjs} for a recent discussion in the context of argumentation). We therefore conclude that, given the research literature, it makes sense to include defeasible inferences in models of argumentation, and therefore any account of argumentation that claims to be general should leave room for them.

A number of works have been proposed as general approaches to argumentation. A well-known and established framework is that of assumption-based argumentation (\emph{ABA}) \cite{BDKT97}, which has made a substantial contribution to our understanding of argumentation, and is shown (in \cite{hp10aspicJAC}) to be a special case of the \ASPIC\ framework in which arguments are built from assumption premises and strict inference rules only and in which all arguments are equally strong. As mentioned earlier, when commenting on Definition~\ref{Definition-knowledge-base}, \cite{hp10aspicJAC}'s result on the relation between \ASPIC\ and \emph{ABA} also holds if all \emph{ABA} assumptions are translated as \ASPIC's ordinary premises. To see why, firstly, note that \emph{ABA} does not accommodate preferences over assumptions. Hence all undermining attacks on assumption premises are preference independent. Since the reconstruction of \emph{ABA} does not accommodate preferences, then undermining attacks on \emph{ABA} assumptions modelled in this paper as ordinary premises, also always succeed as defeats. One can thus straightforwardly replace \cite{hp10aspicJAC}'s assumption premises with ordinary premises, and show (as in \cite{hp10aspicJAC}) that \emph{ABA} can be faithfully reconstructed in this paper's formalisation of \ASPIC\ . Our work is relevant for \emph{ABA}, since \emph{ABA} does not in general satisfy \cite{c+a07}'s consistency postulates \footnote{Just as \ASPIC does not in general satisfy \cite{c+a07}'s postulates, since one is free to instantiate \ASPIC in ways that are not `well-defined' (Definition \ref{DefStrictDefAssumptions}).}. A simple counterexample is an \emph{ABA} deductive system with two rules $\imp p$ and $\imp \neg p$. Note that is not to suggest that \emph{ABA} is flawed; rather, we provide conditions (e.g. that rules be closed under transposition) under which \emph{ABA} satisfies \cite{c+a07}'s consistency postulates.

 More recently, Amgoud \& Besnard \cite{a+b09,a+b10} proposed the abstract-logic approach (AL) to defining structured argumentation. \ASPIC\ is considerably more complex than AL: firstly because \ASPIC\ models the use of preferences to resolve attacks, so it has to distinguish between attack and defeat, and secondly because \ASPIC\ combines deductive and defeasible argumentation, which means that not only the premises, but also the defeasible inferences of an argument can be attacked. This requires that the arguments' structure be made explicit in order to know which parts of an argument can be attacked. By contrast, if all inferences are certain, then arguments can only be attacked on their premises so their internal structure is irrelevant for their evaluation.
%%%
%\marginpar{Sanjay, please check changes in comments on relation AL-ASPIC+}
%%%
In Section \ref{sectionALstrict} we straightforwardly extended AL with preferences and then combined the extended AL with \ASPIC. However, AL cannot accommodate defeasible inferential rules. This is because the inferential reasoning from premises to conclusion is not rendered explicit, but rather is encoded in AL's single consequence operator, which cannot distinguish between strict and defeasible inference rules: there is no way to distinguish $p$'s and $S$'s for which $p \in Cn(S)$ implies that $S \imp p$ should be in $\R_s$ or $S \Imp p$ should be in $\R_d$. Furthermore, recall that in Section \ref{sectionALstrict} we argued that AL's subset minimality condition on premises is not appropriate when accounting for defeasible inference rules.

The inappropriateness of accommodating defeasible argumentation in AL is further illustrated when considering whether \ASPIC's notion of an argument \emph{generates} an abstract logic. If this is the case for a given instance of \ASPIC, then all of \cite{a+b09,a+b10}'s results hold for this instance. Suppose an \ASPIC\ $AT$ and $Cn$ defined as follows\footnote{Since for AL, consistency requirements on arguments are added on top of a \emph{given} consequence notion $Cn$, we cannot incorporate consistency requirements into the \emph{definition} of $Cn$, and so assume \ASPIC\ without the restriction to c-consistent arguments.}: \\[-14pt]
%\ben
% \item{(1)} $p \in Cn(X)$ iff there exists an \ASPIC\ argument $A$, with $\mathtt{Conc}(A) = p$ and $\mathtt{Prem}(A) = X$.
%\\[-14pt]
%\een
\btab{ll}
 (1) & $p \in Cn(X)$ iff there exists an \ASPIC\ argument $A$, with $\mathtt{Conc}(A) = p$ \\
 & and $\mathtt{Prem}(A) = X$.
\etab
It can be shown that conditions (1), (2) and (3) in the definition of an abstract logic (Definition \ref{DefAbstractLogic}) are satisfied. However (4) is in general not satisfied. Consider an $AT$ with $\K = \{p\}$, $R_s = \emptyset$ and $R_d = \{p \Imp q\}$. Also,
%However, under the assumption that $\R_s$ is generated by an abstract logic, (4) is obviously satisfied.
(5) is not in general satisfied. Consider any $AT$ with $\K = \emptyset$ and $R_d = \{ \Imp p \mid p \in \LA\}$.
%Note that this counterexample even holds if $R_s$ is generated by an abstract logic.

Of course, many instances of \ASPIC\ will satisfy (4) and (5), and thus generate abstract logics. But then we should interpret \cite{a+b09,a+b10}'s results, as they apply to these instances, with care. In particular, the notion of consistency of an abstract logic behaves in an unexpected way. Recall that Amgoud \& Besnard investigate whether for any Dung-extension $E$, the set $\bigcup_{(X,p) \in E}X$ is AL consistent (see end of Section \ref{sectionALstrict}). Now, consider an \ASPIC\ $AT$ formalising the above Nixon example in a language $\LA$ including atoms $p$, $r$ and $q$ respectively denoting `Nixon is a pacifist', `Nixon is a republican' and `Nixon is a quaker' and a connective $\leadsto$ for default conditionals. Informally, $\varphi \leadsto \psi$ means `if $\varphi$ then normally $\psi$'. Let the $\mbox{}^{-}$ relation correspond to classical negation, $\R_s$ contain all propositionally valid inferences (including $p ,\neg p \imp \varphi$ for any $\varphi \in \LA$) and $\R_d$ contain a defeasible modus ponens scheme $\varphi, \varphi \leadsto \psi \Imp \psi$. Then if $\K = \{q,r, q \leadsto p, r \leadsto \neg p\}$, any Dung-extension  contains all elements of $\K$ as arguments but does not contain arguments for both $p$ and $\neg p$ so any such extension satisfies \emph{indirect consistency} (the closure under strict rules is consistent). However, in the abstract logic generated by equation $(1)$, the set $\{q,r,q \leadsto p,r \leadsto \neg p\}$ is AL inconsistent, since there exists an \ASPIC\  argument for every $\varphi$ by combining the defeasible arguments for $p$ and $\neg p$ (even though this argument is not in any extension).

This discrepancy is caused by the fact that an abstract logic's consequence operator cannot distinguish between strict and defeasible inferences, and so regards a set $S$ as inconsistent if the closure of $S$ under \emph{both} strict and defeasible rules is directly inconsistent. But this consistency requirement is too strong, since the very idea of defeasible reasoning is that one's knowledge need \emph{not} be closed under defeasible inference, since defeasible inference rules can be defeated even if all their antecedents hold.

Concluding our comparison, the abstract logic approach provides a very interesting and insightful generalisation of earlier work on classical argumentation, but does not apply to mixed strict and defeasible argumentation, such as modelled in \ASPIC\ and earlier by many others. We should here emphasise our view that deductive approaches certainly do have their place in the study and application of of argumentation. However, we argue that a truly general account of argumentation should also accommodate the use of defeasible inference rules.

%The above discussion provides further reasons why the notion of an abstract logic does not fit well with the idea of
% defeasible inference rules. If an instance of \ASPIC\  formally generates an abstract logic at all, the abstract-logic notion of consistency cannot be meaningfully applied to it.

%
%{\bf *** HP: I deleted the following phrase   ``This example illustrates why any \emph{general} structured approach to argumentation should distinguish strict and defeasible inference rules''  (and moved much of th erest of the para in revised form to above). The reason is that strictly speaking the example only shows that AL cannot distinguish between strict and defeasible inference but if one thinks that all defeasible inference is in fact inconsistency handling in classical logic, then one will not draw the conclusion that AL is not general enough. In fact, this is also another reason to move the informal defence of defeasible rules to the beginning of this para. ***}

%Finally, we note that the research history in our field demonstrates requirements for including defeasible rules \cite{Pol87,loui87,LS89,SL92,Pol92,Pol94,Pol95,Vre97,PS97}.
%We therefore conclude that, given the research literature, it makes sense to include defeasible inferences in models of argumentation, and that while classical approaches certainly have their place in argumentation, any general approach should additionally include defeasible inference rules.\\

Amgoud \& Besnard \cite{a+b09,a+b10} also use the abstract logic approach to make some informal negative claims about the suitability of Dung-style semantics.
First of all, they informally claim \cite{a+b10} that to satisfy the consistency postulates, an attack relation should be \emph{valid} in the sense that when two arguments have jointly AL inconsistent premises, they should attack each other. However, this informal claim should be read with care: what they formally show is that validity is a \emph{sufficient} condition for consistency. Their results do not preclude ``invalid'' attack relations, such as undermining attacks, from satisfying consistency. Indeed, in \cite{hp10aspicJAC} and this paper we have identified alternative sufficient conditions for consistency. Furthermore, Amgoud \& Besnard themselves show that, under the assumption that \emph{a Dung framework contains all arguments that can be logically constructed}, their notion of consistency of extensions is satisfied assuming the AL undermining attacks in Definition \ref{ALarg+ALattack}, which is, of course, consistent with our more general result showing that AL satisfies all of \cite{c+a07} postulates  (Remark \ref{ALRemark} and Proposition \ref{ALprop0} in Section \ref{sectionALstrict}). Amgoud \& Besnard regard this assumption as problematic. However, we regard this not as a problem of the attack relation, but of the reasoner: if an imperfect reasoner is modelled who cannot be relied on to produce all relevant arguments, then perfect results cannot be expected. Furthermore, as argued in Section~\ref{Section-ArgumentationAndLogic-DefiningAttacks}, requiring that attacks be `valid' goes against the  dialectical role of attacks and has the computational problem in that it can give rise to infinitely many attacks.
%{\bf *** HP: Why not delete from here? It is such a small point. ***} Amgoud \& Besnard also remark that with undermining attack, different arguments with the same set of premises must be allowed for consistency, which would lead to problematic redundancy. However, we fail to see the problem: if a knowledge base contains a formula $p \land q$, then surely two different arguments for $p$ and $q$ should be allowed.  Also, in propositional logic any single formula implies an infinite number of other formulas.

%In fact, the argumentation community has long respected the insight that rebutting attack only applies to defeasible inferences, as in the systems of \cite{s+l92,pol95,p+s97,g+s04} and also in Defeasible Logic \cite{nute94}.

%
%\textbf{Henry, I've collected together your various remarks on rebutting attacks here. With regard to 1) below, I've found it difficult to follow your line of reasoning. With regard to 2), the counter-example describes a SAF that is not axiom consistent}:

%Finally, Amgoud \& Besnard also claim that their notion of rebutting attack -- $(X,p)$ \emph{rebuts} $(Y,q)$ if $\{p,q\}$ is AL inconsistent -- is not well-behaved, since it violates consistency. However in this paper we have shown that rebut attacks are appropriate in systems with defeasible inference rules.

Finally, in a recent publication \cite{AmgoudWeak}, Amgoud claims that the \ASPIC\ framework suffers from a number of weaknesses. Space limitations preclude a detailed assessment of these claims here, suffice it to say that the formal results in \cite{hp10aspicJAC} and in this paper, contradict a number of informal claims in \cite{AmgoudWeak}. Furthermore, the interested reader may consult a comprehensive rebuttal of \cite{AmgoudWeak}'s claims in \cite{Miscon}.

%{\bf *** HP: I still suggest to delete from here, since what I wrote here is nonsense. For example, rebutting a proper subargument results in asymmetric attck, and $p$ in $\K_n$ will asymmetrically rebut any defeasible argument for $-p$. I am sure we can construct an $AT$ with only symmetric attacks but it will be so simple that the rhetorical force of the example is lost. Finally, we don;'t really need to refute their argument on symmetric attacks: it's a relatively minor point.  ***} Furthermore, Amgoud \& Besnard claim that symmetric attacks are problematic with respect to consistency. However, consider an \ASPIC\ $AT$ with a propositional language, with non-empty $\K_n$ and empty $\K_p$ and $\K_a$, with no undercutters or preferences, and with $\R_s$ corresponding to propositional logic. If the contrary relation is symmetric, and all elements in $\K_n$ and all antecedents and consequents of defeasible rules are propositional literals. {\bf *** HP: comma (if not deleted). ***} Then all attacks are symmetric rebuts\footnote{Note that \ASPIC\ rebuts can be shown to be equivalent to AL rebuts in the same way that \ASPIC\ undermining attacks are shown equivalent to AL undermining attacks in Section \ref{sectionALstrict}.}, and our results show that consistency is satisfied.

\subsection{Comparison with other works on Preference-based Argumentation}\label{Section Comparison with other works on Preference-based Argumentation}

We now consider approaches that accommodate preferences to determine which attacks succeed as defeats. Recently, both Kaci \cite{kaci10} and Amgoud \& Vesic \cite{a+v10, AmgoudRepairPAFs, a+v10sum, AV11AMAI} have addressed the issue of how consistency can be ensured for instantiations of the preference-based argumentation frameworks (\emph{PAF}s) \cite{AC02a} reviewed in Section \ref{SectionBackground}. They all argue that instantiations of standard \emph{PAF}s have problems with unsuccessful asymmetric attacks. \cite{kaci10} argues that all attacks should therefore be symmetric. However, \cite{a+b09} show that for classical argumentation this would still lead to inconsistency problems.  Nevertheless, \cite{a+v10, AmgoudRepairPAFs, AV11AMAI} also criticise `standard' \emph{PAF} approaches, arguing that unsuccessful asymmetric attacks may violate consistency. As a solution they propose that unsuccessful asymmetric attacks should result in rejection of the attacker even if it is not attacked by any argument. However, our consistency results obviate the need for reversing unsuccessful attacks. We have shown that by taking into account the structure of arguments, one can show that if $A$ unsuccessfully attacks $B$, then either some sub-argument of $B$ defeats $A$, or under assumptions on the preference ordering, $B$ can be continuated into an argument that defeats $A$, and that this result is key for showing consistency as discussed in Section \ref{Section-Attack-Defeat-Comparison}.

\begin{figure}[h]
\centering
\includegraphics[width=4.4in]{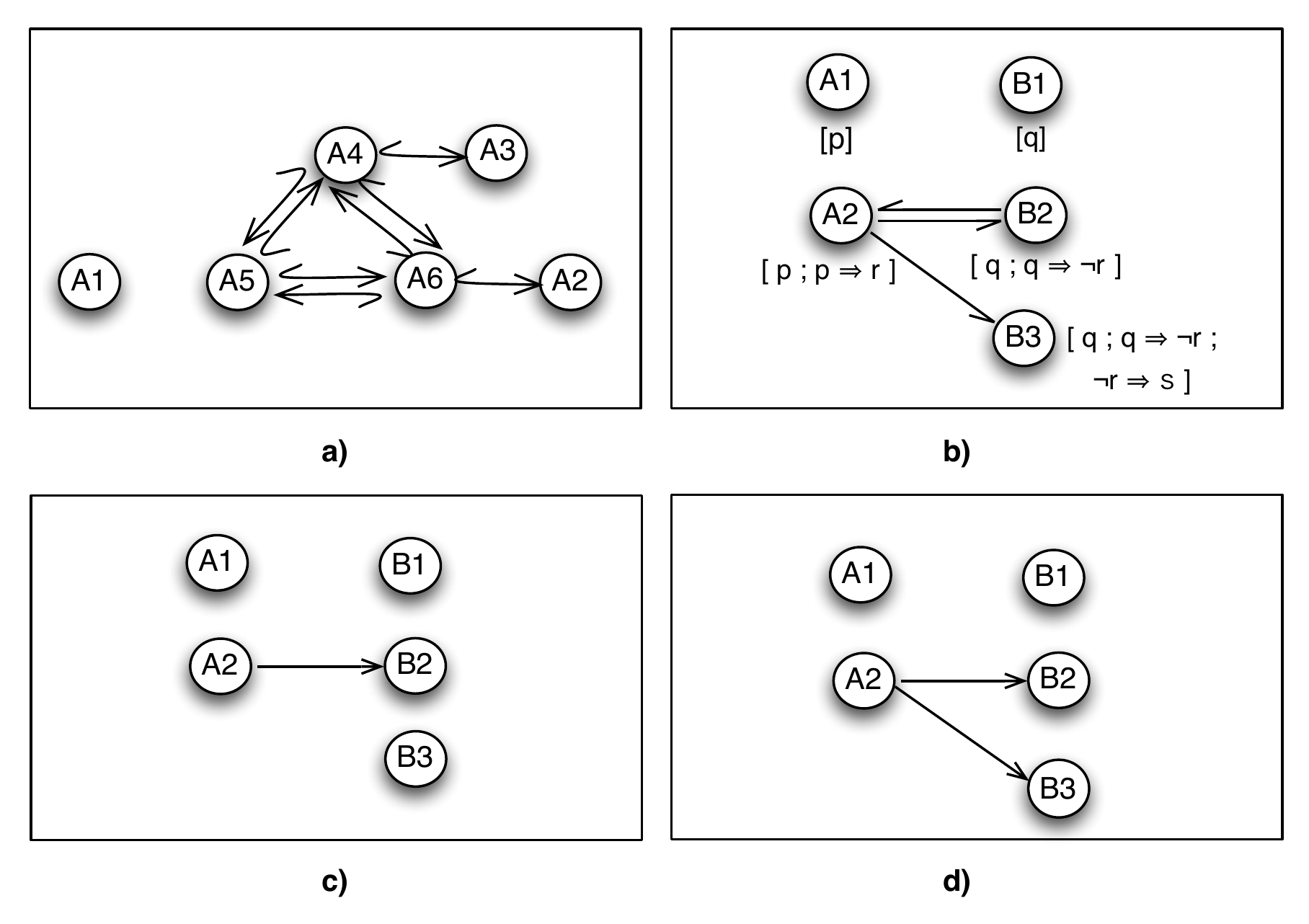}
\caption{Examples illustrating comparison with critiques of \emph{PAF}s}\label{PAFCritique}
\end{figure}

But how do our consistency results square with \cite{a+v10, AmgoudRepairPAFs, AV11AMAI}'s examples of inconsistent \emph{PAF}s? \cite{AmgoudRepairPAFs} give a semiformal example which we described earlier in Section \ref{Section-Attack-Defeat-Comparison} (and which is also described in terms of uninstantiated abstract arguments in \cite{AV11AMAI}). Recall that an expert's argument $A$, that a given violin is a Stradivarius ($s$) and therefore expensive ($e$), is asymmetrically attacked by a child's argument $B$ that it is not a Stradivarius ($\neg s$). The greater reliability of the expert's assertion about the violin means that $A$ is preferred to $B$ so that $B$ does not defeat $A$. We observed that inconsistency is not violated under this paper's attack definition of conflict free, since $\{A,B\}$ is not conflict free and so not admissible. However, even under the defeat definition, we can see that \cite{AmgoudRepairPAFs}'s suggested problem arises only when failing to take into account \emph{all} arguments. Formalising the example in \ASPIC, $A$ = $[s ; s \Imp e ]$ where $s$ is an ordinary premise, and $A'$ = $[s]$ is a sub-argument of $A$. $B = [\neg s]$ where $\neg s$ is an ordinary premise. Hence $B$ attacks $A$ on $A'$, and the expert's greater reliability means that $A'$ is preferred to $B$ and so $B$ does not defeat $A$. However, one must also then acknowledge that $A'$ rebut attacks and defeats $B$, so that $\{A,B\}$ is not admissible.

In \cite{a+v10, AV11AMAI}, Amgoud \& Vesic give a formal classical logic instantiation of a \emph{PAF} that demonstrates inconsistency. The example used is that formalised here in Example \ref{ExampleClassicalLogic}. However, Amgoud \& Vesic state that $A_1$ is strictly preferred to the other arguments, all of which are equally preferred. They thus obtain the defeat graph shown in Figure \ref{PAFCritique}-a), and so the single stable extension $\{A_1,A_2,A_3,A_5\}$, which violates consistency. The difference in outcome arises because \cite{a+v10}'s use of the premise ordering to resolve attacks (which is taken from \cite{AC02a}), differs from our Definition~\ref{DefDefeats} in which if $A$ undermines $B$ on premise $p$, then $A$ defeats $B$ if $A \nprec p$. However, in \cite{a+v10}, $A$ defeats $B$ if $A \nprec B$ based on a comparison of \emph{all} premises of $B$.  This makes a crucial difference. Since both $A_4$ and $A_5$ have $\neg y$ as weakest premise, $A_4$ and $A_5$ are equally  preferred in \cite{a+v10}. However, we have that $A_5$'s attack on $A_4$ is on $A_4$'s subargument $A_1$, so the comparison is between $A_5$ and $A_1$. Now since $x >' \neg y$, we have that $A_1$ is strictly preferred to $A_5$, so $A_5$ does not defeat $A_4$, so $A_4$ strictly defeats $A_5$. But then a set including $\{A_1,A_2,A_3,A_5\}$ is not a stable extension, since it does not defend $A_5$ against $A_4$. Instead, $E_1'$ containing the arguments $A_1, A_4$,  and $A_2$ is stable and satisfies consistency.
We prefer our approach over \cite{AmCay,a+v10}, since we do not see why the preference of the premise $\neg y$ of $A_4$, which is irrelevant to the conflict on the premise $x$, should be relevant in resolving this conflict. Note here that the crucial point is that the structure of arguments and the nature of attack should be taken into account when applying preferences. In this case it is crucial to see that $A_5$'s attack on $A_4$ was a direct attack on $A_4$'s sub-argument $A_1$.

The issue also arises in different ways. Consider the \ASPIC\ example, with $\K_p = \{p,q\}$, $\K_n =
\emptyset$, $\R_s = \emptyset$, $\R_d = \{p \Imp r; q \Imp \neg r; \neg r \Imp
s\}$. We then have the arguments and attacks in Figure \ref{PAFCritique}-d). Then, assuming $p \Imp r > q \Imp \neg r$ and
$\neg r \Imp s > p \Imp r$ , the argument
ordering $B_2 \prec A_2$, $A_2 \prec B_3$ is generated by the last link principle. A \emph{PAF} modelling then generates the defeat graph
in Figure \ref{PAFCritique}-e), so obtaining the single extension (in whatever semantics) $\{A_1,B_1,A_2,B_3\}$. So not
only $A_2$ but also $B_3$ is justified. However, not only are $A_2$ and $B_3$ based on arguments with contradictory conclusions, but the sub-argument closure postulate is violated; $B_3$ is justified, but its sub-argument $B_2$ is not. The problem arises because the \emph{PAF} modelling cannot recognise that $A_2$ attacks $B_3$ on its sub-argument $B_2$, so we should compare $A_2$ with $B_2$, and not $B_3$. Now since $B_2
\prec A_2$, then $A_2$ defeats $B_3$, so the single extension (in
whatever semantics) is $\{A_1,B_1,A_2\}$ and we have that $A_2$ is justified and both $B_2$ and $B_3$ are overruled, as visualised in Figure \ref{PAFCritique}-e). Note that these problems are not due to the use of defeasible rules or the last-link ordering.  Consider a classical logic instantiation of \ASPIC\ in which $\K_n = \emptyset$, $\K_p = \{p,q,\neg p\}$ and $q >' \neg p >' p$. The following arguments can be constructed:\\[2pt]
$A_1 = p$ ;  $A_2 = q$, $A_3 = p, q \rightarrow p \wedge q$ and $B$ = $\neg p$.\\[2pt]
Then, $A_1$ and $B$ attack each other and $B$ attacks $A_3$ (on $p$). Suppose arguments
are compared based on the weakest link principle, applying Section \ref{SubSection-ASPIC+WeakestLast}'s democratic
 principle to the premise sets. Then $A_1 \prec B$ and $B \prec A_3$. The \emph{PAF} for this example then generates
a stable extension containing $A_3$ and $B$, which again violates sub-argument closure. In
\ASPIC\ we instead obtain that $B$ defeats $A_3$ on $A_1$, so the correct outcome is obtained.

Concluding, \cite{a+v10, AmgoudRepairPAFs, AV11AMAI} are right that \emph{PAF}s need to be repaired, but the proper repair is not to change definitions at the abstract level but to make the structure of arguments and the nature of attack explicit. We have seen that seeming problems with unsuccessful asymmetric attack at the abstract level disappear if the structure of arguments and the nature of attack are specified, and that seeming violations of postulates do not occur if the success of an attack on an argument $X$ is based on a preference-based comparison on the sub-argument of $X$ that is attacked. We have also seen in this paper that there are reasonable notions of attack that result in defeat irrespective of preferences, such as \ASPIC's undercutting and contrary attacks. A framework that does not make the structure of arguments explicit cannot distinguish between preference dependent and independent attacks.

Finally, note that besides reversing asymmetric attacks, Amgoud \& Vesic \cite{a+v10, AmgoudRepairPAFs, a+v10sum, AV11AMAI}
also propose a solution to the problematic cases they identify, and that we have countered above, by using the preference ordering over arguments to define an ordering over sets of arguments, privileging those that are conflict free under the \emph{attack} relation. However this precludes the dialectical use of preferences in deciding the success of attacks between \emph{individual} arguments (as described in Section \ref{Section-ArgumentationAndLogic}); it is not clear how their use of preferences can be accounted for in dialogues and proof theoretic argument games. Furthermore, they do not show satisfaction of \cite{c+a07}'s postulates, except in the case of stable extensions, where they show that consistency is satisfied, via a correspondence with Brewka's preferred subtheories \cite{BrewkaPS}. However, we have shown this correspondence without reversing asymmetric attacks, or applying preferences over sets of arguments.

\section{Conclusions}

A newcomer to the area of abstract argumentation theory might legitimately question its added value above and beyond the conceptual insights yielded by its uniform characterisation of the inference relations of non-monotonic formalisms. This paper began with a response to this rhetorical question. Argumentative characterisations of inference encapsulate the dynamic and dialectical processes of reasoning familiar in everyday debate and discourse\footnote{Indeed, recent empirically validated work in cognitive science and psychology claims that the cognitive
capacity for human reasoning evolved primarily in order to assess and counter the claims and arguments of interlocutors in social settings \cite{SperberMercier}.}. It thus serves to both bridge formal logic and human reasoning in order that the one can inform the other, and support communicative interactions in which heterogeneous agents jointly reason and infer in the presence of uncertainty and conflict. We then discussed the declarative and procedural roles that attacks, preferences and defeats should play in the context of this value proposition, and then reviewed and modified \cite{hp10aspicJAC}'s \ASPIC\ framework in light of this discussion. Specifically, the attack relation's denotation of the mutual incompatibility of information in arguments determines whether a given set of arguments is conflict free, as distinct from their possibly preference dependent dialectical use as defeats.

\ASPIC\ provides an account of argumentation that combines Dung's argumentation theory
with structured arguments, attacks and the use of preferences. The added structure accommodates a range of concrete instantiating logics, to the extent that one can meaningfully study  satisfaction of rationality postulates. While the account retains the dialectical apparatus of Dung's theory, one must additionally show that \ASPIC's intermediate level of abstraction allows for a broad range of instantiations, if one is to continue to appeal to the above stated value proposition of argumentation. To this end, we have argued that any general account should accommodate both the traditional use of defeasible inference rules as well as deductive approaches that essentially model non-monotonicity as inconsistency handling. \cite{hp10aspicJAC}'s version of \ASPIC\ reconstructed approaches that use defeasible inference rules (e.g., \cite{Pol95,PS97}) and showed that assumption-based argumentation \cite{BDKT97} and systems using argument schemes can be formalised in \ASPIC. The modelling of defeasible rules inevitably introduced a degree of complexity that exceeds that of other proposals for general frameworks. However we have argued that a truly general framework for structured argumentation must include defeasible rules. In this paper we adapted \ASPIC\ to additionally accommodate deductive approaches that require arguments to have consistent premises, and then showed that the adapted \ASPIC, with the revised definition of conflict free, satisfies key properties of Dung frameworks and \cite{c+a07}'s rationality postulates under some assumptions. We then formalised instantiation of the adapted \ASPIC\ with Tarskian (in particular classical) logics extended with preferences, thus demonstrating satisfaction of rationality postulates by these instantiations, and paving the way for the study of other non-classical Tarskian approaches to argumentation. We also addressed some limitations of the way in which argument orderings are defined in \cite{hp10aspicJAC}, and considered a broader range of instantiations of these preference orderings, showing that they satisfy assumptions required for proof of the aforementioned properties and postulates.

Finally, a key rhetorical claim of this paper is that a proper modelling of the use of preferences requires that we take into account the structure of arguments. We believe this claim to be supported by the results in this paper and our discussion of recent critiques of Dung and preference-based argumentation frameworks.

We conclude by mentioning future research. Firstly, we emphasise that \ASPIC\ is not a system but a framework for specifying systems, such that these systems can be analysed on their properties, for instance, on whether they satisfy the four rationality postulates. An immediate task is to thus show how a range of systems, other than those considered here and in \cite{hp10aspicJAC,gpw07}, can be specified in \ASPIC. Secondly, we are currently developing a structured \ASPIC\ approach to extended argumentation \cite{ModgilAIJ}, building on a preliminary such structuring in \cite{ModPrakSEAFComma}.
Thirdly, \cite{CCD2012} recently proposed the additional so called
`non-interference' and `crash
resistance' rationality postulates, which are about
whether self-defeating arguments can interfere with
the justification status of other arguments in undesired ways. We plan
to study the conditions under which these postulates are satisfied by
the \ASPIC\ framework. Fourthly, we have in this paper focussed on weakest and last link definitions of preference orderings over arguments. We aim in future work to consider other ways of ranking augments, and to study whether such preference orderings satisfy the assumptions identified in this paper for ensuring satisfaction of properties and postulates. Finally, since many conceptual choices made in formalising \ASPIC\ appeal to the use of argumentation in practice, further real-world applications of \ASPIC\ are required to establish the framework's utility. One such existing application concerns the use of \ASPIC\ in modelling the reasoning in a well known legal case \cite{PrakPop}. Furthermore, connections between \ASPIC\ and more informal `human' modes of argumentative practice need to be established. \cite{bexASPIC-AIF} represents an important first step in this direction, in which \ASPIC\ is used to provide formal logical foundations for the Argument Interchange Format \cite{AIF}; an emerging standard for representing argumentation knowledge in both computational and human centered argumentation applications.\\[3pt]
\noindent \textbf{Acknowledgements}: We would like to thank the anonymous reviewers, whose comments on earlier versions of this paper have helped to improve the content and presentation of this paper.

\section{Appendix}

\scalefont{0.95}

\subsection{Proofs for Section \ref{Section-ASPIC+-Propoerties-Attack}}

\noindent \textbf{Proposition \ref{Prop$B'$-extension}} Let $A$ and $B$ be arguments where $B$ is plausible or defeasible and $A$ and $B$ have contradictory conclusions, and assume $\mathtt{Prem}(A)\cup \mathtt{Prem}(B)$ is c-consistent if $A$ and $B$ are defined as in Definition \ref{DefinitionConsistentArguments}. Then:\\[-17pt]
\begin{enumerate}
\item For all $B' \in M(B)$, there exists a strict continuation $A^+_{B'}$ of $(M(B)\backslash \{B'\}) \cup M(A)$ such that $A^+_{B'}$ rebuts or undermines $B$ on $B'$.\\[-17pt]

\item If $B \prec A$, and $\preceq$ is reasonable, then for some $B' \in M(B)$, $A^+_{B'}$ defeats $B$.
\end{enumerate}

\begin{proof}
1) Consider first systems closed under contraposition (Def. \ref{DefStrictDefAssumptions}). Observe first that $\mathtt{Conc}(M(B)) \cup \mathtt{Prem_n}(B) \vdash \mathtt{Conc}(B)$ (i.e., one can construct a strict argument concluding $\mathtt{Conc}(B)$ with all premises taken from $\mathtt{Conc}(M(B))$ and the axiom premises in $B$). By contraposition, and since $\mathtt{Conc}(A)$ and $\mathtt{Conc}(B)$ contradict each other, we have that for any $B_i \in M(B)$: $\mathtt{Conc}(M(B) \setminus \{B_i\})  \cup \mathtt{Prem_n}(B) \cup \mathtt{Conc}(A) \vdash - \mathtt{Conc}(B_i)$. Hence, one can construct a strict continuation $A^+_{B_i}$ that continues $\{A\} \cup M(B) \setminus \{B_i\} \cup \mathtt{Prem_n}(B)$ with strict rules, and that concludes $- \mathtt{Conc}(B_i)$.
    \\
    By construction, $M(B) \setminus \{B_i\}$ and $M(A)$ are the maximal fallible sub-arguments of $A^+_{B_i}$, and $\mathtt{Prem}(A^+_{B_i}) \subseteq \mathtt{Prem}(A) \cup \mathtt{Prem}(B)$.
    \\
    Since by construction of $M(B)$ either $B_i$ is an ordinary premise or ends with a defeasible inference, $A^+_{B_i}$ either undermines or rebuts $B_i$. But then $A^+_{B_i}$ also undermines or rebuts $B$.\\[2pt]
For systems closed under transposition the existence of arguments $A^+_{B_i}$ and $B_i$, for all $B_i \in M(B)$, is proven by  straightforward generalisation of Lemma~6 in \cite{c+a07}. Then the proof can be completed as above.
\\
In the case that $A$ and $B$ are defined as in Def. \ref{DefinitionConsistentArguments}, one only need additionally show that $\mathtt{Prem}(A^+_{B_i})$ is c-consistent, which follows given $\mathtt{Prem}(A^+_{B_i}) \subseteq \mathtt{Prem}(A) \cup \mathtt{Prem}(B)$, and $\mathtt{Prem}(A) \cup \mathtt{Prem}(B)$ is c-consistent by assumption.
\\[4pt]
2) By construction, each $B'$continuation $A^+_{B'}$ of $A$ is a strict continuation of $\{A\} \cup M(B) \setminus \{B'\} \cup \mathtt{Prem_n}(B)$. Hence, letting $M(B)$ = $\bigcup_{i=1}^n B_i$, we have $\{B_1, \ldots, B_n, A\}$ where each $A^+_{B_i}$ is a strict continuation of $\{B_1,\ldots,B_{i-1},B_{i+1},B_n,A\}$. Also, $B$ is a strict continuation of $\{B_1,\ldots,B_n\}$. Since $\preceq$ is reasonable, then by Definition \ref{Def-Reasonable}-(2), it cannot be that: $B \prec A$ \emph{and} $A^+_{B_1} \prec B_1$ \emph{and} $\ldots$ \emph{and}
$A^+_{B_n} \prec B_n$. Since by assumption $B \prec A$, then for some $i$, $A^+_{B_i}$ rebuts or undermines $B$ on $B_i$, $A^+_{B_i} \nprec B_i$, and so $A^+_{B_i}$ defeats $B$.

\end{proof}

%We now state some useful results that will be used for proving Propositions \ref{PropositionConflictFree} and \ref{PropositionFundamentalLemma}.
In what follows, recall Notation \ref{NotationAttacksDefeats}, in which $X \rightharpoonup Y$ denotes $X$ \emph{attacks} $Y$ and $X \hookrightarrow Y$ denotes $X$ \emph{defeats} $Y$.

\begin{lemma}\label{Basic properties}Let ($\A$,
$\C$, $\preceq$) be a \emph{(c-)SAF}:\\[-17pt]
\begin{enumerate}
  \item\label{Basic1} If $A$ is acceptable w.r.t. $S \subseteq \A$ then $A$ is acceptable w.r.t. any superset of $S$.\\[-17pt]
  %\item\label{Basic2} If $A \rightleftharpoons B$ then $A \rightarrow B$ or $B \rightarrow A$.\\[-17pt]
  \item\label{Basic3} If $A \hookrightarrow B$, then $A \hookrightarrow B'$ for some $B' \in \mathtt{Sub}(B)$, and if
      $A \hookrightarrow B'$, $B' \in \mathtt{Sub}(B)$, then $A \hookrightarrow B$.\\[-17pt]
  \item\label{Basic4} If $A$ is acceptable w.r.t. $S \subseteq \A$, $A' \in \mathtt{Sub}(A)$, then $A'$ is acceptable w.r.t. $S$.
\end{enumerate}

\begin{proof} Proofs of \ref{Basic properties}-\ref{Basic1} and \ref{Basic properties}-\ref{Basic3} are straightforward given the definitions of acceptability and defeat. For \ref{Basic properties}-\ref{Basic4}, suppose $B \hookrightarrow A'$. By \ref{Basic properties}-\ref{Basic3}, $B \hookrightarrow A$, and so $\exists C \in S$ s.t. $C \hookrightarrow A$. Hence $A'$ is acceptable w.r.t. $S$.
\end{proof}

\end{lemma}

\begin{lemma}\label{Lemma1} Suppose $B \rightharpoonup A$, where $B$ attacks $A$ on $A'$, and if
$A$ and $B$ are defined as in Def. \ref{DefinitionConsistentArguments}, then $\mathtt{Prem}(A)\cup \mathtt{Prem}(B)$ is c-consistent. If $B \not\hookrightarrow A$ then either:\\[-17pt]

\begin{enumerate}
  \item  $A' \hookrightarrow B$, or;\\[-17pt]
  \item For some $B' \in M(B)$, there is a strict continuation $A'^+_{B'}$ of $(M(B)\backslash \{B'\}) \cup M(A')$ s.t. $A'^+_{B'} \hookrightarrow B$.
\end{enumerate}

\begin{proof}  Since $B \not\hookrightarrow A$, then: $B$ rebuts on the conclusion $\varphi$ of $A'$ where $A'$'s top rule is defeasible, or $B$ undermines the ordinary premise $A'$ = $\varphi$, and $B \prec A'$. Also, $\mathtt{Conc}(B)$ must be a contradictory of $\varphi$ since otherwise $\mathtt{Conc}(B)$ would be a contrary of $\varphi$ implying that $B \hookrightarrow A$ (by virtue of the preference independent attack by contraries).\\ Also, $B$ must be plausible or defeasible since for $B \prec A'$ to be the case, $B$ cannot be strict and firm (under the assumption that $\preceq$ is reasonable (Def. \ref{Def-Reasonable})).
\\
1) If $B$ is an ordinary premise or has a defeasible top rule, $A' \rightharpoonup B$, and since $B \prec A'$, $A' \hookrightarrow B$.
\\
2) If $B$ has a strict top rule, then by Proposition \ref{Prop$B'$-extension} there exists a strict continuation $A'^+_{B'}$ s.t. $A'^+_{B'} \hookrightarrow B$.
\end{proof}

\end{lemma}

The following lemma follows from the fact that if $B$ defeats some strict continuation $A$ of $\{A_1,\ldots,A_n\}$ then the defeat must be on some $A_i$.

\begin{lemma}\label{Lemma3} Let ($\A$,
$\C$, $\preceq$) be a \emph{(c-)SAF}. Let $A \in \A$ be a strict continuation of $\{A_1,\ldots,A_n\}$ $\subseteq \A$, and for $i = 1 \ldots n$, $A_i$ is acceptable w.r.t. $E \subseteq \A$. Then $A$ is acceptable w.r.t. $E$.

\begin{proof}
Let $B$ be any argument s.t. $B \hookrightarrow A$. By Def. \ref{DefAttacks}, $B$ attacks $A$ by undercutting or rebutting on defeasible rules in $A$ or undermining on an ordinary premise in $A$. Hence, by definition of strict continuations (Def. \ref{Def-StrictExtension}), it must be that
$B \rightharpoonup A$ iff $B \rightharpoonup A_i$ for some (possibly more than one) $A_i \in \{A_1,\ldots,A_n\}$. Either:\\[3pt]
1) $B$ undercuts or contrary rebuts/undermines some $A_i$, and so by Def. \ref{DefDefeats}, $B$ defeats $A_i$, or:\\
2) $B$ does not undercut or contrary rebut/undermine some $A_i$. Suppose for all $A_i$, for all sub-arguments $A_i'$ of $A_i$ s.t. $B$ rebuts or undermines $A_i$ on $A_i'$, $B \prec A_i'$. This contradicts $B$ defeats $A$. Hence, for some $A_i$, $B$ defeats $A_i$.\\[3pt]
We have shown that if $B$ defeats $A$ then $B $ defeats some $A_i$. By assumption of $A_i$ acceptable w.r.t. $E$, $\exists C \in E$ s.t. $C$ defeats $B$. Hence, $A$ is acceptable w.r.t. $E$.
\end{proof}
\end{lemma}

For the following proposition, recall that by assumption, any c-\emph{SAF} is well defined and so satisfies c-classicality (Def. \ref{DefStrictDefAssumptions}).
\\[5pt]
\noindent \textbf{Proposition \ref{Lemma4}} Let ($\A$,
$\C$, $\preceq$) be a c-\emph{SAF}. If $A_1,\ldots,A_n$ are acceptable w.r.t. some conflict-free $E \subseteq \A$, then $\bigcup_{i=1}^n\mathtt{Prem}(A_i)$ is c-consistent.
%Then for any   $\varphi_1, \ldots, \varphi_n \to \varphi \in \R_s$, if for $i = 1 \ldots n$, $A_i \in E$ s.t. $\mathtt{Conc}(A_i)$ = $\varphi_i$, then  $\mathtt{Prem}(A_1) \cup \ldots \cup  \mathtt{Prem}(A_n)$ is c-consistent.

\begin{proof}
Suppose for contradiction otherwise, and let $S$ be any minimally c-inconsistent subset of $\bigcup_{i=1}^n\mathtt{Prem}(A_i)$. By assumption of c-classicality:\\[-17pt]
 \begin{quote}
    for all $\varphi \in S$, $S \setminus \{\varphi\} \vdash - \varphi$ and $S \setminus \{\varphi\}$ is c-consistent.\\[-17pt]
 \end{quote}

 \noindent We thus have the set of ordinary premises $S$ = $\{\varphi_1,\ldots,\varphi_m\} \subseteq \bigcup_{i=1}^n\mathtt{Prem}(A_i)$ (that must be non-empty given that $\K_n$ is c-consistent by assumption of axiom consistency (Def. \ref{DefStrictDefAssumptions})), such that for $i=1\ldots m$, there is a strict continuation $B^{+\backslash i}$ of $\{\varphi_1,\ldots,\varphi_{i-1},$ $\varphi_{i+1},\varphi_m\}$ s.t. $B^{+\backslash i} \rightharpoonup \varphi_i$ (recall that elements from $\K$ are also arguments, so this notation is well-defined). \\ Since $\preceq$ is reasonable, for some $i$, $B^{+\backslash i} \nprec \varphi_i$ and so $B^{+\backslash i} \hookrightarrow \varphi_i$. \\ Since for $i = 1 \ldots n$, $A_i$ is acceptable w.r.t. $E$, then: since $\varphi_i \in \bigcup_{i=1}^n\mathtt{Prem}(A_i)$, then by Lemma \ref{Basic properties}-\ref{Basic4}, $\varphi_i$ is acceptable w.r.t. $E$. \\
Since $B^{+\backslash i}$ is a strict continuation of some subset of $\bigcup_{i=1}^n\mathtt{Prem}(A_i)$, then by Lemmas \ref{Basic properties}-\ref{Basic4} and \ref{Lemma3}, $B^{+\backslash i}$ is acceptable w.r.t. $E$.\\ But then since  $B^{+\backslash i} \hookrightarrow \varphi_i$, $\exists X,Y \in E$ s.t. $Y \hookrightarrow B^{+\backslash i}$, $X \hookrightarrow Y$, contradicting $E$ is conflict free.
\end{proof}

\begin{lemma}\label{Lemma2} Let $A$ be acceptable w.r.t an admissible
extension $S$ of a \emph{(c-)SAF} ($\A$,
$\C$, $\preceq$). Then $\forall B \in S \cup \{A\}$, neither $A \hookrightarrow B$ or $B \hookrightarrow A$.

\begin{proof} Suppose for contradiction that: 1) $A \hookrightarrow B$, $B \in S \cup \{A\}$. By assumption of $B$'s acceptability, $\exists C \in S$ s.t. $C \hookrightarrow A$, and by acceptability of $A$, $\exists D \in S$ s.t. $D \hookrightarrow C$, hence $D \rightharpoonup C$, contradicting $S$ is conflict free; 2) $B \hookrightarrow A$, $B \in S$. By acceptability of $A$, $\exists D \in S$ s.t. $D \hookrightarrow B$, hence $D \rightharpoonup B$, contradicting $S$ is conflict free.
\end{proof}

\end{lemma}

\noindent \textbf{Proposition \ref{PropositionConflictFree}} Let $A$ be acceptable w.r.t an admissible
extension $S$ of a \emph{(c-)SAF} ($\A$,
$\C$, $\preceq$). Then
$S'$ = $S$ $\cup$ $\{A\}$  is conflict free.

\begin{proof}
Firstly, since for any $B \in S$, $B$ is acceptable w.r.t. $S$, then by Proposition \ref{Lemma4}, $\mathtt{Prem}(A) \cup \mathtt{Prem}(B)$ is c-consistent.\\[2pt]
Suppose for contradiction that $S'$ is not conflict free. By assumption, $S$ is conflict free. $A$ cannot attack itself since $A$ must then defeat itself,  contradicting Lemma \ref{Lemma2}. Hence, we have the following two cases:\\[2pt]
\noindent \textbf{1)} $\exists B \in S$, $B \rightharpoonup A$, and $B \not\hookrightarrow A$ by Lemma \ref{Lemma2}. By Lemma \ref{Lemma1}, for some sub-argument $A'$ of $A$, either:
\\[2pt]
\textbf{1.1)} $A'$ defeats $B$, hence (by acceptability of $B$) $\exists C \in S$ s.t. $C \hookrightarrow A'$, and so (by Lemma \ref{Basic properties}-\ref{Basic3}) $C \hookrightarrow A$, contradicting Lemma \ref{Lemma2}, or;
\textbf{1.2)} $\exists$ $A'^+_{B'}$ s.t. $A'^+_{B'} \hookrightarrow B$, hence $\exists C \in S$ s.t. $C \hookrightarrow A'^+_{B'}$. By construction of $A'^+_{B'}$ and Lemma \ref{Basic properties}-\ref{Basic3}, it must be that $C \hookrightarrow Z$, $Z \in \mathtt{Sub}(A) \cup  \mathtt{Sub}(B)$. Hence, (by Lemma \ref{Basic properties}-\ref{Basic3}) either $C \hookrightarrow B$, contradicting $S$ is conflict free, or $C \hookrightarrow A$, contradicting Lemma \ref{Lemma2}.\\[2pt]
\noindent \textbf{2)} $\exists B \in S$, $A \rightharpoonup B$, and $A \not\hookrightarrow B$ by Lemma \ref{Lemma2}.
%Hence neither $A \hookrightarrow B$ or $B \hookrightarrow A$.
By Lemma \ref{Lemma1}, for some sub-argument $B'$ of $B$, either:
\\[2pt]
\textbf{2.1)} $B'$ defeats $A$, hence (by acceptability of $A$) $\exists C \in S$ s.t. $C \hookrightarrow B'$ and so (by Lemma \ref{Basic properties}-\ref{Basic3}) $C \hookrightarrow B$, hence $C \rightharpoonup B$, contradicting $S$ is conflict free, or;
\textbf{2.2)}  $\exists B'^+_{A'}$ s.t. $B'^+_{A'} \hookrightarrow A$, hence $\exists C \in S$ s.t. $C \hookrightarrow B'^+_{A'}$. By construction of $B'^+_{A'}$, $C \hookrightarrow Z$, $Z \in \mathtt{Sub}(A) \cup  \mathtt{Sub}(B)$, leading to a contradiction as in \textbf{1.2)}.\end{proof}

\noindent \textbf{Proposition} \ref{PropositionFundamentalLemma} Let $A, A'$ be acceptable w.r.t an admissible
extension $S$ of a \emph{(c-)SAF} ($\A$,
$\C$, $\preceq$). Then:\\[-17pt]
\begin{enumerate}
\item $S'$ = $S$ $\cup$ $\{A\}$ is admissible \\[-17pt]\item $A'$ is
acceptable w.r.t. $S'$.\\[-17pt]
\end{enumerate}
\begin{proof} 1) By Lemma \ref{Basic properties}-\ref{Basic1}, all arguments in $S'$ are acceptable w.r.t. $S'$. By Proposition \ref{PropositionConflictFree}, $S'$ is conflict free. Hence $S'$ is admissible. 2) By Lemma \ref{Basic properties}-\ref{Basic1}, $A'$ is
acceptable w.r.t. $S'$.\end{proof}

\subsection{Proofs for Section \ref{Section-ASPIC+R-RationalityPostulates}}

\noindent \textbf{Theorem \ref{TheoremSub-argumentClosure}} [\emph{\textbf{Sub-argument Closure}}] Let $\Delta$ = $(\A,\C,\preceq)$ be a \emph{(c-)SAF} and $E$ a complete extension of $\Delta$. Then for all $A \in E$: if $A' \in \mathtt{Sub}(A)$ then $A' \in E$.\\[-17pt]
\begin{proof} $A'$ is acceptable w.r.t. $E$ by Lemma \ref{Basic properties}-\ref{Basic4}. $E \cup \{A'\}$ is conflict free by Prop.\ref{PropositionConflictFree}. Hence, since $E$ is complete, $A' \in E$.\end{proof}

\noindent \textbf{Theorem \ref{TheoremStrictRulesClosure}} [\emph{\textbf{Closure under Strict Rules}}] Let $\Delta$ = $(\A,\C,\preceq)$ be a \emph{(c-)SAF} and $E$ a complete extension of $\Delta$. Then $\{ \mathtt{Conc}(A) | A \in E\}$ = $Cl_{R_s}(\{ \mathtt{Conc}(A) | A \in E\})$.\\[-17pt]

\begin{proof} It suffices to show that any strict continuation $X$ of $\{ A | A \in E\}$ is in $E$. By Lemma \ref{Lemma3}, any such $X$ is acceptable w.r.t. $E$. By Proposition \ref{PropositionConflictFree}, $E \cup \{X\}$ is conflict free. Hence, since $E$ is complete, $X \in E$. Note that if $\Delta$ is a c-\emph{SAF}, Proposition \ref{Lemma4} guarantees that $X$'s premises are c-consistent.\end{proof}

%note that sub-argument closure postulate is not properly proved in Arg and Comp paper.

\noindent \textbf{Theorem \ref{TheoremDirectConsistency}} [\emph{\textbf{Direct Consistency}}] Let $\Delta$ = $(\A,\C,\preceq)$ be a \emph{(c-)SAF} and $E$ an admissible extension of $\Delta$. Then $\{ \mathtt{Conc}(A) | A \in E\}$ is consistent.\\[-17pt]

\begin{proof} We show that if $A,B \in E$,  $\mathtt{Conc}(A)$ $\in$ $\overline{\mathtt{Conc}(B)}$ (i.e., $E$ is inconsistent (Def. \ref{argumentation-system})), then this leads to a contradiction:
\\[2pt]
\textbf{1.} $A$ is firm and strict, and:
\\
\textbf{1.1} if $B$ is strict and firm, then this contradicts the assumption of \emph{axiom consistency} (Def. \ref{DefStrictDefAssumptions});
\textbf{1.2} if B is plausible or defeasible, and \textbf{1.2.1} $B$ is an ordinary premise or has a defeasible top rule, then $A \rightharpoonup B$, contradicting $E$ is conflict free, or \textbf{1.2.2} $B$ has a strict top rule (see \textbf{3} below).
\\[2pt]
\textbf{2.} $A$ is plausible or defeasible, and:
\\
\textbf{2.1} if $B$ is strict and firm then under the \emph{well-formed} assumption (Def. \ref{DefStrictDefAssumptions}) $\mathtt{Conc}(A)$ cannot be a contrary of $\mathtt{Conc}(B)$, and so they are a contradictory of each other, and \textbf{2.1.1} $A$ is an ordinary premise or has a defeasible top rule, in which case $B \rightharpoonup A$, contradicting $E$ is conflict free, or \textbf{2.1.2} $A$ has a strict top rule (see \textbf{3} below);
\textbf{2.2} if $B$ is plausible or defeasible and \textbf{2.2.1} $B$ is an ordinary premise or has a defeasible top rule then $A \rightharpoonup B$, contradicting $E$ is conflict free, or \textbf{2.2.2} $B$ has a strict top rule (see \textbf{3} below).
\\[2pt]
\textbf{3.} Each of \textbf{1.2.2}, \textbf{2.1.2} and \textbf{2.2.2} describes the case where $X,Y \in E$, $\mathtt{Conc}(X)$ $\in$ $\overline{\mathtt{Conc}(Y)}$, $Y$ is defeasible or plausible and has a strict top rule, and so by the \emph{well-formed} assumption  $\mathtt{Conc}(X)$ and $\mathtt{Conc}(Y)$ must be contradictory.\\ In the case that $\Delta$ is a c-\emph{SAF}, since $X,Y \in E$, then $X, Y$ are acceptable w.r.t. $E$, and so by Proposition \ref{Lemma4}, $\mathtt{Prem}(A) \cup \mathtt{Prem}(B)$ is c-consistent.\\
By Prop \ref{Prop$B'$-extension} there is a strict continuation $X^+_{Y'}$ of $M(Y)\backslash\{Y'\} \cup M(X)$ s.t. $X^+_{Y'} \rightharpoonup Y$. By Lemma \ref{Lemma3} $X^+_{Y'}$ is acceptable w.r.t. $E$, and by Prop. \ref{PropositionConflictFree}, $E \cup \{X^+_{Y'}\}$ is conflict free, contradicting $X^+_{Y'} \rightharpoonup Y$.\end{proof}

\noindent \textbf{Theorem \ref{TheoremIndirectConsistency}} [\emph{\textbf{Indirect Consistency}}] Let $\Delta$ = $(\A,\C,\preceq)$ be a \emph{(c-)SAF} and $E$ a complete extension of $\Delta$. Then $Cl_{\R_s}(\{\mathtt{Conc}(A) | A \in E\})$ is consistent.\\[-17pt]
\begin{proof} Follows from Theorems \ref{TheoremStrictRulesClosure} and \ref{TheoremDirectConsistency}.
\end{proof}

\subsection{Proofs for Section \ref{Section-Attack-Defeat-Comparison}}

\noindent \textbf{Proposition \ref{Prop-Equivalence-a-d-admissible}} Let $\Delta$ be a \emph{(c-)SAF}. For $T \in$ $\{$ \emph{admissible}, \emph{complete}, \emph{grounded}, \emph{preferred}, \emph{stable}$\}$, $E$ is an \emph{att}-$T$ extension of $\Delta$ iff $E$ is a \emph{def}-$T$ extension of $\Delta$.\\[-17pt]

\begin{proof} We first show that $E$ is conflict free under the attack definition iff $E$ is conflict free under the defeat definition. The left to right half is trivial: if no two arguments in $E$ attack each other, then no two arguments in $E$ defeat each other. For the right to left half, suppose $B, A \in E$, $B \rightharpoonup A$, $B \not\hookrightarrow A$. First note that since $A, B$ are acceptable w.r.t. $E$, then in the case of a \emph{c-SAF} where $A$ and $B$ are defined as in Def. \ref{DefinitionConsistentArguments}, $\mathtt{Prem}(A)\cup \mathtt{Prem}(B)$ is c-consistent by Proposition \ref{Lemma4}. Then, by Lemma \ref{Lemma1}, $\exists A' \in \mathtt{Sub}(A)$ s.t. either: i) $A' \hookrightarrow B$, or ii) there is a strict continuation $A'^+_{B'}$ of $(M(B)\backslash \{B'\}) \cup M(A')$ s.t. $A'^+_{B'} \hookrightarrow B$. In case i), (by acceptability of $B$) $\exists C \in E$ s.t. $C \hookrightarrow A'$, and so (by Lemma \ref{Basic properties}-\ref{Basic3}) $C \hookrightarrow A$, contradicting $E$ is defeat conflict free. In case ii), (by acceptability of $B$), $\exists C \in E$ s.t. $C \hookrightarrow A'^+_{B'}$. By construction of $A'^+_{B'}$ and Lemma \ref{Basic properties}-\ref{Basic3}, $C \hookrightarrow Z$, $Z \in \mathtt{Sub}(A) \cup  \mathtt{Sub}(B)$. Hence, (by Lemma \ref{Basic properties}-\ref{Basic3}) either $C \hookrightarrow B$ or $C \hookrightarrow A$, contradicting $E$ is defeat conflict free.

Next, note that admissible and complete extensions are in Definition~\ref{Dung semantics} defined in terms of conflict-freeness and acceptability, where acceptability is according to Definition~\ref{DefinitionSAFExtensions} defined in terms of defeat relations between arguments. Then since any \emph{att} semantics and \emph{def} semantics agree on the defeat relation between arguments, the proposition follows for admissible and complete semantics. Then since preferred and grounded semantics are defined in terms of complete semantics, it also follows for these semantics, and then since stable semantics is defined in terms of preferred semantics and the defeat relation, it also follows for stable semantics.
%
%Then $X$ undermine attacks $Y$ on some non-axiom premise $Y'$ in $Y$, or rebut attacks on the conclusion of some sub-argument $Y'$ in $Y$ where $Y'$'s top rule is defeasible, and $X \prec Y'$ (\textbf{should we enumerate the other cases where the attack is not valid or must succeed as a defeat ? }).

\end{proof}

\subsection{Proofs for Section \ref{SubSection-ASPIC+WeakestLast}}

In the following proofs, we may write $\mathtt{LDR}$ as an abbreviation for $\mathtt{LastDefRules}$, and $\mathtt{DR}$ as an abbreviation for $\mathtt{DefRules}$. Also, as an abuse of notation we may simply write $\triangleleft$ instead of $\triangleleft_{\mathtt{s}}$. \\[-2pt]

\noindent \textbf{Proposition \ref{PropLastReasonable}} Let $\preceq$ be defined according to the last-link principle, based on a reasonable inducing $\triangleleft_{\mathtt{s}}$. Then $\preceq$ is \emph{reasonable}.\\[-18pt]

\begin{proof} \emph{Proof of the first condition of reasonableness}:
\\[2pt]
\noindent 1-i) Assume $A$ is strict and firm, and so $\mathtt{LDR}(A)$ = $\emptyset$ and $\mathtt{Prem_{p}}(A) = \emptyset$.\\ - If $\mathtt{LDR}(B) \neq \emptyset$, then $A$ and $B$ must be compared by the first condition of Def. \ref{Def Last Link}. By Def.\ref{Def-hd12}, $\mathtt{LDR}(B)$ $\triangleleft_{\mathtt{s}}$ $\mathtt{LDR}(A)$, and so $B \prec A$.
\\ - If $\mathtt{LDR}(B) = \emptyset$, then  $A$ and $B$ must be compared by the second condition of Def. \ref{Def Last Link}. By assumption of $B$ being plausible or defeasible, $\mathtt{Prem_{p}}(B) \neq \emptyset$. By Def.\ref{Def-hd12}, $\mathtt{Prem_{p}}(B)$ $\triangleleft_{\mathtt{s}}$ $\mathtt{Prem_{p}}(A)$, and so $B \prec A$.
\\[3pt]
\noindent 1-ii) Assume $B$ is strict and firm, and so $\mathtt{LDR}(B)$ = $\emptyset$, $\mathtt{Prem_{p}}(B) = \emptyset$. Then by Def.\ref{Def-hd12}, $\mathtt{LDR}(B)$ $\ntriangleleft_{\mathtt{s}}$ $\mathtt{LDR}(A)$ and $\mathtt{Prem_{p}}(B)$ $\ntriangleleft_{\mathtt{s}}$ $\mathtt{Prem_{p}}(A)$, and so $B \nprec A$ by the first or second condition of Def. \ref{Def Last Link}.
\\[3pt]
\noindent 1-iii) Follows straightforwardly from Def. \ref{Def Last Link}, given that $A'$ differs from $A$ only in its strict rules and/or axiom premises.
\\[5pt]
\emph{Proof of the second condition of reasonableness}:
\\[3pt]
Assume for contradiction that:\\[-17pt]

\begin{quote} $\forall i$, there is a strict continuation $C^{+\backslash i}$ of $\{C_1,\ldots,C_{i-1},$ $C_{i+1},\ldots,C_n\}$
such that $C^{+\backslash i} \prec C_i$ %\hspace{85mm}\textbf{i)}\\[-17pt]
\end{quote}
\noindent At least one argument in $\{C_1,\ldots,C_n\}$ must be defeasible or plausible, else every $C^{+\backslash i}$ would be strict and firm, contradicting 1-ii) above.
\\[2pt]
\noindent 1) Suppose for some $i$ = $1\ldots n$, $\mathtt{LDR}(C_i) \neq \emptyset$. W.l.o.g. we can assume $i = 1$.
Then, $C^{+\backslash 1} \prec C_1$ by virtue of condition 1 of Def. \ref{Def Last Link}.
That is to say:  ($\mathtt{LDR}(C^{+\backslash 1})$ =  $\bigcup_{j =2}^n\mathtt{LDR}(C_j)$) $\lhd_{\mathtt{s}}$ $\mathtt{LDR}(C_1)$.\\[2pt]
By Def.\ref{Def-hd12}-(1), it must be that $\bigcup_{j =1}^n\mathtt{LDR}(C_j) \neq \emptyset$. W.l.o.g. assume $\{C_2,\ldots,C_{m\leq n}\}$ are the arguments in $\{C_2,\ldots,C_n\}$ such that for  $k = 2 \ldots m$, $\mathtt{LDR}(C_k) \neq \emptyset$.
We have that:\\[-17pt]

\begin{quote} \hspace{32mm}$\bigcup_{k = 2}^m\mathtt{LDR}(C_k)$ $\lhd_{\mathtt{s}}$ $\mathtt{LDR}(C_1)$.%\hspace{38mm}\textbf{ii)}\\[-17pt]
\end{quote}

\noindent Since $\triangleleft_{\mathtt{s}}$ is reasonable inducing (Def.\ref{Def-s-reasonable inducing}), for some $k = 2 \dots m$, $\mathtt{LDR}(C_k)$ $\triangleleft_{\mathtt{s}}$ $\mathtt{LDR}(C_1)$. We can w.l.o.g. assume:

		      \begin{quote}
                         $\mathtt{LDR}(C_2)$ $\triangleleft_{\mathtt{s}}$ $\mathtt{LDR}(C_1)$.
                       \end{quote}

\noindent
By assumption, $C^{+\backslash 2} \prec C_2$. Since $\mathtt{LDR}(C_2) \neq \emptyset$, one can reason as above to conclude that $\bigcup_{k = 1, k \neq 2}^m\mathtt{LDR}(C_k)$ $\lhd_{\mathtt{s}}$ $\mathtt{LDR}(C_2)$. Suppose $m = 2$. Then $\mathtt{LDR}(C_1) \triangleleft_{\mathtt{s}} \mathtt{LDR}(C_2)$. But then by transitivity of $\triangleleft_{\mathtt{s}}$, $\mathtt{LDR}(C_1) \triangleleft_{\mathtt{s}} \mathtt{LDR}(C_1)$, contradicting the irreflexivity of $\triangleleft_{\mathtt{s}}$. Suppose $m > 2$ and w.l.o.g. assume

 \begin{quote}
                                                                                          $\mathtt{LDR}(C_3) \triangleleft_{\mathtt{s}} \mathtt{LDR}(C_2)$.
                                                                                        \end{quote}

\noindent Reasoning as above: $\bigcup_{k = 1, 2, k \neq 3}^m\mathtt{LDR}(C_k)$ $\lhd_{\mathtt{s}}$ $\mathtt{LDR}(C_3)$.  Suppose $m = 3$, and either
 $\mathtt{LDR}(C_2) \triangleleft_{\mathtt{s}} \mathtt{LDR}(C_3)$ or  $\mathtt{LDR}(C_1) \triangleleft_{\mathtt{s}} \mathtt{LDR}(C_3)$.
Via transitivity of $\triangleleft_{\mathtt{s}}$, either would contradict the irreflexivity of $\triangleleft_{\mathtt{s}}$. Suppose $m > 3$ and
w.l.o.g. assume

\begin{quote}
$\mathtt{LDR}(C_4) \triangleleft_{\mathtt{s}} \mathtt{LDR}(C_3)$.
\end{quote}

\noindent It is easy to see that we can continue to reason in the same way until we have that:

\begin{quote}   $\bigcup_{k = 1}^{m-1}\mathtt{LDR}(C_k)$ $\lhd_{\mathtt{s}}$ $\mathtt{LDR}(C_m)$ and for $k = 1 \ldots m-1$, $C_{k+1} \triangleleft_{\mathtt{s}} C_{k}$. \end{quote}

But then it must be that for some $k = 1 \ldots m-1$,  $\mathtt{LDR}(C_k)$ $\triangleleft_{\mathtt{s}}$ $\mathtt{LDR}(C_m)$,  contradicting (via transitivity) the irreflexivity of $\lhd_{\mathtt{s}}$.
\\[3pt]
2) Suppose for $i$ = $1\ldots n$, $\mathtt{LDR}(C_i) = \emptyset$. Then $C^{+\backslash i} \prec C_i$ by virtue of condition 2 of Def. \ref{Def Last Link}. That is to say,   $\bigcup_{j = 1,j\neq i}^n\mathtt{Prem_{p}}(C_j)$ $\lhd_{\mathtt{s}}$ $\mathtt{Prem_{p}}(C_i)$.
 %and $\forall i$, $\bigcup_{j = 1,j\neq i}^n\mathtt{Prem_{p}}(C_j)$ $\trianglelefteq$ $\mathtt{Prem_{p}}(C_i)$.
 One can then, by virtue of $\triangleleft_{\mathtt{s}}$ being reasonable inducing,
 reason to a contradiction as above.
%  (by virtue of $\trianglelefteq$ satisfying Def.\ref{Def-s-reasonable inducing}-(\ref{Lem_lhd_discrete-1}) and Def.\ref{Def-s-reasonable inducing}-(\ref{Lem_lhd_discrete-2}), to conclude that $\mathtt{Prem_{p}}(C_i)$ $\trianglelefteq$ $\ldots$ $\trianglelefteq$ $\mathtt{Prem_{p}}(C_n)$ $\trianglelefteq$ $\mathtt{Prem_{p}}(C_i)$, leading to a contradiction as above.
\end{proof}

\noindent \textbf{Proposition \ref{PropWeakestReasonable}} Let $\preceq$ be defined according to the weakest-link principle, based on a set comparison $\triangleleft_{\mathtt{s}}$ that is reasonable inducing. Then $\preceq$ is \emph{reasonable}.\\[-18pt]

\begin{proof} \emph{Proof of the first condition of reasonableness}:
\\[2pt]
\noindent 1-i) Assume $A$ is strict and firm, and so \sm{$\mathtt{DR}(A)$} = $\emptyset$ and $\mathtt{Prem_{p}}(A) = \emptyset$:
\\ -- Suppose $B$ is strict \sm{($\mathtt{DR}(B)$ = $\emptyset$)}. Then by assumption, $B$ is plausible, i.e,. $\mathtt{Prem_{p}}(B) \neq \emptyset$. Hence by Def.\ref{Def-hd12}, $\mathtt{Prem_{p}}(B)$ $\triangleleft_{\mathtt{s}}$ $\mathtt{Prem_{p}}(A)$, and so by Def. \ref{Def Weakest Link}-1), $B \prec A$.
\\ -- Suppose $B$ is firm \sm{($\mathtt{Prem_{p}}(B) = \emptyset$)}. Then by assumption $B$ is defeasible, i.e,. $\mathtt{DR}(B) \neq \emptyset$. Hence by Def.\ref{Def-hd12}, $\mathtt{DR}(B)$ $\triangleleft_{\mathtt{s}}$ $\mathtt{DR}(A)$, and so by Def. \ref{Def Weakest Link}-2), $B \prec A$.
\\
-- Suppose $B$ is defeasible and plausible. Then by Def.\ref{Def-hd12}, $\mathtt{Prem_{p}}(B)$ $\triangleleft_{\mathtt{s}}$ $ \mathtt{Prem_{p}}(A)$,  and $\mathtt{DR}(B)$ $\triangleleft_{\mathtt{s}}$ $\mathtt{DR}(A)$, and so by Def. \ref{Def Weakest Link}-3), $B \prec A$.
\\[3pt]
\noindent 1-ii) Assume $B$ is strict and firm, and so $\mathtt{DR}(B)$ = $\emptyset$, $\mathtt{Prem_{p}}(B) = \emptyset$. Then by Def.\ref{Def-hd12}, it cannot be that $\mathtt{DR}(B)$ $\triangleleft_{\mathtt{s}}$ $\mathtt{DR}(A)$ or $\mathtt{Prem_{p}}(B)$ $\triangleleft_{\mathtt{s}}$ $\mathtt{Prem_{p}}(A)$, and so it cannot be that $B \prec A$, by the first, second or third condition of Def. \ref{Def Weakest Link}.
\\[3pt]
\noindent 1-iii) Follows straightforwardly from Def. \ref{Def Weakest Link}, given that $A'$ differs from $A$ only in its strict rules and/or axiom premises.
\\[3pt]
\emph{Proof of the second condition of reasonableness}:
\\[2pt]
Suppose for contradiction that:\\[-18pt]

\begin{quote} $\forall i$, there is a strict continuation $C^{+\backslash i}$ of $\{C_1,\ldots,C_{i-1},$ $C_{i+1},\ldots,C_n\}$
such that $C^{+\backslash i} \prec C_i$
\end{quote}
\noindent At least one argument in $\{C_1,\ldots,C_n\}$ must be defeasible or plausible, else every $C^{+\backslash i}$ would be strict and firm, contradicting 1-ii) above. \\[2pt]
\noindent 1) Suppose for some $i$ = $1\ldots n$, $\mathtt{DR}(C_i) \neq \emptyset$ and w.l.o.g assume $i = 1$. Then the assumed weakest link preference $C^{+\backslash 1} \prec C_1$ holds on the basis of Def. \ref{Def Weakest Link}-2) or \ref{Def Weakest Link}-3). One can then reason in exactly the same way as in the proof of the second condition of reasonableness in Proposition \ref{PropLastReasonable} -- case 1) -- substituting `$\mathtt{DR}$' for `$\mathtt{LDR}$', showing that this leads to a contradiction.
\\[3pt]
\noindent 2) Suppose for some $i$ = $1\ldots n$, $\mathtt{Prem_{p}}(C_i) \neq \emptyset$ and (w.l.o.g) $i = 1$. Then, $C^{+\backslash 1} \prec C_1$ by Def. \ref{Def Weakest Link}-1) or \ref{Def Weakest Link}-3). One can then reason in exactly the same way as in the proof of the second condition of reasonableness in Proposition \ref{PropLastReasonable} -- case 1) -- substituting `$\mathtt{Prem_{p}}$' for `$\mathtt{LDR}$', showing that this leads to a contradiction.
\end{proof}

\noindent \textbf{Proposition \ref{PropReasonableInducingEli}} $\triangleleft_{\mathtt{Eli}}$ is reasonable inducing.

\begin{proof} Firstly, note that given a partial preordering $\leq$ (reflexive and transitive) over any $\Gamma$, then it is straightforward to show that the strict counterpart $<$ is a strict partial ordering (irreflexive and transitive).

 $\triangleleft_{\mathtt{Eli}}$ is irreflexive:\\[2pt]
\noindent Assume $\Gamma \triangleleft_{\mathtt{Eli}} \Gamma$. Then $\exists X \in \Gamma$, $X < X$, contradicting the irreflexivity of $<$.

$\triangleleft_{\mathtt{Eli}}$ is transitive:\\[2pt]
\noindent Suppose $\Gamma \triangleleft_{\mathtt{Eli}} \Gamma' \triangleleft_{\mathtt{Eli}} \Gamma''$.
By Def.\ref{Def-hd12}-1) it must be that $\Gamma \neq \emptyset, \Gamma' \neq \emptyset$. If $\Gamma'' = \emptyset$ then $\Gamma \triangleleft_{\mathtt{Eli}} \Gamma''$ by Def.\ref{Def-hd12}-2). Else if $\Gamma'' \neq \emptyset$, $\exists X \in \Gamma$ s.t. $\forall X' \in \Gamma'$, $X < X'$, and $\exists X' \in \Gamma'$ s.t. $\forall X'' \in \Gamma''$, $X' < X''$. Hence by transitivity of $<$, $\exists X \in \Gamma$ s.t. $\forall X'' \in \Gamma''$, $X < X''$, i.e., $\Gamma \triangleleft_{\mathtt{Eli}} \Gamma''$.
\\[-8pt]

We show  \sm{$\triangleleft_{\mathtt{Eli}}$} satisfies the property in Definition \ref{Def-s-reasonable inducing} of reasonable inducing orderings:
\\[2pt]
Assume $\bigcup_{i=1}^n\mathtt{kr}(B_i)  \lhd_{\mathtt{Eli}} \mathtt{kr}(A) $.\\[2pt] By Def.\ref{Def-hd12}-1), it must be that $\bigcup_{i=1}^n\mathtt{kr}(B_i) \neq \emptyset$. \\
Suppose $\mathtt{kr}(A)$ = $\emptyset$. Then by Def.\ref{Def-hd12}-2), $\forall B_i$ s.t. $\mathtt{kr}(B_i) \neq \emptyset$, $B_i \triangleleft_{\mathtt{Eli}} A$.
\\ Suppose $\mathtt{kr}(A) \neq \emptyset$. By assumption,
$\exists Y \in \bigcup_{i=1}^n\mathtt{kr}(B_i)$ s.t. $\forall X \in \mathtt{kr}(A)$, $Y < X$, and so for some $i = 1 \ldots n$, $Y \in \mathtt{kr}(B_i)$,  and so $B_i \triangleleft_{\mathtt{Eli}} A$.
\end{proof}

\noindent \textbf{Proposition \ref{PropReasonableInducingDem}} $\triangleleft_{\mathtt{Dem}}$ is reasonable inducing.
\begin{proof}
 $\triangleleft_{\mathtt{Dem}}$ is irreflexive:\\[2pt]
\noindent Assume $\Gamma \triangleleft_{\mathtt{Dem}} \Gamma$. By Def.\ref{Def-hd12}-1), $\Gamma \neq \emptyset$. Hence $\exists X \in \Gamma$, $X < X$, contradicting the irreflexivity of $<$.

$\triangleleft_{\mathtt{Dem}}$ is transitive:
\\[2pt]
Suppose $\Gamma \triangleleft_{\mathtt{Dem}} \Gamma' \triangleleft_{\mathtt{Dem}} \Gamma''$.
By Def.\ref{Def-hd12}-1) it must be that $\Gamma \neq \emptyset, \Gamma' \neq \emptyset$.\\ If $\Gamma'' = \emptyset$ then $\Gamma \triangleleft_{\mathtt{Dem}} \Gamma''$ by Def.\ref{Def-hd12}-2). Else if $\Gamma'' \neq \emptyset$, $\forall X \in \Gamma$, $\exists X' \in \Gamma'$ s.t. $X < X'$, and $\forall X' \in \Gamma'$, $\exists X'' \in \Gamma''$, s.t. $X' < X''$. Hence, by transitivity of $<$, $\forall X \in \Gamma$, $\exists X'' \in \Gamma''$ s.t. $X < X''$, i.e., $\Gamma \triangleleft_{\mathtt{Dem}} \Gamma''$.
\\[-8pt]

We show  $\triangleleft_{\mathtt{Dem}}$ satisfies the property in Definition \ref{Def-s-reasonable inducing} of reasonable inducing orderings:
\\[2pt]
Assume $\bigcup_{i=1}^n\mathtt{kr}(B_i)  \lhd_{\mathtt{Dem}} \mathtt{kr}(A)$. By Def.\ref{Def-hd12}-1), it must be that $\bigcup_{i=1}^n\mathtt{kr}(B_i) \neq \emptyset$. \\
Suppose $\mathtt{kr}(A)$ = $\emptyset$. Then by Def.\ref{Def-hd12}-2), $\forall B_i$ s.t. $\mathtt{kr}(B_i) \neq \emptyset$, $B_i \triangleleft_{\mathtt{Dem}} A$.
 \\[2pt]
 Suppose $\mathtt{kr}(A) \neq \emptyset$. By assumption, $\forall Y \in \bigcup_{i=1}^n\mathtt{kr}(B_i)$, $\exists X \in \mathtt{kr}(A)$, $Y < X$. Hence $\forall B_i$ s.t. $\mathtt{kr}(B_i) \neq \emptyset$, $\forall Y \in \mathtt{kr}(B_i)$, $\exists X \in \mathtt{kr}(A)$ s.t. $Y < X$ and so $B_i \triangleleft_{\mathtt{Dem}} A$.
 \end{proof}

% some $i = 1 \ldots n$, $\forall Y \in \mathtt{kr}(B_i)$,  $\exists X \in \mathtt{kr}(A)$, $Y \leq X$.
%\\[2pt]
%(2) By assumption, $\mathtt{kr}(A) \ntrianglelefteq_{\mathtt{Dem}}  \bigcup_{i=1}^n\mathtt{kr}(B_i) $; i.e., it is not the case that $\forall X \in \mathtt{kr}(A) $, $\exists Y \in  \bigcup_{i=1}^n\mathtt{kr}(B_i) $, $X \leq Y$. That is:\\[-7pt]
%
%\hspace{15mm}  $\exists X \in \mathtt{kr}(A)$ s.t. $\forall Y \in \bigcup_{i=1}^n\mathtt{kr}(B_i) $ , $\neg(X \leq Y)$\hspace{29mm} \textbf{5)}
%\\[3pt]
%Suppose for some $B_i$, $\mathtt{kr}(A) \trianglelefteq \mathtt{kr}(B_i)$. Then $\forall X \in \mathtt{kr}(A) $, $\exists Y \in \mathtt{kr}(B_i) $ s.t. $X \leq Y$, contradicting \textbf{5)}.
%
%
%\end{proof}

In the following results we make use of the fact that property 1-ii) of reasonable argument orderings (if $B$ is strict and firm then $B \nprec A$) is satisfied by the last and weakest link principle, assuming \emph{any} set comparison $\triangleleft_{\mathtt{s}}$, as this property is shown in Propositions \ref{PropLastReasonable} and \ref{PropWeakestReasonable}, without relying on any properties of $\triangleleft_{\mathtt{s}}$.\\

\noindent \textbf{Proposition} \ref{PropLastStrictPartial}  Let $\prec$ be defined according to the last-link principle, based on a set comparison  $\triangleleft_{\mathtt{s}}$ that is a strict partial order. Then $\prec$ is a strict partial order.\\[-17pt]

\begin{proof}
\noindent \emph{Irreflexivity} Suppose $A \prec A$. Hence $A$ is plausible or defeasible. Either
 $\mathtt{LDR}(A) \triangleleft_{\mathtt{s}} \mathtt{LDR}(A)$ or $\mathtt{Prem_{p}}(A)$ $\triangleleft_{\mathtt{s}}$ $\mathtt{Prem_{p}}(A)$ contradicting the irreflexivity of $\triangleleft_{\mathtt{s}}$.
\\[2pt]
\emph{Transitivity}: Suppose $C \prec B \prec A$. It must be that both $C$ and $B$ are plausible or defeasible.
\\[2pt]
1) Suppose $B \prec A$ by Def. \ref{Def Last Link}-1. By Def. \ref{Def-hd12}-2), $\mathtt{LDR}(B) \neq \emptyset$. Hence  $C \prec B$ by Def. \ref{Def Last Link}-1, $\mathtt{LDR}(C) \neq \emptyset$,  and  $\mathtt{LDR}(C) \triangleleft_{\mathtt{s}} \mathtt{LDR}(B) \triangleleft_{\mathtt{s}} \mathtt{LDR}(A)$. By transitivity of  $\triangleleft_{\mathtt{s}}$, $\mathtt{LDR}(C) \triangleleft_{\mathtt{s}} \mathtt{LDR}(A)$, and so  $C \prec A$ by Def. \ref{Def Last Link}-1.
\\[2pt]
2) Suppose $B \prec A$ by Def. \ref{Def Last Link}-2. Hence $\mathtt{LDR}(A) = \mathtt{LDR}(B)  = \emptyset$ and $\mathtt{Prem_{p}}(B)$ $\triangleleft_{\mathtt{s}}$ $\mathtt{Prem_{p}}(A)$. \\
2.1) Suppose $\mathtt{LDR}(C) \neq \emptyset$. Then (by Def. \ref{Def-hd12}-2), $\mathtt{LDR}(C) \triangleleft_{\mathtt{s}} \mathtt{LDR}(A)$  and  $C \prec A$ by Def. \ref{Def Last Link}-1.\\
2.2) Suppose $\mathtt{LDR}(C) = \emptyset$. Then $\mathtt{Prem_{p}}(C) \neq \emptyset$ and $C \prec B$ by Def. \ref{Def Last Link}-2. Hence $\mathtt{Prem_{p}}(C)$ $\triangleleft_{\mathtt{s}}$ $\mathtt{Prem_{p}}(B)$ $\triangleleft_{\mathtt{s}}$ $\mathtt{Prem_{p}}(A)$, and so $\mathtt{Prem_{p}}(C)$ $\triangleleft_{\mathtt{s}}$ $\mathtt{Prem_{p}}(A)$, and so $C \prec A$ by Def. \ref{Def Last Link}-2.
\end{proof}

\noindent \textbf{Proposition} \ref{PropWeakestStrictPartial} Let $\prec$ be defined according to the weakest-link principle, based on a set comparison $\triangleleft_{\mathtt{s}}$ that is a strict partial order. Then $\prec$ is a strict partial order.\\[-17pt]

\begin{proof}\emph{Irreflexivity} Suppose $A \prec A$. Hence $A$ is plausible or defeasible. Then
 $\mathtt{LDR}(A) \triangleleft_{\mathtt{s}} \mathtt{LDR}(A)$ and/or $\mathtt{Prem_{p}}(A)$ $\triangleleft_{\mathtt{s}}$ $\mathtt{Prem_{p}}(A)$ contradicting the irreflexivity of $\triangleleft_{\mathtt{s}}$.
\\[2pt]
\emph{Transitivity}: Suppose $C \prec B \prec A$. It must be that both $C$ and $B$ are plausible or defeasible.
\\[2pt]
1) Suppose $B \prec A$ by Def. \ref{Def Weakest Link}-1 ($B$ and $A$ are strict). Hence $\mathtt{Prem_{p}}(B) \neq \emptyset$, and  $\mathtt{Prem_{p}}(B)\triangleleft_{\mathtt{s}} \mathtt{Prem_{p}}(A)$. Also $C \prec B$ by 1 or 3 in Def. \ref{Def Weakest Link}. \\
1.1) In the former case, $C$ is strict, hence $\mathtt{Prem_{p}}(C) \neq \emptyset$, $\mathtt{Prem_{p}}(C)\triangleleft_{\mathtt{s}} \mathtt{Prem_{p}}(B)$, and so by transitivity $\mathtt{Prem_{p}}(C)\triangleleft_{\mathtt{s}} \mathtt{Prem_{p}}(A)$ and $C \prec A$ by Def. \ref{Def Weakest Link}-1.\\
1.2) In the latter case, $\mathtt{DR}(C) \neq \emptyset$ and so $\mathtt{DR}(C)\triangleleft_{\mathtt{s}} \mathtt{DR}(A)$ (since $\mathtt{DR}(A)$ = $\emptyset$), and $\mathtt{Prem_{p}}(C)\triangleleft_{\mathtt{s}} \mathtt{Prem_{p}}(A)$ as shown in 1.1, and so $C \prec A$ by Def. \ref{Def Weakest Link}-3.\\[3pt]
2) Suppose $B \prec A$ by Def. \ref{Def Weakest Link}-2 ($B$ and $A$ are firm). Hence $\mathtt{DR}(B) \neq \emptyset$, and  $\mathtt{DR}(B)\triangleleft_{\mathtt{s}} \mathtt{DR}(A)$. Also $C \prec B$ by 2 or 3 in Def. \ref{Def Weakest Link}. \\
2.1) In the former case, $C$ is firm, hence $\mathtt{DR}(C) \neq \emptyset$, $\mathtt{DR}(C)\triangleleft_{\mathtt{s}} \mathtt{DR}(B)$, and so by transitivity $\mathtt{DR}(C)\triangleleft_{\mathtt{s}} \mathtt{DR}(A)$ and $C \prec A$ by Def. \ref{Def Weakest Link}-2.\\
2.2) In the latter case, $\mathtt{Prem_{p}}(C) \neq \emptyset$ and so $\mathtt{Prem_{p}}(C)\triangleleft_{\mathtt{s}} \mathtt{Prem_{p}}(A)$ (since $\mathtt{Prem_{p}}(A)$ = $\emptyset$), and $\mathtt{DR}(C)\triangleleft_{\mathtt{s}} \mathtt{DR}(A)$ as shown in 2.1, and so $C \prec A$ by Def. \ref{Def Weakest Link}-3.
\\[3pt]
3) Suppose $B \prec A$ by Def. \ref{Def Weakest Link}-3 ($B$ and $A$ are plausible and defeasible and $\mathtt{Prem_{p}}(B)\triangleleft_{\mathtt{s}} \mathtt{Prem_{p}}(A)$, $\mathtt{DR}(B)\triangleleft_{\mathtt{s}} \mathtt{DR}(A)$). Hence $\mathtt{Prem_{p}}(B) \neq \emptyset$, $\mathtt{DR}(B) \neq \emptyset$, and so $C \prec B$ by Def. \ref{Def Weakest Link}-3. Hence, by transitivity $\mathtt{Prem_{p}}(C)\triangleleft_{\mathtt{s}} \mathtt{Prem_{p}}(A)$, $\mathtt{DR}(C)\triangleleft_{\mathtt{s}} \mathtt{DR}(A)$ and $C \prec A$ by Def. \ref{Def Weakest Link}-3.
\end{proof}

\subsection{Proofs for Section~\ref{sectionALstrict}}\label{SectionProofsAL}

\begin{lemma}\label{prop2}Let $(AS,\K)$ be the abstract logic argumentation theory based on $(\LA',Cn)$ and $(\Sigma,\leq')$. Then for all $X \subseteq \Sigma$ it holds that $p \in Cn(X)$ iff $X \vdash p$.\\[-17pt]

\begin{proof}
From left to right, suppose $p \in Cn(X)$ for some $X \subseteq \Sigma$. By Def. \ref{DefAbstractLogic}-(3), $X$ is finite, so $X \imp p \in \R_s$ (by Def. \ref{asbasedonal}), so $X \vdash p$ (recall that $X \vdash p$ denotes a strict \ASPIC\ argument for p based on premises $X' \subseteq X$).
\\[2pt]
From right to left is proven by induction on the structure of arguments. Assume $A$ = $X \vdash p$, where by Def. \ref{asbasedonal}, $X \subseteq \K_p$, $X \subseteq \Sigma$. Assume first $p \in X$. Then $p \in Cn(X)$ by Def. \ref{DefAbstractLogic}-(1). Consider next any $A = A_1,\ldots,A_n \imp \varphi$. By inductive hypothesis, $\mathtt{Conc}(A_1),\ldots,\mathtt{Conc}(A_n) \in Cn(X)$. Since $\mathtt{Conc}(A_1),\ldots,\mathtt{Conc}(A_n) \imp \varphi \in \R_s$, then by Def.\ref{asbasedonal}, $\varphi \in Cn(\bigcup_{i=1}^n{Conc}(A_i))$. Since $\bigcup_{i=1}^n{Conc}(A_i) \subseteq Cn(X)$, then by monotonicity, $\varphi \in Cn(Cn(X))$. But then by Def. \ref{DefAbstractLogic}-(2), $\varphi \in Cn(X)$.
\end{proof}

\end{lemma}

\begin{lemma}\label{easylemma}  Let $(\LA,Cn)$ be an abstract logic. For any finite $S \subseteq \LA$ and any $p \in \LA$, if $S \cup \{p\}$ is AL inconsistent, then there exists an $s \in Cn(S)$ such that $\{s,p\}$ is AL-inconsistent.\\[-17pt]

\begin{proof}
Since $S$ is finite, we can with repeated application of the definition of adjunction conclude that there exists an $s$ such that $Cn(\{s\}) = Cn(S)$. By Def. \ref{DefAbstractLogic}-(1), $s \in Cn(\{s\})$ and so $s \in Cn(S)$. By  Def. \ref{DefAbstractLogic}-(2), $Cn(S \cup \{p\}) = Cn(\{s\} \cup \{p\})$. Then $Cn(\{s\} \cup \{p\}) = \LA$ so $\{s,p\}$ is AL-inconsistent.
\end{proof}
\end{lemma}

\begin{lemma}\label{hardlemma}
If $X \cup Z$ is AL-inconsistent and $X \subseteq Cn(Y)$ then $Y \cup Z$ is AL-inconsistent.\\[-17pt]

\begin{proof}
We prove the contraposition that if $Y \cup Z$ is AL-consistent and $X \subseteq Cn(Y)$, then $X \cup Z$ is AL-consistent. To prove this, we first show that:\\[-17pt]

\begin{quote}
If $X \subseteq Cn(Y)$, then $Cn(X \cup Y) = Cn(Y)$ \hspace{15mm} (\textbf{1})\\[-17pt]
\end{quote}
By monotonicity, $Cn(Y) \subseteq Cn(X \cup Y)$. We prove that $Cn(X \cup Y) \subseteq Cn(Y)$. Let $X \subseteq Cn(Y)$. By Def.\ref{DefAbstractLogic}-(1), $Y \subseteq Cn(Y)$. But then $X \cup Y \subseteq Cn(Y)$. Then by monotonicity $Cn(X \cup Y) \subseteq Cn(Cn(Y))$. Since $Cn(Cn(Y)) = Cn(Y)$ (by Def.\ref{DefAbstractLogic}-(2)), then $Cn(X \cup Y) \subseteq Cn(Y)$. So $Cn(X \cup Y) = Cn(Y)$.\\ We have shown (\textbf{1}). Then by property (7) in Section \ref{sectionALstrict}, $Cn(X \cup Y \cup Z) = Cn(Y \cup Z)$. Now if $Cn(Y \cup Z) \neq \LA$ then $Cn(X \cup Y \cup Z) \neq \LA$. But then by monotonicity  $Cn(X \cup Z) \neq \LA$.\\[-12pt]

%Given that $Y \cup Z$ is AL-consistent and (*), then by Property (3), $Cn(X \cup Y \cup Z) = Cn(Y \cup Z)$. Now if $Cn(Y \cup Z) \neq \LA$ then $Cn(X \cup Y \cup Z) \neq \LA$. But then by nonmonotonicity  $Cn(X \cup Z) \neq \LA$.
\end{proof}

\end{lemma}

\noindent \textbf{Proposition~\ref{ab=aspic}}
A \emph{c-SAF} based on an AL argumentation theory is closed under contraposition, axiom consistent, c-classical, and well-formed.
\begin{proof}
Well-formedness immediately follows from the fact that the $\mbox{}^{-}$ relation is symmetric. Axiom consistency follows from the fact that $\K_n = \emptyset$. To prove satisfaction of contraposition we must prove:\\[-15pt]
\begin{quote}
 If $S \vdash p$ then for all $s \in S$ it holds that $S \setminus \{s\} \cup \{-p\} \vdash -s$ for any $-p$ and $-s$.\\[-15pt]
\end{quote}
 By definition of $\vdash$, if $S \vdash p$ then $S' \vdash p$ for some finite $S' \subseteq S$. Therefore we can without loss of generality assume that $S$ is finite. Next, if $S \vdash p$ then $p \in Cn(S)$ by Lemma~\ref{prop2}. Consider any $-p$. Then $\{p,-p\}$ is AL-inconsistent by Def.~\ref{asbasedonal}(\ref{conconb}). But then $S \cup \{-p\}$ is AL-inconsistent by Lemma~\ref{hardlemma}.
 Then by simple rewriting for all $s \in S$ it holds that $(S \setminus \{s\}) \cup \{-p\} \cup \{s\} $ is AL-inconsistent. By assumption that  $S$ is finite, Lemma~\ref{easylemma} then yields that there exists an $s' \in Cn(S \setminus \{s\} \cup \{-p\})$ such that $\{s',s\}$ is AL-inconsistent. Then by Def.~\ref{asbasedonal}(\ref{conconc}) there exists an $s'' \in Cn(\{s'\})$ such that $s'' = -s$. By Def.\ref{DefAbstractLogic}-(2) and monotonicity, $s'' \in Cn(S \setminus \{s\} \cup \{-p\})$. Hence, $-s \in$ $Cn(S \setminus \{s\} \cup \{-p\})$. Hence, %for some minimal subset $T \subseteq S \setminus \{s\} \cup \{-p\}$ we have that $T \imp -s \in \R_s$, and so
 $S \setminus \{s\} \cup \{-p\} \vdash -s$ by Lemma \ref{prop2}. \\[3pt]
C-classicality is proven as follows. We first prove that if
$S \subseteq \LA$ is AL-inconsistent, then some finite $S' \subseteq S$
is AL-inconsistent. By Def.\ref{DefAbstractLogic}-(4), $\exists p \in \LA$ such that $Cn(p) = \LA$. Let $p$ be denoted by
$\bot$. Now suppose $S$ is inconsistent. Then $Cn(S) = \LA$, so $\bot
\in Cn(S)$. By Def.\ref{DefAbstractLogic}-(3), $\bot \in Cn(S')$
for some finite $S' \subseteq S$. But since $Cn(Cn(S')) = Cn(S')$
(Def.\ref{DefAbstractLogic}-(2)), this implies that $Cn(S') = \LA$.\\
Now assume $S \subseteq \LA$ is minimally c-inconsistent. Then for some
$p$ it holds that $S  \vdash p,-p$.  By Lemma \ref{prop2}, $\{p,-p\} \subseteq Cn(S)$. By
Def.~\ref{asbasedonal}(\ref{conconb}), $\{p,-p\}$ is
AL-inconsistent, and so by monotonicity $Cn(S)$ is
AL-inconsistent. But then by Def.\ref{DefAbstractLogic}-(2),$S$ is
AL-inconsistent, and so some finite $S' \subseteq S$ is AL-inconsistent.
But since $S$ is minimally inconsistent, it holds that $S' = S$. Consider
any $s \in S$. Then $S \setminus \{s\} \cup \{s\}$ is inconsistent. By
adjunction and finiteness of $S$ there exists a formula $x \in \LA$ that
has exactly the same consequences as $S \setminus \{s\} $. Then by property (7) in Section \ref{sectionALstrict} $\{x,s\}$ is AL-inconsistent, and so
by Def.~\ref{asbasedonal}(\ref{conconc}) there exists a $y \in
Cn(\{x\})$ such that $y = -s$. But then $S \setminus \{s\} \vdash -s$.\\[-18pt]

 \end{proof}

\begin{lemma}\label{LemmaFact}
For any $(AS,\K)$ based on $(\LA',Cn)$ and $(\Sigma,\leq')$, it holds that $\{p,q\}$ is AL-inconsistent iff $p \vdash -q$ for some $-q$.\\[-17pt]
\begin{proof}
Suppose $\{p,q\}$ is AL-inconsistent. By Def.~\ref{asbasedonal}(\ref{concon}) $-q \in Cn\{p\}$ for some $-q$. By Lemma \ref{prop2}, $p \vdash -q$. Suppose $p \vdash -q$ for some $-q$. Then $-q \in Cn(\{p\})$ by Lemma \ref{prop2}.  By monotonicity $-q \in Cn(\{p,q\})$.  Furthermore, $q \in Cn(\{p,q\})$ by Def.\ref{DefAbstractLogic}-(1). Since $q$ and $-q$ are contradictories, then by Def.~\ref{asbasedonal}(\ref{concon})-a\&b,  $\{q,-q\}$ is AL-inconsistent. Hence $Cn(\{q,-q\}) = \LA$. Then by Def. \ref{DefAbstractLogic}-(2) and monotonicity, $Cn(\{p,q\}) = \LA$.
\end{proof}
\end{lemma}
%
%\begin{lemma}\label{factAL}
%For any $(AS,\K)$ based on abstract logic $(\LA',Cn)$ it holds that $\{p,q\}$ is AL-inconsistent iff $p \vdash -q$ for some $-q$.
%\end{fact}
%This fact allows us to define two notions of AL-undermining and AL-rebutting attack for \ASPIC\ where in Definition~\ref{DefAttacks} the two occurrences of ``$\mathtt{Conc}(A) \in \con{\varphi}$'' are replaced by ``$\mathtt{Conc}(A) \vdash \con{\varphi}$''.
%%Then an \emph{abstract argumentation framework} on the basis of argumentation theory $AT$ is defined as a pair $<$$\A$, \emph{Def}$>$ such that $\A$ is the set of arguments on the basis of $AT$ as defined by Definition \ref{arg},
%Then an \emph{AL (c-)SAF} defined by argumentation theory $AT$ is defined as in Definition~\ref{DefinitionStructuredAF} except that the defeat relation is generated by the new definition of attack. Then:

\noindent \textbf{Proposition~\ref{ALprop}}  Let $(AS,\K)$ be based on $(\LA',Cn)$ and $(\Sigma,\leq')$. Let $\Delta_1$ be the \emph{c-SAF} defined by $(AS,\K)$ and $\leq'$, and $\Delta_2$ the \emph{AL-c-SAF} defined by $(AS,\K)$ and $\leq'$. Then, for $T \in \{$complete, grounded, preferred, stable$\}$, $E$ is a $T$ extension of $\Delta_1$ iff $E$ is a $T$ extension of $\Delta_2$.\\[-17pt]

  \begin{proof} We first show that:

 \hspace{5mm}$X$ is acceptable w.r.t. $E$ in $\Delta_1$ iff $X$ is acceptable w.r.t. $E$ in $\Delta_2$\hspace{10mm}\textbf{(1)}
  \\1.1) Firstly, if $Y$ undermines $X$, then $p$ $(= \mathtt{Conc}(Y)) \in \con{q}$, where $q \in \mathtt{Prem}(X)$. By def.\ref{asbasedonal}(\ref{concon})-b, $\{p,q\}$ is AL-inconsistent. By Lemma \ref{LemmaFact}, $p \vdash -q$. Hence $Y$ \ASPIC-AL-undermines $X$.
   \\
   1.2) Secondly, if $Y$ \ASPIC-AL-undermines $X$, then $p$ $(= \mathtt{Conc}(Y)) \vdash -q$, where $q \in \mathtt{Prem}(X)$. Hence $Y$ can be strictly continued into an argument $Y'$ that concludes $-q$, and so $Y'$ undermines $X$.
   \\[2pt]
   Suppose $Y$ defeats $X$ in $\Delta_1$ ($Y \hookrightarrow_{\Delta_1} X$), and so $\exists Z \in E$, $Z \hookrightarrow_{\Delta_1} Y$. By 1.1), $Y \hookrightarrow_{\Delta_2} X$ and $Z \hookrightarrow_{\Delta_2} Y$. Hence the left to right half of 1) is shown.\\ Suppose $Y \hookrightarrow_{\Delta_2} X$, and so $\exists Z \in E$, $Z \hookrightarrow_{\Delta_2} Y$. By 1.2) there is a strict continuation $Y'$ of $Y$ s.t. $Y'$ undermines $X$. By condition 1-iii) of $\preceq$ being reasonable (Def.\ref{Def-Reasonable}), it remains the case that
   $Y' \nprec X$, and so $Y' \hookrightarrow_{\Delta_1} X$. By the same reasoning, there is a strict continuation $Z'$ of $Z$, s.t. $Z' \nprec Y$ and so $Z' \hookrightarrow_{\Delta_1} Y$. Since $Z'$ undermines  $Y$ then $Z'$ undermines  $Y'$. By condition 1-iii) of $\preceq$ being reasonable,  $Z' \nprec Y'$. Hence $Z' \hookrightarrow_{\Delta_1} Y'$. By Lemma \ref{Lemma3}, $Z'$ is acceptable w.r.t. $E$, and by Proposition \ref{PropositionConflictFree}, $E \cup \{Z'\}$ is conflict free. Hence, since $E$ is complete, $Z' \in E$. Hence the right to left half of 1) is shown.
   \\[2pt]
    Notice that (\textbf{1}) is shown in exactly the same way assuming the defeat definition of conflict free, except that in the final part of the above proof, we include the extra step of reasoning that since $E \cup \{Z'\}$ is conflict free under the attack definition, it trivially follows that $E \cup \{Z'\}$ is conflict free under the defeat definition (since the defeat relation is a subset of the attack relation).
    \\[2pt]
    Given (\textbf{1}), the main proposition now follows from the following:

     \hspace{10mm}$E$ is conflict free in $\Delta_1$ iff $E$ is conflict free in $\Delta_2$ \hspace{10mm}\textbf{(2)}\\
     \emph{Proof of} \textbf{(2)} \emph{under the attack definition of conflict free}: Suppose $E$ is conflict free in $\Delta_1$, $E$ is not conflict free in $\Delta_2$. Then $\exists Y,X \in E$ such that $Y$ does not undermine $X$, $Y$ \ASPIC-AL-undermines $X$. By 1.2, there is a strict continuation $Y'$ of $Y$ s.t. $Y'$ undermines $X$. Applying the same reasoning as for $Z'$ above, $Y' \in E$, contradicting $E$ is conflict free in $\Delta_1$. Suppose $E$ is conflict free in $\Delta_2$, $E$ is not conflict free in $\Delta_1$. Then $\exists Y,X \in E$ such that $Y$ does not \ASPIC-AL-undermines $X$, $Y$ undermines $X$, contradicting 1.1. Hence \textbf{(2)} is shown.
     \\[2pt]
     \emph{Proof of} \textbf{(2)} \emph{under the defeat definition of conflict free}: Suppose $E$ is conflict free in $\Delta_1$, $E$ is not conflict free in $\Delta_2$. Then $\exists Y,X \in E$ such that $Y$ does not defeat $X$ in $\Delta_1$, $Y$ defeats $X$ in $\Delta_2$. Given the latter, $Y$ \ASPIC-AL-undermines $X$ on $X'$, $Y \nprec X'$. But then there is a strict continuation $Y'$ of $Y$ s.t. $Y'$ undermines $X$ on $X'$, and by condition 1-iii) of $\preceq$ being reasonable, $Y' \nprec X'$, and so $Y'$ defeats $X$ on $X'$. Applying the same reasoning as for $Z'$ above, $Y' \in E$, contradicting $E$ is conflict free in $\Delta_1$.\\ Suppose $E$ is conflict free in $\Delta_2$, $E$ is not conflict free in $\Delta_1$. Then $\exists Y,X \in E$ such that $Y$ does not defeat $X$ in $\Delta_2$, $Y$ defeats $X$ in $\Delta_1$. Given the latter, $Y$ undermines $X$ on $X'$, $Y \nprec X'$. But then we immediately have that $Y$ \ASPIC-AL-undermines $X$ on $X'$ and so $Y$ defeats $X$ in $\Delta_2$. Contradiction.
  \end{proof}

%\begin{proof}
%If $A$ undermines/rebuts $B$ then $A$ AL-undermines/AL-rebuts $B$. Moreover, if $A$ AL-undermines/AL-rebuts $B$ then $A$ can with a single strict inference be continued into to an argument $A'$ that undermines/rebuts $B$. Given our assumptions on $\preceq$ it holds for any $B$ that $A$ and $A'$ have the same preference relation with $B$. {\bf *** HP: The proof of Prop 25 must be changed in the same way. ***} Then, since $A$ and $A'$ have the same premises and defeasible inferences, they also have the same attackers and defeaters. Then clearly $A \in E$ iff $A' \in E$.
%\end{proof}

\noindent \textbf{Proposition~\ref{abarg=aspicarg}} Let $(AS,\K)$ be based on $(\LA',Cn)$ and $(\Sigma,\leq')$. Then $A$ is a c-consistent premise minimal argument on the basis of $(AS,\K)$ iff $(\mathtt{Prem}(A),\mathtt{Conc}(A))$ is an abstract logic argument on the basis of $(\Sigma,\leq')$.

\begin{proof}
\emph{Right to left }: let $(X,p)$ be an \emph{AL} argument. Then $X \imp p \in \R_s$, and so $A$ is a strict \ASPIC\ argument with $\mathtt{Prem}(A) = X$, $\mathtt{Conc}(A) = p$. $A$ must be premise minimal since otherwise there is a strict \ASPIC\ argument $A'$ for $p$ with $X' = \mathtt{Prem}(A')$, $X' \subset X$. But then by Lemma \ref{prop2}, $p \in Cn(X')$, contradicting the minimality of $(X,p)$. Suppose for contradiction that $A$ is not c-consistent. Then $X \vdash p,-p$, and so by Lemma \ref{prop2}, $\{p,-p\} \subseteq Cn(X)$. By Def.~\ref{asbasedonal}(\ref{conconb}), $\{p,-p\}$ is AL-inconsistent, and so $Cn(\{p,-p\}) = \LA$. Then by Def. \ref{DefAbstractLogic}-(2) and monotonicity, $Cn(X) = \LA$, so $X$ is AL-inconsistent. Contradiction.
\\[2pt]
\noindent \emph{Left to right} : let $A$ be a c-consistent premise minimal \ASPIC\ argument with $\mathtt{Prem}(A) = X$, $\mathtt{Conc}(A) = p$ (i.e., $X \vdash p$). Then $X \subseteq \Sigma$ and $p \in Cn(X)$ (by Lemma~\ref{prop2}), satisfying (1) and (3) of Def.~\ref{ALarg+ALattack}. If $p \in Cn(X')$ for some $X' \subset X$, then $X' \imp p \in \R_s$, and since $X' \subseteq \K_p$, there is an $A'$ s.t. $\mathtt{Prem}(A') = X'$, $\mathtt{Conc}(A) = p$, contradicting $A$ is premise minimal. Hence condition (4) of Def.~\ref{ALarg+ALattack} is satisfied. Suppose for contradiction that $X$ is AL-inconsistent. By repeated application of adjunction to the formulae in $X \setminus \{\varphi\}$, for some $\varphi \in X$, we have that $Cn(X)$ = $Cn(\{\phi,\varphi\})$ = $\LA'$. By Def.~\ref{asbasedonal}(\ref{conconc}), $\exists \phi' \in Cn(\{\phi\})$, s.t. $\phi' \in \con{\varphi}$, and so $\phi' = -\varphi$. By monotonicity, $-\varphi \in Cn(\{\phi,\varphi\})$. By Def.\ref{DefAbstractLogic}-(1), $\varphi \in Cn(\{\phi,\varphi\})$. Hence $\varphi, -\varphi \in Cn(X)$. By Lemma~\ref{prop2}, $X \vdash \varphi$, $X \vdash -\varphi$; i.e., $A$ is c-inconsistent. Contradiction.
\end{proof}

%\noindent \textbf{Proposition \ref{PropositionEquivalenceAL-ASPIC}} Let $\Gamma$ = $(\A',\C',\preceq')$ be the \emph{PAF} defined by the theory $(\Sigma,\leq')$ in the abstract logic $(\LA',Cn)$, and $(AS,\K)$ the AL argumentation theory based on $(\LA',Cn)$ and $(\Sigma,\leq')$.
%\\[2pt]
%Let $\Delta$ be the \emph{c-SAF} $(\A,\C,\preceq)$ defined by $(AS,\K)$, where in Definition \ref{DefinitionStructuredAF}, `$\A$ is the set of all premise minimal c-consistent arguments' replaces `$\A$ is the set of all c-consistent arguments', and $\preceq$ = $\preceq'$.
%\\[2pt]
% Then for $T \in \{$complete, grounded, preferred, stable$\}$, $E$ is a $T$ extension of $\Gamma$ iff $E$ is a $T$ extension of $\Delta$.
%
%\begin{proof}jj
%
%\end{proof}

In the following lemma we introduce, for any argument $A \in \A$, the additional notation $A_+$ to denote any argument such that $\mathtt{Prem}(A) \subseteq \mathtt{Prem}(A_+)$ and $\mathtt{Conc}(A) =  \mathtt{Conc}(A_+)$.

\begin{lemma}\label{classlemma}
Consider any $AT$ for which $\preceq$ is defined such that $\forall A, B, A_-$, if $A \nprec B$ then $A_-  \nprec B$.\\[-17pt]
\ben
\item\label{A-defeatsB+} If $A$ defeats $B$ then $A_-$ defeats $B_+$ for all $A_-$ and $B_+$.\\[-17pt]
\item\label{A-inE}
For all complete extensions $E$:\\[-17pt]
\ben
\item if $A \in E$ then $A_-  \in E$ for all $A_-$;\\[-15pt]
\item if $B \not\in E$ then $B_+ \not\in E$ for all $B_+$.\\[-17pt]
\een
\een
\end{lemma}
\begin{proof}
(1) If $A$ defeats $B$ based on a preference independent attack on a sub-argument $B'$ of $B$, then $A_-$ preference independent attacks and defeats $B_+$ on $B'$; else $A$ preference dependent attacks and defeats $B$ on $B'$, $A \nprec B'$. Since $A_-  \nprec B'$, $A_-$ defeats $B_+$ (on $B'$).\\
(2a) Let $A \in E$ and $B$ defeat any $A_-$. Then $B$ also defeats $A$ by \ref{A-defeatsB+}. Then there exists a $C$ that defeats $B$, so $A_-$ is acceptable with respect to $A$, so (since $E$ is complete) $A_- \in E$.\\
(2b) Let  $B \not\in E$. Then there exists an $A$ that defeats $B$, and no $C \in E$ s.t. $C \hookrightarrow A$. But then $A$ defeats $B_+$ by \ref{A-defeatsB+}, so $B_+$ is not acceptable with respect to $E$.
\end{proof}

\noindent \textbf{Proposition \ref{classicalmin}} Let $\Delta$ be the \emph{(c)-SAF} $(\A,\C,\preceq)$ defined on the basis of an $AT$ for which $\preceq$ is defined such that for any $A \in \A$,  if $A \nprec B$ then $A_- \nprec B$.
\\[2pt]
 Let $\Delta^-$ be the premise minimal \emph{(c)-SAF} $(\A^-,\C^-,\preceq^-)$ where:\\[-17pt]

\begin{itemize}
\item $\A^-$ is the set of premise minimal arguments in $\A$.\\[-17pt]
\item $\C^-$ = $\{(X,Y) | (X,Y) \in \C, X,Y \in \A^-\}$.\\[-17pt]
\item $\preceq^-$ = $\{(X,Y) | (X,Y) \in \preceq, X,Y \in \A^-\}$.
\end{itemize}

 Then for $T \in \{$complete, grounded, preferred, stable$\}$, $E$ is a $T$ extension of $\Delta$ iff $E'$ is a $T$ extension of $\Delta^-$, where $E'\subseteq E$ and
$ \bigcup_{X\in E}\mathtt{Conc}(X) = \bigcup_{Y\in E'}\mathtt{Conc}(Y)$.

\begin{proof}

\noindent $T$ = \emph{\textbf{complete}}:\\ 1) Suppose $E$ is a complete extension of $\Delta$. We show that $E-$ is a complete extension of $\Delta^-$, where $E- = E \cap \A^-$.\\
Note first that $E-$ is (attack) conflict-free by construction.
\\
(i) Let $A \in E-$ and consider any $B \in \A^-$ that defeats $A$. Since $A \in E$, there exists a $C \in E$ that defeats $B$. But then (by Lemma~\ref{classlemma}(\ref{A-defeatsB+})) for all $C_-$ we also have that $C_-$ defeats $B$. Since all such $C_- \in E$ (by Lemma~\ref{classlemma}(\ref{A-inE})), all such $C_-$ are in $E-$, and so $A$ is acceptable with respect to $E-$.
\\
(ii) Let $A \in \A^-, A \not\in E-$. Then $A \not \in E$, some $B$ defeats $A$, $\neg \exists C \in E$, $C$ defeats $B$. There exists a $B_-$ that (by Lemma~\ref{classlemma}(\ref{A-defeatsB+})) defeats $A$. Suppose for contradiction that $A$ is acceptable w.r.t. $E-$. Then, $\exists C' \in E^-$, $C'$ defeats $B_-$. By Lemma~\ref{classlemma}(\ref{A-defeatsB+}), $C'$ defeats $B$. Since $E- \subseteq E$, $C' \in E$, contradicting $\neg \exists C \in E$, $C$ defeats $B$. So $A$ is not acceptable with respect to $E-$.\\
Given i) and ii), $E-$ is a complete extension of $\Delta^-$.
\\[2pt]
2) Suppose $E$ is a  complete extension of $\Delta^-$. We show that $E+$ is a $T$-extension of $\Delta$, where  $E+ = E \cup \{A_+ \in \A \mid A \in E$ and $ \mathtt{Prem}(A_+) \subseteq E\}$. Suppose $E$ is a complete extension of $\Delta^-$.\\
Note first that $E+$ is conflict-free by construction.
\\
(i)  For any $A \in E$ and any $B \in \A$ that defeats $A$, we have that some $B_- \in \A^-$ defeats $A$ by Lemma~\ref{classlemma}(\ref{A-defeatsB+}), so some $C \in E$ defeats $B_-$. But then $C$ also defeats $B$ by Lemma~\ref{classlemma}(\ref{A-defeatsB+}). Since $C \in E+$, we have that $A$ is acceptable with respect to $E+$.\\
ii) Consider any $A_+ \in E+, A_+ \not\in E$. Suppose $B \in \A$ defeats $A_+$. Then $B$ defeats $A_+$ on some $A'$ that is a premise in $\mathtt{Prem}(A_+)$. By definition of $E+$ and subargument closure of $E$, $A' \in E$. By Lemma~\ref{classlemma}(\ref{A-defeatsB+}), $B_- \in \A^-$ defeats $A'$, and $B_-$ is defeated by some $C \in E$. Since $C$ defeats $B$ (by Lemma ~\ref{classlemma}(\ref{A-defeatsB+})) and $C \in E+$, $A_+$ is acceptable with respect to $E+$.\\
iii) Consider finally any $A \in \A$, $A \notin E+$. Then no $A_-$ is in $E$, so for all $A_-$ there exists a $B$ defeats $A_-$, $\neg \exists C \in E$, $C$ defeats $B$. By Lemma~\ref{classlemma}(\ref{A-defeatsB+}) $B$ also defeats $A$. Suppose for contradiction that $A$ is acceptable w.r.t. $E+$, and so $\exists C' \in E+$, $C'$ defeats $B$. Hence $C'$ must be some $C_+$, where by Lemma~\ref{classlemma}-\ref{A-inE} and construction of $E+$, $C \in E+$ and $C \in E$. By by Lemma~\ref{classlemma}(\ref{A-defeatsB+}), $C$ defeats $B$, contradicting  $\neg \exists C \in E$, $C$ defeats $B$.\\
By i), ii) and iii), $E+$ is a complete extension of $\Delta$.
\\[4pt]
\noindent $T$ = \emph{\textbf{preferred}}:\\
1) Suppose $E$ is a preferred extension of $\Delta$. Suppose for contradiction that $E-$ is not a preferred extension
of $\Delta-$. We have shown that $E-$ is
a complete extension of $\Delta-$. Hence there must be some $E' \supset E-$ that is a complete extension of $\Delta-$. We have shown that $E'+$ is a complete extension of $\Delta$. It is easy to see by construction of $E-$ and $E'+$ that $E \subset E'+$, contradicting $E$ is a preferred extension of $\Delta$.
\\[2pt]
2) Suppose $E$ is a preferred extension of $\Delta-$. Suppose for contradiction that $E+$ is not a preferred extension of $\Delta$. We have shown that $E+$ is a
complete extension of $\Delta$. Hence there must be some $E' \supset E+$ that is a complete extension of $\Delta$. We have shown that $E'-$ is a complete extension of $\Delta-$. It
is easy to see by construction of $E'-$ and $E+$, that $E
\subset E'-$, contradicting $E$ is a preferred extension of $\Delta-$.
\\[4pt]
\noindent $T$ = \emph{\textbf{grounded}}:\\
1) Suppose $E$ is the grounded extension of $\Delta$. Suppose for contradiction that $E-$ is not the grounded extension of $\Delta-$. We have shown that $E-$ is a complete extension of $\Delta-$. Hence there must be some $E' \subset E-$ that is a complete extension of $\Delta-$. We have shown that $E'+$ is a complete extension of $\Delta$. It is easy to see by construction of $E-$ and $E'+$ that $E'+ \subset E$, contradicting $E$ is the grounded extension of $\Delta$.
\\[2pt]
2) Suppose $E$ is the grounded extension of $\Delta-$. Suppose for contradiction that $E+$ is not the grounded extension of $\Delta$. We have shown that $E+$ is a
complete extension of $\Delta$. Hence there must be some $E' \subset E+$ that is a complete extension of $\Delta$. We have shown that $E'-$ is a complete extension of $\Delta-$. It
is easy to see by construction of $E'-$ and $E+$, that $E'-
\subset E$, contradicting $E$ is the grounded extension of $\Delta-$.
\\[4pt]
\noindent $T$ = \emph{\textbf{stable}}:\\
1) Let $E$ be a stable (and so preferred) extension of $\Delta$. Then $E-$ is a preferred extension  of $\Delta-$.
Suppose for contradiction that $E-$ is not a stable extension. Then $\exists B \in \A-$, $B \notin E-$, and $B$ is not defeated by an argument in $E-$. Note that $B \in \A$. It cannot be that $B \in E$ since the fact that $B \in \A-$ would imply by construction of $E-$ that $B \in E-$. Since $E$ is stable, some $C \in E$ defeats $B$. By Lemma~\ref{classlemma}(\ref{A-defeatsB+}) all $C_-$ also defeat $B$, and by Lemma~\ref{classlemma}(\ref{A-inE}a) all such $C_-$ are in
$E$. By construction of $E-$ all such $C_-$ are in $E-$, and so $B$ is defeated by an argument in $E-$.
Contradiction.
\\[2pt]
2) Let $E$ be a stable (and so preferred) extension of $\Delta-$. Then $E+$ is a preferred extension of
$\Delta$. Suppose for contradiction that $E+$ is not a stable extension.
Then $\exists B \in \A$, $B \notin E+$,  and $B$ is not defeated by an argument in $E+$. Since $E+$ is preferred,
$B$ is defeated by a $C \in \A \setminus E+$, where $C$ is not
defeated by an argument in $E+$.
\\
$C$ defeats $B$ on some $\varphi \in \mathtt{Prem}(B)$. Note that $\varphi \in \A-$.
If  $\varphi \in E$ then, since $\varphi$ acceptable w.r.t $E$, $C$ is
defeated by some argument in $E$. But since $E \subseteq E+$ this contradicts that $C$ is
not defeated by an argument in $E+$.  If  $\varphi \not\in E$ then, since $E$ is a
stable extension of $\Delta-$, we have that $\varphi$ is defeated by a $D \in E$,
but then $D$ also defeats $B$. Since $D \in E+$ this contradicts that
$B$ is not defeated by an argument in $E+$.\\[6pt]
Finally, we clearly have for any $E$ and $E-$ that $\mathtt{Conc}(E) = \mathtt{Conc}(E-)$, and likewise for any $E$ and $E+$.  Then the proposition follows from (1) and (2) for each of the above semantics.\\[-13pt]

\end{proof}

\noindent \textbf{Corollary \ref{classicalminCorollary}} Given $\Delta$ and $\Delta^-$ as defined in Proposition \ref{classicalmin}:\\[-13pt]
\begin{enumerate}
  \item $\varphi$ is a $T$ credulously (sceptically) justified conclusion of $\Delta$ iff $\varphi$ is a $T$ credulously (sceptically) justified conclusion of $\Delta^-$.\\[-17pt]
  \item $\Delta^-$ satisfies the postulates \emph{closure under strict rules}, \emph{direct consistency}, \emph{indirect consistency} and \emph{sub-argument closure}.\\[-17pt]

\end{enumerate}

\begin{proof} 1) and \emph{closure under strict rules}, \emph{direct consistency}, and \emph{indirect consistency} immediately follow from Proposition \ref{classicalmin}.  For \emph{sub-argument closure} expressed in Theorem \ref{TheoremSub-argumentClosure}, note that the proof of this theorem appeals to Lemma \ref{Basic properties} which can straightforwardly be seen to apply to $\Delta^-$. The proof also depends on any sub-argument of $A \in E$ not being in conflict with any argument in $E$. This immediately follows for $E-$ in the proof of Proposition \ref{classicalmin}, given that $E-\subseteq E$.\\[-14pt]

\end{proof}

\noindent \textbf{Proposition \ref{PropStrengtheningAssumption}} Let $(\A,\C,\preceq)$ be defined by an AL argumentation theory, where $\preceq$ is defined under the weakest or last link principles, based on the set comparison \sm{$\triangleleft_{\mathtt{Eli}}$}. Then $\forall A,B \in \A$, $\forall A_- \in \A$, if $A \nprec B$ then $A_- \nprec B$.\\[-15pt]

\begin{proof}
Since all arguments are strict continuations of ordinary premises, the last and weakest link principles are evaluated in the same way. Suppose $A \nprec B$. Then $\mathtt{Prem}(A) \sm{\ntriangleleft_{\mathtt{Eli}}} \mathtt{Prem}(B)$. That is to say, it is not the case that $\exists X \in \mathtt{Prem}(A) $ s.t. $\forall Y \in \mathtt{Prem}(B) $, \sm{$X < Y$}, i.e.,  $\forall X \in \mathtt{Prem}(A) $, $\exists Y \in \mathtt{Prem}(B) $ s.t. \sm{$X \nless Y$}. Since $\mathtt{Prem}(A_-) \subseteq \mathtt{Prem}(A)$, it trivially follows that  $\forall X \in \mathtt{Prem}(A_-) $, $\exists Y \in \mathtt{Prem}(B) $ s.t. \sm{$X \nless Y$}, i.e., $A_- \nprec B$.\\[-22pt]

\end{proof}

\noindent \textbf{Proposition~\ref{ALprop0}} Let $\Delta$ be the \emph{c-SAF} based on $(\LA',Cn)$ and $(\Sigma,\leq')$. Then for any complete extension $E$ of $\Delta$: $S = \{\phi | \phi \in \mathtt{Prem}(A), A \in E\}$ is AL-inconsistent iff $S' = Cl_{\R_s}(\{\mathtt{Conc}(A) | A \in E\})$ is inconsistent.\\[-15pt]

\begin{proof}
\emph{Left to right}: if $S$ is AL inconsistent then $\varphi, -\varphi \in Cn(S)$ for any $\varphi$. By definition of $\R_s$, for some $T,T' \subseteq S$ there exist rules $T \imp \varphi$ and $T' \imp -\varphi$ in $\R_s$. Since $E$ is closed under sub-arguments and premises are sub-arguments, $\{\mathtt{Conc}(A) | A \in E\}$ includes $T$ and $T'$. Hence $\varphi, -\varphi \in S'$. That is, $S'$ is inconsistent.\\
\noindent \emph{Right to left}: If $S'$ is inconsistent then $\varphi, -\varphi \in S'$ for some $\varphi$. Since $E$ is closed under sub-arguments, $S \subseteq S'$, and so $S \vdash \varphi$ and $S \vdash -\varphi$. By  Def.~\ref{asbasedonal}(\ref{conconb}), $\{\varphi,-\varphi\}$ is AL-inconsistent, so $Cn(\{\varphi,-\varphi\}) = \LA$. But since $\{\varphi,-\varphi\} \subseteq Cn(S)$ and $Cn(Cn(S)) = Cn(S)$, we have by monotonicity of $Cn$ that $Cn(S) = \LA$ so $S$ is AL-inconsistent.
\end{proof}

\subsection{Proofs for Section \ref{SectionClassicalLogicInstantiations}}

%
%\noindent \textbf{Proposition \ref{DUprop}}
%Let $\A$ be the set of arguments on the basis of $AT$ as defined in Def. \ref{arg} or Def. \ref{DefinitionConsistentArguments}. Let $\Delta_{DU}$ = $(\A,\C_{DU})$, $\Delta_{DD}$ = $(\A,\C_{DD})$, and $\Delta_{UN}$ = $(\A,\C_{UN})$ be Dung frameworks, where $\C_{DU}$ and $\C_{DD}$ are the direct undercut and direct defeat relations, and $\C_{UN}$ the undermining attack relation as defined in Def. \ref{DefAttacks}. Then for $T \in$ $\{$\emph{admissible, complete, preferred, grounded, stable}$\}$:
%
%
%
%\begin{enumerate}
%\item $E$ is a $T$ extension of $\Delta_{UN}$ if and only if $E$ is a $T$ extension of $\Delta_{DU}$
%\item $E$ is a $T$ extension of $\Delta_{UN}$ if and only if $E$ is a $T$ extension of $\Delta_{DD}$
%\end{enumerate}
%
%\begin{proof}
%If $A$ undermines $B$ then $A$ directly undercuts (defeats) $B$. Moreover, if $A$ directly undercuts (defeats) $B$ then $A$ can with a single strict inference be continued into an argument $A'$ that undermines $B$. Given our assumption that $\preceq$ is reasonable (Def. \ref{Def-Reasonable}), then (by condition 1-iii)) it holds that $A$ and $A'$ are equally strong. Then, since $A$ and $A'$ have the same premises, they also have the same defeaters. Then clearly $A \in E$ iff $A' \in E$.
%\end{proof}

\noindent \textbf{Theorem \ref{TheoremPS}} Let $(\A,\C,\preceq)$ be a c-\emph{SAF} corresponding to a default theory $\Gamma$, and for any $\Sigma \subseteq \Gamma$,
let $\mathtt{Args}(\Sigma) \subseteq \A$ be the set of all arguments with premises taken from $\Sigma$. Then:\\
\noindent 1) If $\Sigma$ is a preferred subtheory of $\Gamma$, then $\mathtt{Args}(\Sigma)$ is a stable extension of $(\A,\C,\preceq)$.\\[3pt]
\noindent 2) If $E$ is a stable extension of $(\A,\C,\preceq)$, then $\bigcup_{A \in E}\mathtt{Prem}(A)$ is a preferred subtheory of $\Gamma$.\\[-15pt]

\begin{proof}

\noindent \emph{Proof of 1)}: Firstly, we show that $\mathtt{Args}(\Sigma)$ is conflict free. Since $\Sigma$ is consistent, $\Sigma \nvdash_c \alpha, \neg \alpha$ for any $\alpha$. Suppose for contradiction that $\mathtt{Args}(\Sigma)$ is not conflict free, in which case $\exists X,Y \in \mathtt{Args}(\Sigma)$ s.t. $\mathtt{Conc}(X)$ = $\alpha$, $\neg \alpha \in \mathtt{Prem}(Y)$. But then since every such argument is obtained by applying the strict rules encoding all classical inferences to $\Sigma$, this implies $\Sigma \vdash_c \alpha, \neg \alpha$. Contradiction.
\\[2pt]
We now show that for any $Y \in \A \setminus \mathtt{Args}(\Sigma)$, $\exists X \in \mathtt{Args}(\Sigma)$ s.t. $X$ defeats $Y$. Consider any such $Y$.  Then $\exists \gamma \in
\mathtt{Prem}(Y)$, $\gamma \notin \Sigma$. By construction, $\Sigma$ =
$\Sigma_1 \cup \ldots \cup \Sigma_n$ such that for $i = 1 \ldots n$,
$\Sigma_1 \cup \ldots \cup \Sigma_i$ is a maximal consistent subset of
$\Gamma_1,\ldots, \Gamma_i$. Hence, suppose $\gamma \in \Gamma_j$ for
some $j = 1 \ldots n$. Then $\Sigma_1 \cup \ldots \cup \Sigma_j \cup
\{\gamma\} \vdash_c \bot$. Hence $\Sigma_1 \cup \ldots \cup \Sigma_j
\vdash_c \neg \gamma$.
Hence, $\exists X \in \mathtt{Args}(\Sigma_1 \cup \ldots \cup \Sigma_j)$
s.t. $\mathtt{Conc}(X)$ = $\neg \gamma$, and so $X \rightharpoonup Y$.
Since $\gamma \in \Gamma_j$, and all premises in $X$ are in $\Gamma_i$,
$i \leq j$ (i.e., every premise in $X$ is greater or equal to $\gamma$) then
$\mathtt{Prem}(Y)$ $\trianglelefteq_{\mathtt{Eli}}$ $\mathtt{Prem}(X) $, and so
 by the weakest or last link principle, $X \nprec Y$. Hence $X \hookrightarrow Y$.
\\[4pt]
\noindent \emph{Proof of 2)}: Firstly, we show that $\bigcup_{A \in E}\mathtt{Prem}(A)$ must be
consistent. Suppose for contradiction that $\exists X,Y \in E$ s.t.
$\mathtt{Prem}(X)$ $\cup$ $\mathtt{Prem}(Y) \vdash_c \bot$. Let
$\{\alpha_1,\ldots,\alpha_m\}$ be a minimal (under set inclusion) subset
of $\mathtt{Prem}(X)$ $\cup$ $\mathtt{Prem}(Y)$ s.t.
$\alpha_1,\ldots,\alpha_m \vdash \bot$. Hence,
$\alpha_1,\ldots,\alpha_{m-1} \vdash \neg \alpha_m$. Since $E$ is stable
and so complete, then by sub-argument closure (Theorem
\ref{TheoremSub-argumentClosure}), $\{A_1,\ldots,A_m\}\subseteq E$,
where for $i = 1 \ldots m$, $\mathtt{Prem}(A_i) = \{\alpha_i\}$.
By Lemma \ref{Lemma3}, if $\{A_1,\ldots,A_m\} \subseteq E$, where $E$ is
a complete extension, then any strict continuation of $\{A_1,\ldots,A_m\}$
is acceptable w.r.t. $E$, and so in $E$. Hence $A \in E$ where $A$
concludes $\neg \alpha_m$. Hence, $A \rightharpoonup A_m$, contradicting
$E$ is conflict free.
  \\[2pt]
Next, let $E_1, \ldots, E_n$ be the partition of $\mathtt{Form}(E)$ s.t.
for $i = 1 \ldots n$, $E_i$ is a (possibly) empty subset of $\Gamma_i$
in the stratification $\Gamma_1,\ldots,\Gamma_n$ of $\Gamma$. Suppose
for contradiction that $\mathtt{Form}(E)$ is not a preferred subtheory.
Then, for some $i$, for $k$ = $1 \ldots i-1$,
  $E_1, \ldots, E_k$ is a maximal consistent subset of $\Gamma_1,\ldots,
\Gamma_{i-1}$, and $\exists \alpha \in \Gamma_i$ s.t. $\alpha \notin
E_i$, and $E_1 \cup \ldots \cup E_{i-1} \cup E_i \cup \{\alpha \}
\nvdash \bot$. Hence, $\exists Y \in \A$, $\mathtt{Prem}(Y)$ =
$\{\alpha\}$, $Y \notin E$. By assumption of $E$ being a stable
extension, $\exists X \in E$, $X \hookrightarrow Y$. Since $E_1 \cup \ldots
\cup E_{i-1} \cup E_i \cup \{\alpha \} \nvdash \bot$, then $E_1 \cup
\ldots \cup E_{i-1} \cup E_i \nvdash \neg \alpha$, and so it must be
that some $\beta \in \mathtt{Prem}(X)$ is in $E_j$, $j > i$; i.e.,  $\beta \in \mathtt{Prem}(X)$,  $\mathtt{Prem}(Y)$ =
$\{\alpha\}$, and $\beta < \alpha$. \sm{Hence
$\mathtt{Prem}(X) \triangleleft_{\mathtt{Eli}} \mathtt{Prem}(Y)$},
and so $X \prec Y$ under the weakest or
last link principle, contradicting $X \hookrightarrow Y$.\\[-20pt]

\end{proof}
%%%
%\marginpar{Delete a$+$v10, a$+$v10sum, CDM03, GM00, Dun93a, vre93thesis and more? And I corrected the various refs to COMMA.}

\scalefont{1.05}

%\bibliographystyle{named}
%\bibliographystyle{plain}
%\bibliography{MetaArg}

\end{document}